\documentclass{siamonline1116}

\usepackage{hyperref}
\usepackage{amsmath,amssymb,amsfonts}
\usepackage{mathrsfs}
\usepackage{mathtools}
\usepackage{cleveref}
\usepackage{acro}
\usepackage{graphicx,subcaption}
\usepackage{placeins}
\usepackage[squaren,thickqspace]{SIunits}
\usepackage{booktabs}
\usepackage{enumerate,paralist}
\usepackage{tikz-cd}
\usepackage{tikz}
\usetikzlibrary{shapes}
\usetikzlibrary{arrows}
\usetikzlibrary{matrix}
\usetikzlibrary{overlay-beamer-styles}
\usetikzlibrary{shapes,arrows}
\usetikzlibrary{decorations.text}
\usetikzlibrary{shapes.geometric}
\usetikzlibrary{fit,calc,positioning,automata}  
\usepackage{tikz-cd}  
\usepackage{pgffor}
\usepackage{pgfplots}
\usepackage{todonotes}

\newcommand{\Real}{\mathbb{R}}
\newcommand{\Integer}{\mathbb{Z}}
\DeclareMathOperator{\OpB}{\mathcal{B}}
\DeclareMathOperator{\OpC}{\mathcal{C}}
\DeclareMathOperator{\range}{ran}

\DeclareMathOperator*{\argmin}{arg\,min}
\DeclareMathOperator*{\argmax}{arg\,max}
\newcommand{\MeasurableMaps}{\mathscr{M}}
\newcommand{\Lebesgue}{L}
\newcommand{\stochastic}[1]{\mathsf{#1}}
\newcommand{\SAlg}[1]{\mathfrak{S}_{#1}}
\newcommand{\PClass}[1]{\mathscr{P}_{#1}}
\DeclareMathOperator{\Expect}{\mathbb{E}}
\newcommand{\ProbMeas}{\mathbb{P}}
\newcommand{\CondLaw}{\pi}
\newcommand{\Loss}{L}
\newcommand{\distance}[1]{\ell_{#1}}
\newcommand{\risk}{\mathcal{R}}
\newcommand{\AverageRisk}{\risk}
\newcommand{\DataSpace}{Y}
\newcommand{\DataSpaceSAlg}{\SAlg{\DataSpace}}
\newcommand{\DataModel}{\mathcal{M}}
\newcommand{\DataLikelihood}{\mathcal{L}}
\newcommand{\ProbNoise}{\ProbMeas_{\mathrm{noise}}}
\newcommand{\datamanifold}{\mathbb{M}}
\newcommand{\data}{y}
\newcommand{\stdata}{\stochastic{\data}}
\newcommand{\datanoise}{e}
\newcommand{\stdatanoise}{\stochastic{\datanoise}}
\newcommand{\RecSpace}{X}
\newcommand{\SignalLoss}{\distance{\RecSpace}}
\newcommand{\RecSpaceSAlg}{\SAlg{\RecSpace}}
\newcommand{\ProbSignalTrue}{\pi^*}
\newcommand{\signaldomain}{\Omega}
\newcommand{\signal}{x}
\newcommand{\signaltrue}{\signal^*}
\newcommand{\stsignal}{\stochastic{\signal}}
\newcommand{\signalother}{x'}
\newcommand{\SignalPrior}{\pi_0}
\newcommand{\MLSignalParamSet}{\Theta}
\newcommand{\MLsignalparam}{\theta}
\newcommand{\DecisionClass}{\mathscr{D}}
\newcommand{\DecisionSpace}{D}
\newcommand{\DecisionLoss}{\distance{\DecisionSpace}}
\newcommand{\DecisionSpaceSAlg}{\SAlg{\DecisionSpace}}
\newcommand{\decision}{d}
\newcommand{\decisionother}{d'}
\newcommand{\decisiontrue}{\decision^*}

\DeclareMathOperator{\featuremap}{\tau}
\newcommand{\TaskSpace}{\triangle}
\newcommand{\TaskPrior}{\eta_{0}}
\newcommand{\TaskLoss}{\Loss_{\DecisionSpace}}
\newcommand{\TaskModel}{\ProbMeas}
\newcommand{\task}{z}
\newcommand{\sttask}{\stochastic{\task}}
\newcommand{\MLtaskparam}{\vartheta}
\newcommand{\MLTaskParamSet}{\Xi}
\newcommand{\JointLawSignalData}{\mu}
\newcommand{\JointLawTaskSignal}{\eta}
\newcommand{\JointLawTaskData}{\nu}
\newcommand{\JointLawTaskSignalData}{\sigma}
\DeclareMathOperator{\ForwardOp}{\mathcal{A}}
\newcommand{\RecOp}[1]{\ForwardOp_{#1}^{\dagger}}
\newcommand{\RecOpOptim}{\hspace{2pt}\widehat{\hspace{-2pt}\ForwardOp}^{\dagger}}
\newcommand{\TaskOp}[1]{\mathcal{T}_{#1}}
\newcommand{\TaskOpOptim}{\widehat{\mathcal{T}}}
\newcommand{\JointLoss}{\ell_{\mathrm{joint}}}
\DeclareMathOperator{\EndToEnd}{\mathcal{B}}

\newcommand{\WordSpace}{\mathcal{W}}
\newcommand{\CaptionSpace}{\mathcal{C}}
\newcommand{\word}{w}
\newcommand{\stopword}{\word_{\text{stop}}}
\newcommand{\dint}{\,\mathrm{d}}
\newcommand{\Cdot}{\,\cdot\,}

\newcommand{\otoprule}{\midrule[\heavyrulewidth]}
\theoremstyle{remark}

\crefname{subsection}{section}{sections}
\Crefname{subsection}{Section}{Sections}
\crefname{subsubsection}{section}{sections}
\Crefname{subsubsection}{Section}{Sections}
\crefname{equation}{}{}
\Crefname{equation}{}{}
\crefname{item}{}{}
\Crefname{item}{}{}
\AtBeginDocument{%
	\ifpdf
	\hypersetup{
		pdftitle={\TheTitle},
		pdfauthor={\TheAuthors},
		colorlinks=true,
		linkbordercolor= white, 
		linkcolor=blue,
		filecolor=blue,      
		urlcolor=blue,
		citecolor = red 
	}
	\fi
}
\addunit{\pixel}{pixel}
\addunit{\voxel}{voxel}
\addunit{\decibel}{dB}
\addunit{\byte}{B}
\addunit{\hounsfield}{HU}

\DeclareAcronym{ML}{
	short = ML,
	long = maximum likelihood}
\DeclareAcronym{EM}{
	short = EM,
	long = expectation maximization}
\DeclareAcronym{MCMC}{
	short = MCMC,
	long = Markov chain Monte Carlo}
\DeclareAcronym{MAP}{
	short = MAP,
	long = maximum a posteriori}
\DeclareAcronym{CM}{
	short = CM,
	long = conditional mean}
\DeclareAcronym{ADMM}{
	short =  ADMM,
	long = alternating direction method of multipliers}
\DeclareAcronym{TV}{
	short = TV,
	long = total variation}
\DeclareAcronym{ODE}{
	short = ODE,
	long = ordinary differential equation}
\DeclareAcronym{PDE}{
	short = PDE,
	long = partial differential equation}
\DeclareAcronym{RKHS}{
	short = RKHS,
	long = reproducing kernel Hilbert space}
\DeclareAcronym{ODL}{
	short = ODL,
	long = Operator Discretization Library}
\DeclareAcronym{API}{
	short = API,
	long = application programming interface}
\DeclareAcronym{CT}{
	short = CT,
	long = computed tomography}
\DeclareAcronym{MRI}{
	short = MRI,
	long = magnetic resonance imaging}
\DeclareAcronym{FBP}{
	short = FBP,
	long = filtered backprojection}
\DeclareAcronym{PSNR}{
	short = PSNR,
	long = peak signal to noise ratio}
\DeclareAcronym{ReLU}{
	short = ReLU,
	long = rectified linear unit}  
\DeclareAcronym{GPU}{
	short = GPU,
	long = graphics processing unit}  
\DeclareAcronym{CPU}{
	short = CPU,
	long = central processing unit}      
\DeclareAcronym{IID}{
	short = i.i.d.,
	long = independent and identically distributed}      
\DeclareAcronym{LSTM}{
	short = LSTM,
	long = long short-term memory}
\DeclareAcronym{SGD}{
	short = SGD,
	long = stochastic gradient descent}
\DeclareAcronym{GAN}{
	short = GAN,
	long = generative adversarial network}

\newcommand{\TheTitle}{Task adapted reconstruction for inverse problems}
 
\newcommand{\TheAuthors}{Jonas Adler, Sebastian Lunz, Carola-Bibiane Sch\"onlieb, and Ozan \"Oktem}
\title{{\TheTitle}}

\author{
	Jonas Adler\thanks{Department of Mathematics, KTH--Royal Institute of Technology, 100 44 Stockholm, Sweden; Elekta AB, Box 7593, SE-103 93 Stockholm, Sweden (\email{jonasadl@kth.se}).}
	\and
	Sebastian Lunz\thanks{Centre for Mathematical Sciences, University of Cambridge, Cambridge CB3 0WA, United Kingdom (\email{sl767@cam.ac.uk}).}
	\and
	Olivier Verdier\thanks{Department of Mathematics, KTH--Royal Institute of Technology, 100 44 Stockholm, Sweden;
    Department of Computing, Mathematics and Physics, Western Norway University of Applied Sciences, Bergen, Norway
		(\email{olivierv@kth.se}, \email{olivier.verdier@hvl.no}).}
	\and
	Carola-Bibiane Sch\"onlieb\thanks{Centre for Mathematical Sciences, University of Cambridge, Cambridge CB3 0WA, United Kingdom (\email{cbs31@cam.ac.uk}).}
	\and
	Ozan \"Oktem\thanks{Department of Mathematics, KTH--Royal Institute of Technology, 100 44 Stockholm, Sweden   
		(\email{ozan@kth.se}).}
}

\begin{document}
	\maketitle

	\begin{abstract}
		The paper considers the problem of performing a task defined on a model parameter that is only observed indirectly through noisy data in an ill-posed inverse problem.
		A key aspect is to formalize the steps of reconstruction and task as appropriate estimators (non-randomized decision rules) in statistical estimation problems. 
		The implementation makes use of (deep) neural networks to provide a differentiable parametrization of the family of estimators for both steps.
		These networks are combined and jointly trained against suitable supervised training data in order to minimize a joint differentiable loss function, resulting in an end-to-end task adapted reconstruction method.
		The suggested framework is generic, yet adaptable, with a plug-and-play structure for adjusting both the inverse problem and the task at hand.
		More precisely, the data model (forward operator and statistical model of the noise) associated with the inverse problem is exchangeable, e.g., by using neural network architecture given by a learned iterative method.
		Furthermore, any task that is encodable as a trainable neural network can be used. 
		The approach is demonstrated on joint tomographic image reconstruction, classification and joint tomographic image reconstruction segmentation. 
	\end{abstract}

	
	
	\acresetall
	\acuse{MRI,ADMM,GPU,CPU}

	\section{Introduction}\label{sec:Intro}
	The overall goal in inverse problems is to determine model parameters such that model predictions match measured data to sufficient accuracy.
	Such problems arises in various scientific disciplines. One example is biomedical imaging where the image is the ``model parameter'' that needs to be determined from data acquired using an imaging device like a tomographic scanner or a microscope.
	The prime example of this is tomographic imaging in medicine which has revolutionized health care over the past 30 years, allowing doctors to find disease earlier and improve patient outcomes \cite{Deuflhard:2010aa,Rubin:2014aa}.
	Likewise, scientific computing is nowadays considered to be the ``third pillar of science'' standing right next to theoretical analysis and experiments for scientific discovery, much thanks to possibilities for simulating and optimizing complex physical and engineering systems.
	A key element in realizing this role is the ability to solve the inverse problem of calibrating parameters in a mathematical model of the system so that simulations match benchmark data \cite{Biegler:2011aa}.
	
	The inverse problem of reconstructing the model parameter from data is often only one out of many steps in a procedure where the recovered model parameter is used in decision making.
	The reconstructed model parameter is typically summarized, either by an expert or automatically, and resulting task dependent descriptors are then used as basis for decision making, see \cref{fig:Pipeline}. 
	
	Clearly, there are several disadvantages with performing the various parts of the above pipeline independently from each other. 
	Each single step is prone to introduce approximations that are not accounted for by subsequent steps, the reconstruction may not consider the end task, and the feature extraction may not consider measured data.
	In fact, the task is almost always only accounted for at the very final step.
	It is therefore natural to ask whether one may adapt the reconstruction method for the specific task at hand. 
	\emph{Task adapted reconstruction} refers to methods that integrate the reconstruction procedure with (parts of) the decision making procedure associated with the task. 
	This is sometimes also referred to as ``end-to-end'' reconstruction. 
	\begin{figure}[t]
		\centering
		\tikzstyle{smoothblock} = [rectangle, draw,  
		text width=6.7em, text centered, rounded corners, minimum height=4em]
		\tikzstyle{block} = [rectangle, draw,  
		text width=6.7em, text centered, minimum height=4em]
		\tikzstyle{line} = [draw, -latex']
		\begin{tikzpicture}[node distance = 0.23\linewidth, auto,scale=0.975, transform shape]
		\node [smoothblock] (sample) {Sample};    
		\node [block, below of=sample, node distance = 0.12\textheight] (preparation) {Sample preparation};    
		\node [block, right of=preparation, node distance = 0.35\linewidth] (acquisition) {Data \\ acquisition};    
		\node [block, right of=acquisition, node distance = 0.35\linewidth] (preprocess) {Data \\ pre-processing};
		\node [block, below of=preparation,node distance = 0.17\textheight] (reconstruction) {Reconstruction};
		\node [block, right of=reconstruction, node distance = 0.35\linewidth] (feature) {Feature extraction};
		\node [block, right of=feature, node distance = 0.35\linewidth] (model building) {Model building};
		\node [smoothblock, below of=model building, node distance = 0.12\textheight] (model) {Task adapted \\ model};
		\path [line] (sample) -- node[align=left] {} (preparation);
		\path [line] (preparation) -- node[align=left] {} (acquisition);
		\path [line] (acquisition) -- node[align=left] {Raw \\ data} (preprocess);
		\draw [->,postaction={decorate,decoration={text along path,raise={1.0ex},reverse path,
				text align=center,text={Clean data}}}] (preprocess) to [out=270,in=90,looseness=0.45] (reconstruction);   
		\path [line] (reconstruction) -- node[above,align=left] {Model \\ parameter} (feature);
		\path [line] (feature) -- node[above,align=left] {Extracted \\ features} (model building);
		\path [line] (model building) -- node[align=left] {} (model);
		\node (boxed) [draw=red, fit= (reconstruction) (feature) (model building), inner sep=0.295cm, 
		dashed, ultra thick, fill=red!20, fill opacity=0.2] {};
		\end{tikzpicture}
		\caption{Typical workflow involving an inverse problem. The second row represents the data acquisition where raw data is acquired and pre-processed, 
			resulting in cleaned data. In the third row, the cleaned data is used as input to a reconstruction step that recovers the model parameter, which is then 
			post-processed to extract features that are used as input for model building. 
			The final outcome is a task adapted model that can be used for decision making.
			The dotted rectangular part outlines the steps that are unified by task adapted reconstruction.}\label{fig:Pipeline}
	\end{figure}
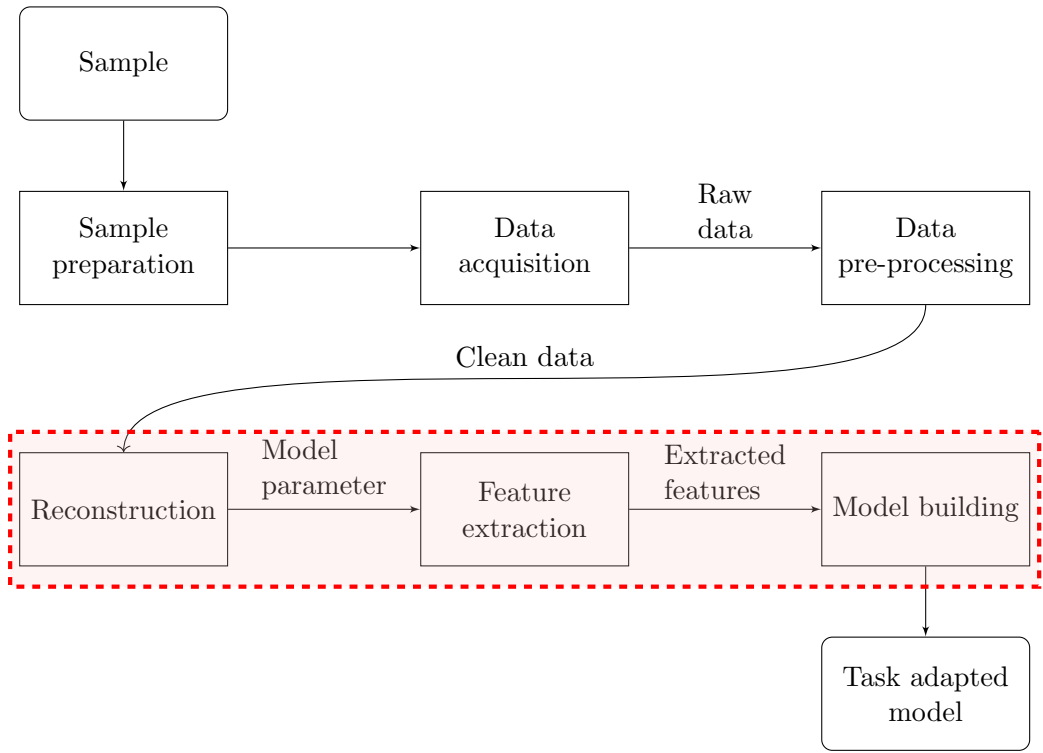
	
	\section{Overview}
	We start with a brief survey of existing approaches to task adapted reconstruction in the context of tomographic image reconstruction (\Cref{sec:Overview}), which also points out the drawbacks that come with these approaches.
	The section that follows (\cref{sec:InvProb}) introduces the statistical view on inverse problems. 
	More specifically, we consider Bayesian inversion (\cref{sec:BayesianInversion}) in which a reconstruction method is a statistical estimator (\cref{sec:RecoDecisionRule}).
	After pointing out some key challenges associated with Bayesian inversion (\cref{sec:TheoryStatReg}), we introduce the learned iterative methods (\cref{sec:LIR}) that are later used in our applications of task adapted reconstruction (\cref{sec:Applications}).
	We then switch gears and consider tasks on the model parameter space that can be formulated as a statistical estimation problem (\cref{sec:TaskModelParam}). 
	The first step is to provide an abstract framework (\cref{sec:AbstractTask}) with a plug-and-play structure for adapting to a specifik task.
	To further illustrate the wide applicability of this framework, \cref{sec:ExampleTasks} describes a number of tasks that are worked out in detail followed by further examples in \cref{sec:OtherTasks}.
	
	\Cref{sec:TaskAdaptedRec} introduces task adapted reconstruction in an abstract setting (\cref{sec:TaskAdaptedAbstract}).
	It assumes that both the reconstruction (\cref{sec:RecoDecisionRule}) and task (\cref{sec:AbstractTask}) are given by appropriate decision rules. 
	An important part is the computational implementation (\cref{sec:TaskCompImpl}) that is based on neural networks.  
	This is followed by two applications that are worked out in detail in \cref{sec:Applications}.
	In \cref{sec:TheoryConsider} we provide some theoretical considerations regarding regularizing properties and the potential advantage that comes with using a joint approach.
	The final section (\cref{sec:Discussion}) contains a discussion and outlook on future research in this area.
	
	\section{Specific contributions}
	The paper offers a generic, yet highly adaptable, framework for task adapted reconstruction that is based on considering both the reconstruction and the task as statistical estimation problems.
	The implementation uses neural networks for both these steps, which is essential for both performance in terms of quality and computational feasibility as shown in the example for joint tomographic reconstruction and segmentation.
	Both networks are trained jointly using a joint loss function \cref{eq:JointLoss} that ``interpolates'' between sequential and end-to-end approaches. 
	Here, sequential refers to a setting where the neural network for reconstruction is trained separately and its output is used in the training of the neural network for the task. End-to-end is when a neural network for the task is trained directly against data without explicitly introducing a reconstruction step.
	
	To the best of our knowledge, this is the first paper that offers an approach to task adapted reconstruction that unifies reconstruction with such a diverse set of tasks in a computationally feasible manner under the guiding principles of statistical decision theory, learning, and efficient inference algorithms.
	This allows for re-using algorithmic components thereby opening up new ways of thinking about machine learning and inverse problems that may ultimately lead to deeper understanding of the possibilities for integrating elements of decision making into the reconstruction.
	Furthermore, introducing the joint loss in \cref{eq:JointLoss} and investigating its properties (\cref{sec:TheoryConsider}) are novel contributions. 
	Our work also leaves many open questions and future research directions for the inverse problems and machine learning communities as outlined in \cref{sec:Discussion}.
	
	\section{Survey of task adapted tomographic reconstruction}\label{sec:Overview}
	There is an ongoing effort within the inverse problems community to include signal processing steps associated with performing a task jointly with the reconstruction step. 
	In tomographic imaging, most of the tasks considered this far correspond to feature extraction, e.g., segmentation or extraction of features expressed through sparse representations as in compressed sensing.
	In such case, task adapted reconstruction reduces to joint reconstruction and feature extraction. 
	
	Current approaches to task adapted reconstruction are primarily based one the classical approach to inverse problems.
	In this setting, the problem is to recover the true unknown feature $\decisiontrue \in \DecisionSpace$ from data $\data \in \DataSpace$ by solving an operator equation:  
	\begin{equation}\label{eq:InvProbOpEq}
		\data = \ForwardOp(\signaltrue) + \datanoise
		\quad\text{and}\quad
		\decisiontrue = \TaskOp{}(\signaltrue).
	\end{equation}
	The forward operator $\ForwardOp \colon \RecSpace \to \DataSpace$ models how a model parameter (image in tomographic imaging) gives rise to data and the task operator $\TaskOp{} \colon \RecSpace \to \DecisionSpace$ represents the feature extraction. In the above, both these are assumed to be known. 
	Likewise, $\datanoise$ is generated by a $\DataSpace$-valued random variable $\stdatanoise \sim \ProbNoise$ with known distribution, so task adapted reconstruction reduces to finding the dotted operator in \cref{eq:FeatureRec}. 
	\begin{equation}\label{eq:FeatureRec}
		\begin{tikzcd}[column sep=large,row sep=large,ampersand replacement=\&]
			\RecSpace \arrow[r, "\ForwardOp"] \arrow[swap, d, "\TaskOp{}"] \& 
			\DataSpace \arrow[ld, dotted] \\
			\DecisionSpace \&
		\end{tikzcd}
	\end{equation}
	Note that the task operator $\TaskOp{} \colon \RecSpace \to \DecisionSpace$ is often highly non-injective, so it makes no sense to consider $\ForwardOp \circ \TaskOp{}^{-1} \colon \DecisionSpace \to \DataSpace$ as ``new forward operator''.
	
	Approaches based on ``solving'' \cref{eq:InvProbOpEq} heavily depend on the nature of the $\TaskOp{}$ and below is a brief list of prior work in the context of tomographic imaging.
	\begin{description}
		\item[Edge recovery:]
		Lambda-tomography \cite{Krishnan:2015aa} is a non-iterative method that recovers edges directly from noisy tomographic data using the canonical relation from microlocal analysis.
		Another non-iterative approach combines the method of approximate inverse with an explicit task operator, e.g., a Canny edge detector \cite{Louis:2011aa}.
		Finally, it is also possible to use a variational approach with suitable regularizer. 
		Examples relevant for edge recovery are variants of \ac{TV} \cite{Burger:2013aa,Benning:2018aa} or sparsity promoting $\ell_1$-type of regularizers with an underlying dictionary that is specifically designed to sparsely represent edges, like curvelets, shearlets, beamlets, and bandlets \cite{Foucart:2013aa,Rubinstein:2010aa,Kutyniok:2012aa}.
		\item[Segmentation:]
		Methods for joint reconstruction and segmentation is an active area of research. 
		Most approaches are based on a variational scheme with suitably chosen regularizers and control variables.
		One example is usage of Mumford-Shah penalty \cite{Ramlau:2007aa,Hohm:2015aa}, another is based on level set approaches \cite{Yoon:2010aa}.
		A further refinement is to consider semantic segmentation. 
		Here, a variational scheme that amounts to computing a \ac{MAP} estimator with a Gauss-Markov-Potts type of prior shows promising results on small-scale examples \cite{Mohammad-Djafari:2009aa,Romanov:2016aa}.
		\item[Image registration:]
		To register a template against an indirectly observed target (indirect image registration) is a key step for reconstruction in spatiotemporal imaging. 
		The (temporal) deformation can be modeled using optical flow \cite{Burger:2017aa} or diffeomorphic deformations \cite{Chen:2018aa,Gris:2018aa}. 
		Yet another approach is to consider  optimal transport \cite{Karlsson:2017aa}.
	\end{description}
	
	Approaches to task adapted reconstruction that are based on solving  \cref{eq:InvProbOpEq} suffer from two issues that seriously limit their usefulness in practical imaging applications.
	The first is the requirement for an explicit handcrafted task operator $\TaskOp{} \colon \RecSpace \to \DecisionSpace$.
	More advanced tasks, like many mentioned in \cref{sec:ExampleTasks,sec:OtherTasks}, are difficult to encode in this way and available examples mostly consider task operators that extract ``simple'' features that must be further processed before they can be used for decision making.
	This is also the reason for why the term ``feature reconstruction'' \cite{Louis:2011aa} is used as a proxy for task adapted reconstruction.
	The second issue relates to computational feasibility.  
	Evaluating the task operator, like in segmentation and image registration, is computationally demanding and requires setting values for extra (nuisance) parameters. 
	Furthermore, most state-of-the-art approaches for solving \cref{eq:InvProbOpEq} are based on variational methods, which quickly become computationally unfeasible for large scale imaging problems.
	
	Many complex tasks have been successfully addressed using techniques from machine learning, so it makes sense to investigate whether such techniques can be integrated with reconstruction for task adapted reconstruction. 
	One example is given in \cite{Wu:2017} for abnormality (tumor) detection in low-dose \ac{CT} imaging. 
	The idea here is to jointly train a learned iterative scheme for reconstruction \cite{Adler:2018aa} with a 3D convolutional neural network for detecting the abnormality in the reconstructed images.
	Another examples introduces a unified deep neural network architecture (SegNetMRI) for combined Fourier inversion (MRI image reconstruction) and segmentation \cite{Sun:2018ab}. 
	Here, one has two neural networks with the same encoder-decoder structure, one for MRI reconstruction consisting of multiple cascaded blocks, each containing an encoder-decoder unit and a data fidelity unit, and the other for segmentation. 
	These are pre-trained and coupled by ensuring they share reconstruction encoders. 
	These two examples are special cases of the generic approach develop in \cref{sec:TaskAdaptedRec}.
	
	\section{Statistical inverse problems}\label{sec:InvProb}
	Let $\RecSpace$ and $\DataSpace$ denote separable Banach spaces where  $(\RecSpace,\RecSpaceSAlg)$ and $(\DataSpace,\DataSpaceSAlg)$ are measurable spaces. Next, let  $\PClass{\RecSpace}$ and $\PClass{\DataSpace}$ denote spaces of probability measures on $\RecSpace$ and $\DataSpace$, respectively.
	Following \cite{Evans:2002aa}, a \emph{(statistical) inverse problem} amounts to reconstructing (estimating) $\signaltrue \in \RecSpace$ from measured data $\data \in \DataSpace$ that is generated by a $\DataSpace$-valued random variable $\stdata$ where 
	\begin{equation}\label{eq:DataModel}
		\stdata \sim \DataModel(\signaltrue)
		\quad\text{with known $\DataModel \colon \RecSpace \to \PClass{\DataSpace}$ (data model).}
	\end{equation}
	Elements in $\RecSpace$ (model parameter space) represent possible model parameters and elements in $\DataSpace$ (data space) represent possible data. 
	In tomographic imaging, elements in $\RecSpace$ are often functions defined on a fixed domain in $\Real^n$ representing images and elements in $\DataSpace$ are real-valued functions defined on a fixed manifold $\datamanifold$, which is given by the acquisition geometry associated with the measurements.
	Furthermore, just as in the classical setting, most statistical inverse problems do not have a unique solution in the sense that the model parameter is not \emph{identifiable} \cite[section~2.3]{Evans:2002aa}.

	A common data model is when data is contaminated with additive noise:
	\begin{equation}\label{eq:ForwardOpRel}
		\stdata = \ForwardOp(\signaltrue) + \stdatanoise
		\quad\text{with $\stdatanoise \sim \ProbNoise$ for some known $\ProbNoise \in \PClass{\DataSpace}$.}
	\end{equation}
	Here, $\ForwardOp \colon \RecSpace \to \DataSpace$ (forward operator) models how data is generated in absence of noise and $\stdatanoise \sim \ProbNoise$ models noise. 
	If $\stdatanoise$ is independent from $\signaltrue$, then \cref{eq:ForwardOpRel} amounts to the data model 
	\[ \DataModel(\signal) = \delta_{\ForwardOp(\signal)} \circledast \ProbNoise = \ProbNoise\bigl( \Cdot - \ForwardOp(\signal) \bigr)
	\quad\text{for any $\signal \in \RecSpace$.} 
	\]
	Another data model is when $\DataModel(\signal)$ is a Poisson random measure on $\DataSpace$ with mean $\ForwardOp(\signal)$. This is a suitable data model for imaging modalities that rely on counting statistics in a low-dose setting, such as line of response PET \cite{Kadrmas:2004aa} \cite[section~3.2]{Natterer:2001ab} and variants of fluorescence microscopy \cite{Hell:2007aa,Diaspro:2007aa}, see also \cite{Hohage:2016,Streit:2010aa} for a more abstract treatment. 
	
	\subsection{Bayesian inversion}\label{sec:BayesianInversion}
	Only seeking an estimate of $\signaltrue \in \RecSpace$ is limiting since it does not account for the uncertainty.
	A more comprehensive analysis is based on introducing a $\RecSpace$-valued random variable $\stsignal \sim \ProbSignalTrue$ whose true (unknown) probability distribution $\ProbSignalTrue \in \PClass{\RecSpace}$ generates $\signaltrue$. 
	One can then rephrase the inverse problem stated earlier as the task of recovering the probability measure $\ProbSignalTrue \in \PClass{\RecSpace}$ given data $\data \in \DataSpace$ generated by $\stdata$, which is related to $\signaltrue$ through the data model as in \cref{eq:DataModel}.
	An important special case is when $\ProbSignalTrue$ is parametrized by $\signaltrue \in \RecSpace$ in a known way, so the inverse problem reduces to the task of recovering $\signaltrue \in \RecSpace$.
	
	In a Bayesian setting, one considers the posterior distribution of $\stsignal$ given $\stdata=\data$ up to a constant of proportionality. 
	More precisely, consider a setting where the joint law $(\stsignal,\stdata) \sim \JointLawSignalData$ can be written in terms of conditional probabilities:
	\begin{equation}\label{eq:BayesianSetting}
		\JointLawSignalData = \SignalPrior(\signaltrue) \otimes \CondLaw( \stdata \mid \stsignal=\signaltrue) = \SignalPrior(\signaltrue) \otimes \DataModel(\signaltrue).
	\end{equation} 
	Here, $\SignalPrior$ serves as a (possibly improper) prior and the last equality in \cref{eq:BayesianSetting} follows from the definition of the data model as the conditional distribution of $\stdata$ given $\stsignal = \signaltrue$.
	In particular, the joint law $\JointLawSignalData$ in \cref{eq:BayesianSetting} is proportional to the posterior, so the decomposition above exists as soon as Bayes' theorem holds.
	This is the case in a rather general setting \cite[Theorem~14]{Dashti:2016aa}, but a decomposition is also possible is some cases where the prior is not proper\footnote{Under certain circumstances it is possible to work with improper priors on the model parameter space, e.g., by computing posterior distributions that approximate the posteriors one would have obtained using proper conjugate priors whose extreme values coincide with the improper prior.}. 

	A key point in the Bayesian setting is to explore the posterior distribution of $\stsignal$ given $\stdata=\data$ assuming that \emph{both} $\signal \mapsto \SignalPrior \in \PClass{\RecSpace}$ (prior) and $\signal \mapsto \DataModel(\signal)$ (data model) are known, but $\signaltrue \in \RecSpace$ is unknown.
	The data model often has an associate density $\DataLikelihood$ (data likelihood) that is known to sufficient degree of accuracy, in which case $\dint \DataModel(\signal)(\data) = \DataLikelihood(\data \mid \signal) \dint\data$.

	\subsection{Reconstruction as an optimal decision rule}\label{sec:RecoDecisionRule}
	A reconstruction method is formally a measurable $\RecSpace$-valued mapping on $\DataSpace$, which in the statistical setting corresponds to an estimator.
	More precisely, $\bigl( (\DataSpace,\DataSpaceSAlg),\{ \DataModel(\signal) \}_{\signal \in \RecSpace} \bigr)$ defines a statistical model parametrized by the model parameter space $\RecSpace$ and a reconstruction method corresponds to a point estimator.
	The latter is a non-randomized decision rule for a statistical estimation problem where the model parameter space $\RecSpace$ parametrizes the underlying statistical model and at the same time constitutes the decision space.
	The reader may here consult \cite[section~3.1]{Liese:2008aa} for formal definitions of decision theoretic notions used here.
	
	There are many possible reconstruction methods (estimators) so one needs a framework where these can be compared against each other. 
	Statistical decision theory offers such a framework by associating a notion of risk to a decision rule.
	This quantifies the downside that comes with using a particular reconstruction method.
	The first step is to define the \emph{loss function} on the decision space, which in our specific setting becomes a measurable mapping (see \cite[Definition~3.2]{Liese:2008aa} for the definition in a general setting):
	\begin{equation}\label{eq:RecLoss}
		\SignalLoss \colon \RecSpace \times \RecSpace \to \Real.
	\end{equation}
	A common choice in imaging inverse problems is the $\Lebesgue^2$-loss, which is the squared $\Lebesgue^2$-distance. 
	There are however alternatives that are not based on point-wise differences but on differences between high-level image features, e.g., the Wasserstein distance \cite{Adler:2017ab} and perceptual losses \cite{Johnson:2016aa}.
	
	Having selected a loss function as in \cref{eq:RecLoss} and a prior $\SignalPrior$ in \cref{eq:BayesianSetting} on the model parameter space, the $\SignalPrior$-average risk (Bayes risk or expected loss) for reconstruction is given as 
	\begin{equation}\label{eq:BayesRiskReco}
		\AverageRisk_{\SignalPrior}(\RecOp{}) 
		= \Expect_{\SignalPrior \otimes \DataModel(\signal)}\Bigl[ \SignalLoss\bigl(\stsignal,\RecOp{}(\stdata) \bigr) \Bigr].
	\end{equation}
	A natural criteria to select a reconstruction method (estimator) is to minimize Bayes risk, i.e., to select an estimator (non-randomized decision rule) that minimizes  $\RecOp{} \mapsto \AverageRisk_{\SignalPrior}(\RecOp{})$ in \cref{eq:BayesRiskReco}.
	
	Note here that in the finite dimensional setting, minimizing Bayes risk is the same as computing the conditional mean (posterior mean) if and only if the loss function in \cref{eq:RecLoss} is the Bregman distance of a strictly convex non-negative differentiable functional \cite{Banerjee:2005aa}.
	This holds in particular when the loss function is given by the squared $\Lebesgue^2$-norm. 
	Next, another common choice is the \ac{MAP} estimator that maximizes the posterior, so it corresponds to the most likely reconstruction given the data.
	On the other hand, a maximum likelihood estimator maximizes the negative log-likelihood of data, i.e., it corresponds to the model parameter that generates the most likely data. This is an unsuitable estimator in ill-posed inverse problems since it frequently leads to overfitting. 

	To summarize, we will henceforth consider a reconstruction method that minimizes Bayes risk and, as already mentioned, this equals the conditional mean when using a $\Lebesgue^2$-loss. 
	
	\subsection{Challenges with Bayesian inversion}\label{sec:TheoryStatReg}
	In the Bayesian setting (\cref{sec:BayesianInversion}), both the true model parameter and data are assumed to be generated by random variables, and the goal is to recover the conditional probability of the model parameter given data (posterior) \cite{Kaipio:2005aa,Evans:2002aa,Stuart:2010aa,Dashti:2016aa,Calvetti:2017aa}.
	In contrast, classical (deterministic) approaches view an inverse problem as an operator equation of the type \cref{eq:InvProbOpEq} \cite{Engl:2000aa,Kaltenbacher:2008aa,Scherzer:2009aa,Kirsch:2011aa} where data may be generated by a random variable, but there are no statistical assumptions on model parameters.
	
	The Bayesian viewpoint offers a more complete analysis than the classical approach that is based on solving \cref{eq:InvProbOpEq} in the sense that the posterior describes all possible solutions.
	In particular, different reconstructions can be obtained by using different estimators and there is a natural framework for uncertainty quantification, e.g., by computing Bayesian credible sets.
	Furthermore, small changes in the data lead to small changes in the posterior distribution in a fairly general setting \cite[Theorem~16]{Dashti:2016aa} (continuity of the posterior distribution in the Hellinger metric), so working with probability measures on the model parameter space (posterior) and adopting a suitable prior stabilizes an ill-posed inverse problem.  
	
	The posterior is, on the other hand, often quite complicated with no closed form expression.
	Much of the contemporary research therefore focuses on realizing the above advantages with Bayesian inversion without having access to the full posterior.
	Key topics are designing a ``good'' prior $\SignalPrior \in \PClass{\RecSpace}$ and to have computationally feasible means for exploring the posterior.
	
	\subsubsection{Designing good priors}\label{sec:GoodPrior}
	The difficulty in selecting an appropriate prior lies in capturing the relevant a priori information. 
	Many of the results from the statistical community focus on characterizing priors that lead to Bayesian inference methods with desirable asymptotic properties, like consistency and good contraction rates.
	
	Bayesian non-parametric theory provides a large class of handcrafted priors, see, e.g., \cite[chapter~2]{Ghosal:2017aa}, \cite[section~2]{Dashti:2016aa}, and \cite{Kaipio:2005aa,Calvetti:2017aa}.
	These however only capture a fraction of the a priori information that is available.
	To illustrate this claim, a natural a priori information in medical imaging is that the object being imaged is a human being. 
	It is very difficult, if not impossible, to explicitly construct a prior that encodes this information. 
	
	An alternative approach is to consider a prior that is learned from examples in $\RecSpace$ through some predictive generative model. 
	A simplistic way is to select a Gaussian density that matches the first two sample moments \cite{Calvetti:2005aa}.
	More elaborate approaches can be based on generative adversarial networks that are trained on unsupervised data, e.g., a generative adversarial network can be used to learn a Gibbs type of prior in a \ac{MAP} estimator \cite{Lunz:2018aa}.

	\subsubsection{Computational feasibility}\label{sec:Computation}
	Exploring the posterior requires sampling from a high dimensional probability distribution.
	It is not possible to directly simulate from the posterior distribution in the infinite dimensional setting unless the model parameter is decomposed into more elementary finite-dimensional components. 
	This quickly becomes computationally challenging in large scale problems, like in imaging where the posterior is a probability distribution over the set of images.
	
	Computational methods used for Bayesian inversion often combine analytic approximations of the posterior with various \ac{MCMC} techniques, see \cite[section~5]{Dashti:2016aa} for a nice survey.
	There is an extensive theory that guarantees that these techniques are statistically consistent, but it comes with two critical drawbacks that has prevented   widespread usage of \ac{MCMC} techniques in imaging.
	First, many approaches require access to the prior in closed form, and as already argued for (\cref{sec:GoodPrior}), such handcrafted priors are woefully inadequate in representing natural images. 
	Second, these methods are still not sufficiently scalable for exploring the posterior in an efficient manner in large scale inverse problems, such as those that arise in 2D/3D tomographic imaging \cite[chapter~1]{Biegler:2011aa}.
	Alternatively, one can approximate the posterior with more tractable distributions (deterministic inference), which includes variational Bayes \cite{Fox:2012aa} and expectation propagation \cite{Minka:2001aa}.
	Variational Bayes methods have in particular emerged as a popular alternative to the classical \ac{MCMC} methods, see \cite{Blei:2017aa} for some guidance (on p.~860) on when to use \ac{MCMC} or variational Bayes.
	
	To summarise, one can sometimes with reasonable efficiency compute point estimators that do not involve any integration over the model parameter space, like a \ac{MAP} estimator.
	Estimators requiring such integration, like the estimator that minimize Bayes risk, are however computationally unfeasible. This also includes computational steps relevant for uncertainty quantification.
	
	\subsection{Learned iterative methods}\label{sec:LIR}
	As outlined in \cref{sec:TheoryStatReg}, there are two challenges associated with using Bayesian inversion: selecting a ``good'' prior (\cref{sec:GoodPrior}) and providing a computationally feasible approach for computing suitable estimators, like the one that minimizes Bayes risk (\cref{sec:Computation}).
	
	As we outline here, \emph{learned iterative methods address both these challenges}.
	It makes use of techniques from machine learning, and deep neural networks in particular, which have demonstrated a remarkable capacity in capturing intricate relations from example data \cite{LeCun:2015aa}.
	A key element is usage of highly parametrized generic models that can be adapted to specific decision rules, such as reconstruction by \cref{eq:BayesRiskReco}, by training against example data. 
	Learned iterative methods use a deep neural network to define an estimator (reconstruction method) that minimizes Bayes risk while accounting for the knowledge about how data is generated. 
	
	To give a more precise description, consider the joint law $\JointLawSignalData = \SignalPrior \otimes \DataModel(\signal)$ in \cref{eq:BayesRiskReco} used for defining Bayes risk. 
	In most practical applications, this joint law is unknown.
	Often one may however have access to the corresponding empirical measure given by supervised training data $(\signal_1,\data_1), \ldots, (\signal_m,\data_m) \in \RecSpace \times \DataSpace$ generated by $(\stsignal,\stdata) \sim \JointLawSignalData$.
	This \emph{avoids introducing a handcrafted prior} $\SignalPrior \in \PClass{\RecSpace}$.
	Furthermore, searching over all non-randomized decision rules is computationally unfeasible. 
	Instead, we restrict our attention to those given by a (deep) neural network architecture, which are known to have large capacity (can approximate any Borel measurable mapping arbitrarily well \cite{Pinkus:1999aa}) and there  are computationally feasible implementations.
	To summarize, we have a family of reconstruction methods $\RecOp{\MLsignalparam} \colon \DataSpace \to \RecSpace$ parametrized by a finite dimensional parameter set $\MLSignalParamSet$ and the optimal one is given by solving the training problem  
	\begin{equation}\label{eq:TrainRecLoss}
		\MLsignalparam^* \in 
		\argmin_{\MLsignalparam\in \MLSignalParamSet} \Bigl\{ \frac{1}{m} \sum_{i=1}^m \SignalLoss\bigl(\signal_i,\RecOp{\MLsignalparam}(\data_i)\bigr) \Bigr\}.
	\end{equation}
	
	The above approach for defining a reconstruction operator $\RecOp{\MLsignalparam^*} \colon \DataSpace \to \RecSpace$ is fully data driven in the sense that neither a prior on model parameter space nor a data model are handcrafted beforehand.
	Instead, all information is derived from the training data, which in particular does not utilize knowledge about how data is generated.
	This becomes a serious issue when the number of independent samples in training data are low compared to number of unknowns, which is commonly the case in imaging. 
	Next, in many inverse problem the data model $\signal \mapsto \DataModel(\signal)$ that describes how data is generated is known.
	Thus, it is unnecessarily pessimistic to disregard this information as in a fully data driven approach to reconstruction. 
	
	\emph{Learned iterative schemes} \cite{Adler:2017aa,Adler:2018aa} define a non-linear reconstruction operator parametrized by a deep convolutional neural network architecture that accounts for the data model, or more precisely the data likelihood. 
	The idea is to unroll a fixed point iterative scheme relevant for solving the inverse problem and replace the explicit iterative updating rule with a learned one given by a deep convolutional residual network.
	The approach can be formulated as a general scheme for solving (possibly non-linear) inverse problems \cite{Adler:2017aa,Adler:2018aa,Gupta:2018aa}, see also \cite{Mardani:2017aa,Mardani:2018aa,Diamond:2017aa} for a formulation that learns proximal updates in linear inverse problems.
	This results in a computationally feasible approach with surprisingly low requirements on training data and good generalization properties that outperforms state-of-the-art image reconstruction in \ac{CT} \cite{Adler:2017aa,Adler:2018aa,Gupta:2018aa}, \ac{MRI} \cite{Mardani:2017ab,Mardani:2017aa,Hammernik:2018aa,Mardani:2018aa}, photoacoustic tomography \cite{Hauptmann:2018aa}, and superresolution \cite{Mardani:2017ab,Mardani:2017aa,Diamond:2017aa}.
	
	\section{Tasks on model parameters}\label{sec:TaskModelParam}
	We consider tasks formulated as an operator that acts on model parameter space $\RecSpace$ and that takes values in a set $\DecisionSpace$ (decision space).
	We will start with the abstract formalization of such tasks using the language of statistical decision theory. 
	Similar to how reconstruction was treated (\cref{sec:RecoDecisionRule}), the task is represented by a non-randomized decision rule and we will select the one that minimizes Bayes risk. 
	Next, we also indicate how such decision rules can be computed efficiently using (deep) neural networks and supervised learning that minimizes the empirical risk.
	The remainder of the section is devoted to providing examples that concretizes the abstract framework and illustrates its general applicability.
	
	\subsection{Abstract setting}\label{sec:AbstractTask}
	Let $\bigl( (\RecSpace,\RecSpaceSAlg), \{ \TaskModel_{\task} \}_{\task \in \TaskSpace} \bigr)$ be a statistical model where the model parameter space $(\RecSpace,\RecSpaceSAlg)$ is a measurable space and $\{ \TaskModel_{\task} \}_{\task \in \TaskSpace} \subset \PClass{\RecSpace}$ is some family of probability measures on $\RecSpace$ parametrized by elements in $\TaskSpace$.
	Next, there is a measurable space $(\DecisionSpace,\DecisionSpaceSAlg)$ (decision space) and a fixed (task adapted) loss function (\cite[Definition~3.2]{Liese:2008aa})
	\begin{equation}\label{eq:TaskLoss}
		\TaskLoss \colon \TaskSpace \times \DecisionSpace \to \Real
		\quad\text{where}\quad
		\TaskLoss(\task,\decision) := \DecisionLoss\bigl(\featuremap(\task),\decision\bigr) 
	\end{equation}
	with given $\featuremap \colon \TaskSpace \to\DecisionSpace$ and $\DecisionLoss \colon \DecisionSpace \times \DecisionSpace \to \Real$ (decision distance).
	The statistical model along with the decision space and loss function defines a statistical estimation problem.
	Many tasks can now be seen as an appropriate non-randomized decision rule $\TaskOp{} \colon \RecSpace \to \DecisionSpace$ (task operator).
	
	Before proceeding, it is worth reflecting over the roles of the above sets. 
	In our set-up, the decision making associated with the task is based on elements in the decision space $\DecisionSpace$ whereas actual observables are elements in $\RecSpace$, so the task is represented by a measurable mapping $\TaskOp{} \colon \RecSpace \to \DecisionSpace$ (task operator).
	Often it is more natural to formalize the task as a mapping $\featuremap \colon \TaskSpace \to\DecisionSpace$ where elements in the set $\TaskSpace$ are related to those in $\RecSpace$.
	A difficulty is that elements in $\TaskSpace$ are not observable and the mapping relating its elements to those in $\RecSpace$ is unknown. 
	Hence, the challenge is to infer an appropriate mapping $\TaskOp{}$ given $\featuremap$ by resorting to some suitable ``optimality'' principle. 
	The examples in \cref{sec:ExampleTasks} will further clarify the various roles of these sets in decision making. 
	
	Just as in \cref{sec:RecoDecisionRule}, we consider a decision rule that minimizes Bayes risk.
	More precisely, assume $\TaskSpace$ is itself a measurable space and consider a fixed probability measure $\TaskPrior \in \PClass{\TaskSpace}$ (task prior).
	The \emph{task operator} is the non-randomized decision rule $\TaskOp{} \colon \RecSpace \to \DecisionSpace$ that minimizes the associated Bayes risk:
	\begin{equation}\label{eq:TaskAverageRisk}
		\AverageRisk_{\TaskPrior}(\TaskOp{}) 
		:= \Expect_{\TaskPrior \otimes \TaskModel_{\task}}\Bigl[ \DecisionLoss\bigl(\featuremap(\sttask),\TaskOp{}(\stsignal) \bigr) \Bigr]
		\quad\text{where $(\sttask, \stsignal) \sim \TaskPrior \otimes \TaskModel_{\task}$.}
	\end{equation}
	A difficulty is to provide a `reasonable' task prior $\TaskPrior \in \PClass{\TaskSpace}$.
	Another is that $\TaskModel_{\task} \in \PClass{\RecSpace}$ is not known.
	Hence, one needs to consider the joint law $\JointLawTaskSignal := \TaskPrior \otimes \TaskModel_{\task}$ in \cref{eq:TaskAverageRisk} as an unknown.
	Note that this differs from reconstruction, where the joint law is either known (as in \cref{sec:BayesianInversion}), or the prior is unknown but the data likelihood is known (as in \cref{sec:LIR}).
	Since the joint law is unknown, we replace it by the empirical measure given by (supervised) training data $(\task_1,\signal_1), \ldots, (\task_m,\signal_m) \in \TaskSpace \times \RecSpace$, i.e., one has i.i.d. samples generated by a $(\TaskSpace \times \RecSpace)$-valued random variable $(\sttask,\stsignal) \sim \JointLawTaskSignal$.
	Furthermore, due to issues associated with computational feasibility (\cref{sec:Computation}), we consider a parametrized family of decision rules $\TaskOp{\MLtaskparam} \colon \RecSpace \to \DecisionSpace$ given by a (deep) neural network architecture.
	Then, the task operator is the decision rule $\TaskOp{\MLtaskparam^*} \colon  \RecSpace \to \DecisionSpace$ parametrized by a finite dimensional parameter in $\MLTaskParamSet$ and the optimal one $\MLtaskparam^* \in \MLTaskParamSet$ is given by \emph{empirical risk minimization}:
	\begin{equation}\label{eq:TrainTaskLoss}
		\MLtaskparam^* \in 
		\argmin_{\MLtaskparam\in \MLTaskParamSet} \Bigl\{ \frac{1}{m} \sum_{i=1}^m \DecisionLoss\bigl( \featuremap(\task_i),\TaskOp{\MLtaskparam}(\signal_i) \bigr) \Bigr\}.
	\end{equation}
	
	We conclude with examples showing how a wide range of image processing tasks can be phrased as decision rules in a statistical decision problem.
	
	\subsection{Examples}\label{sec:ExampleTasks}
	The abstract framework in \cref{sec:AbstractTask} for formalizing a task on model parameter space is very generic and covers a wide range of possible tasks.
	In the following, we list concrete examples from imaging in order to show how this framework can be adapted to specific cases.
	To ensure a computational feasible implementation, our focus is on tasks that have been successfully addressed using techniques from deep learning.
	Deep learning has proven to be an efficient computational framework for many tasks, much thanks to its ability to progressively learn discriminative hierarchal features of the input data by means of training a suitable deep neural network.
	Hence, this limitation is not as restrictive as it may seem at a first glance, which will also become evident by the examples listed here and in \cref{sec:OtherTasks}.
	
	Unless otherwise stated, tasks are formulated for grey-scale images defined on a fixed domain $\signaldomain \subset \Real^n$, i.e., $\RecSpace := \Lebesgue^2(\signaldomain,\Real)$. 
	We will also assume that $\RecSpace$ is a measurable space for some $\sigma$-algebra $\RecSpaceSAlg$. 
	Finally, $\MeasurableMaps$ denotes the space of measurable mappings, e.g., $\MeasurableMaps(\RecSpace,\DecisionSpace)$ is $\DecisionSpace$-valued measurable mappings defined on $\RecSpace$.
	
	\subsubsection{Classification}\label{sec:Classification}
	The task is to classify an image into one of $k$ distinct labels, or more precisely, associate an image to a probability distribution over all $k$ labels.
	This task is represented by a non-randomized decision rule in a statistical estimation problem where $\TaskSpace := \Integer_k$ and the decision space $\DecisionSpace := \PClass{\TaskSpace}$ is probability distributions over the $k$ labels. 
	The task adapted loss function is given by \cref{eq:TaskLoss} with  
	\[
	\DecisionLoss(\decision,\decisionother) 
	:= - \sum_{\task \in \TaskSpace} \decision(\task) \log \decisionother(\task)
	\text{ for $\decision, \decisionother \in \DecisionSpace$}
	\quad\text{and}\quad
	\featuremap(\task) 
	:= \delta_{\task} \text{ for $\task \in \TaskSpace$.}
	\]
	Bayes risk in \cref{eq:TaskAverageRisk} associated with a decision rule $\TaskOp{} \colon \RecSpace \to \DecisionSpace$ for given task prior $\TaskPrior \in \PClass{\TaskSpace}$ becomes 
	\[
	\AverageRisk_{\TaskPrior}(\TaskOp{}) 
	:= \Expect_{\TaskPrior \otimes \TaskModel_{\task}}\Bigl[ \DecisionLoss\bigl(\featuremap(\sttask),\TaskOp{}(\stsignal) \bigr) \Bigr]
	= \int_{\RecSpace} \int_{\TaskSpace}  
	\Bigl[ - \log\bigl[ \TaskOp{}(\signal)(\task) \bigr] \Bigr] 
	\dint\TaskPrior(\task) \dint \TaskModel_{\task}(\signal).
	\]  
	The corresponding empirical risk minimization in \cref{eq:TrainTaskLoss} is  
	\begin{equation}\label{eq:TrainingClass}
		\MLtaskparam^* 
		\in \argmin_{\MLtaskparam\in \MLTaskParamSet} 
		\biggl\{ \frac{1}{m} \sum_{i=1}^m \Bigl[ - \log\bigl[ \TaskOp{}(\signal_i)(\task_i) \bigr] \biggr\}
		\quad\text{for training data $(\task_i,\signal_i) \in \TaskSpace \times \RecSpace$.}
	\end{equation}
	
	There are several papers dealing with how to construct a suitable (deep) neural network architecture for the set of decision rules $\DecisionClass = \{ \TaskOp{\MLtaskparam} \}_{\MLtaskparam\in\Real^N}$ and solving \cref{eq:TrainingClass} will then correspond to training a classifier, see \cite{LeBoBeHa98} for an early approach based on a convolutional neural network, AlexNet  \cite{Krizhevsky:2012aa} and ResNet \cite{HeZhReSu16} represent examples of further development along this line.
	
	\subsubsection{Semantic segmentation}\label{sec:Segmentation}
	The task here is to classify each point in an image into one of $k$ possible labels, so the special case $k=2$ corresponds to (binary) segmentation. 
	Stated more formally, semantic segmentation applies a mapping that associates each point in an image in $\RecSpace$ to a probability distribution over all $k$ labels.
	
	This task becomes a non-randomized decision rule in a statistical estimation problem where $\TaskSpace := \MeasurableMaps(\signaldomain,\Integer_k)$ and the decision space $\DecisionSpace := \MeasurableMaps(\signaldomain, \PClass{\Integer_k})$ is the set of measurable mappings from $\signaldomain$ to the class of probability measures on $\Integer_k$. 
	The task adapted loss function is given by \cref{eq:TaskLoss} with  
	\begin{align*}
		\DecisionLoss(\decision,\decisionother) 
		&:= \int_{\signaldomain} \Bigl[ - \sum_{i \in \Integer_k} \decision(t)(i) \log\bigl[ \decisionother(t)(i) \bigr] \Bigr]\dint t
		\quad\text{for $\decision, \decisionother \colon \signaldomain \to \PClass{\Integer_k}$,} \\
		\featuremap(\task)(t) 
		&:= \delta_{\task(t)} \text{ for $\task \colon \signaldomain \to \Integer_k$ and $t \in \signaldomain$.}
	\end{align*}
	The decision distance $\DecisionLoss \colon \DecisionSpace \times \DecisionSpace \to \Real$ simply integrates the point-wise cross entropy of the (point-wise) independent  probability measures $\decision(t)$ and $\decisionother(t)$.
	The cross entropy is a well-known notion from information theory for quantifying the dissimilarity between probability distributions \cite{Csiszar:2008aa} and it is often used as a learning objective in generative models involving probability distributions. 
	
	Bayes risk in \cref{eq:TaskAverageRisk} associated with a decision rule $\TaskOp{} \colon \RecSpace \to \DecisionSpace$ for a given task prior $\TaskPrior \in \PClass{\TaskSpace}$ can then be written as  
	\begin{align*}
		\AverageRisk_{\TaskPrior}(\TaskOp{}) 
		&:= \Expect_{\TaskPrior \otimes \TaskModel_{\task}}\Bigl[ \DecisionLoss\bigl(\featuremap(\sttask),\TaskOp{}(\stsignal) \bigr) \Bigr]
		\\
		&=  \int_{\RecSpace} \biggl[ 
		\int_{\TaskSpace} \biggl[ \int_{\signaldomain} -\log\Bigl[ \TaskOp{}(\signal)(t)\bigl(\task(t) \bigr) \Bigr] \dint t \biggr]
		\dint\TaskPrior(\task)
		\biggr] \dint \TaskModel_{\task}(\signal).
	\end{align*}
	The corresponding empirical risk minimization in \cref{eq:TrainTaskLoss} is  
	\begin{equation}\label{eq:TrainingSeg}
		\MLtaskparam^* 
		\in \argmin_{\MLtaskparam\in \MLTaskParamSet} 
		\biggl\{ \frac{1}{m} \sum_{i=1}^m \int_{\signaldomain} -\log\Bigl[ \TaskOp{\MLtaskparam}(\signal_i)(t)\bigl(\task_i(t) \bigr) \Bigr] \dint t \biggr\}.
	\end{equation}
	Note that $\TaskOp{}(\signal)(t)$ is a probability distribution over $\Integer_k$ and $\task(t) \in \Integer_k$ when $\task \in \TaskSpace$, so in particular $\TaskOp{}(\signal)(t)\bigl(\task(t) \bigr) \in [0,1]$ for any $t \in \signaldomain$.
	
	The set of decision rules $\DecisionClass = \{ \TaskOp{\MLtaskparam} \}_{\MLtaskparam\in\Real^N}$ can be parametrized by (deep) neural networks, in which case solving \cref{eq:TrainingSeg} corresponds to training a segmentation operator. 
	Deep neural net architectures suitable for semantic segmentation are presented in \cite{Long:2015aa,Noh:2015aa,Saito:2016aa}, see also the surveys in \cite{Thoma:2016aa,Guo:2018aa}. 
	In particular, the SegNet architecture has been successful for semantic segmentation of 2D images \cite{Badrinarayanan:2017aa}. 
	For (binary) segmentation one may use the U-net \cite{Ronneberger:2015aa,Cicek:2016aa}.
	
	\subsubsection{Anomaly detection}\label{sec:Comparison}
	The task here is to detect the difference (anomaly) between two grey-scale images, so $\RecSpace = \Lebesgue^2(\signaldomain,\Real) \times \Lebesgue^2(\signaldomain,\Real)$ for a fixed domain $\signaldomain \subset \Real^n$.
	This becomes a non-randomized decision rule in a statistical estimation problem where $\TaskSpace := \RecSpace$ and the anomaly is represented by grey-scale images, so the decision space is $\DecisionSpace := \Lebesgue^2(\signaldomain,\Real)$.
	The task adapted loss function is given by \cref{eq:TaskLoss} with  
	\[
	\DecisionLoss(\decision,\decisionother) 
	:= \bigl\Vert \decision - \decisionother \bigr\Vert_{\TaskSpace}^2
	\text{ for $\decision, \decisionother \in \DecisionSpace$}
	\quad\text{and}\quad
	\featuremap(\task) 
	:= \task_1 - \task_2  \text{ for $\task=(\task_1,\task_2) \in \TaskSpace$.}
	\]
	Bayes risk in \cref{eq:TaskAverageRisk} associated with a decision rule $\TaskOp{} \colon \RecSpace \to \DecisionSpace$ for given task prior $\TaskPrior \in \PClass{\TaskSpace}$ becomes 
	\begin{align*}
		\AverageRisk_{\TaskPrior}(\TaskOp{}) 
		&:= \Expect_{\TaskPrior \otimes \TaskModel_{\task}}\Bigl[ \DecisionLoss\bigl(\featuremap(\sttask),\TaskOp{}(\stsignal) \bigr) \Bigr]
		\\
		&=  \int_{\RecSpace} \biggl[ 
		\int_{\TaskSpace} \Bigl\Vert (\task_1 - \task_2) - \TaskOp{}(\signal_1,\signal_2) \Bigr\Vert_{\DecisionSpace}^2 
		\dint\TaskPrior(\task)
		\biggr] \dint \TaskModel_{\task}(\signal_1,\signal_2)
	\end{align*}
	and note that $\task=(\task_1,\task_2) \in \TaskSpace$ and $\signal = (\signal_1,\signal_2) \in \RecSpace$.
	The corresponding empirical risk minimization in \cref{eq:TrainTaskLoss} is  
	\begin{equation}\label{eq:TrainingDifference}
		\MLtaskparam^* 
		\in \argmin_{\MLtaskparam\in \Real^N} 
		\biggl\{ \frac{1}{m} \sum_{i=1}^m \Bigl\Vert (\signal^i_1 - \signal^i_2) - \TaskOp{\MLtaskparam}(\signal^i_1,\signal^i_2) \Bigr\Vert_{\DecisionSpace}^2 \biggr\}
	\end{equation}
	where $(\signal^i_1,\signal^i_2) \in \RecSpace$ is the supervised training data.
	One may here use a (deep) convolutional neural net to parametrize the decision rules in $\DecisionClass = \{ \TaskOp{\MLtaskparam} \}_{\MLtaskparam\in\Real^N}$.

	\subsubsection{Caption generation}\label{sec:CaptionGeneration}
	Caption generation refers to the task of associating an image to an appropriate sentence, or paragraph, that describes its content. 
	More precisely, define $\WordSpace$ to be the set of words in a chosen language, enlarged by a ``stop word'' $\stopword$ that marks the end of the caption. 
	Next, let $\CaptionSpace$ denote the set of captions, which are finite sequences made up of elements from $\WordSpace$ where each sequence contains $\stopword$ exactly once as its last element.  
	Then, caption generation is the task of mapping an image to a sequence in $\CaptionSpace$.
	
	This task becomes a non-randomized decision rule in a statistical estimation problem where $\TaskSpace := \CaptionSpace$ is the set of captions, so $\TaskModel_{\task}$ is the probability of a caption $\task$. The decision space is $\DecisionSpace := \PClass{\CaptionSpace}$, i.e., probability distributions over set of captions $\CaptionSpace$, and the task adapted loss function is given by \cref{eq:TaskLoss} with  
	\begin{align*}
		\DecisionLoss(\decision,\decisionother) 
		&:= -\sum_{c \in \CaptionSpace} \decision(c) \log \decisionother(c)
		\quad\text{for $\decision, \decisionother \in \DecisionSpace$,} \\
		\featuremap(\task) 
		&:= \delta_{\task} \in \PClass{\CaptionSpace} \quad\text{for a caption $\task \in \TaskSpace$.}
	\end{align*}
	
	To express Bayes risk that is to be minimized when computing the optimal decision rule, consider an element $\task \in \TaskSpace = \CaptionSpace$ (caption) and let $\task_i \in \WordSpace$ denote its $i$:th word. 
	Next, let $\WordSpace_n \subset \WordSpace$ denote the set of sequences of $n$ words that do not contain the stop word $\stopword$ (unfinished sentences), i.e., 
	\[ \WordSpace_n := \bigr\{ ( \word_i )_{i=1}^n \subset \WordSpace \mid \word_i \neq \stopword \text{ for $i=1,\ldots,n$} \bigr\}. \]
	Since an element in the decision space $\decision \in \DecisionSpace := \PClass{\TaskSpace}$ is a probability measure on sequences of words (captions), it will in particular yield the following probability measure on $\WordSpace_n$:
	\[
	\decision_{n}\bigl( ( \word_i )_{i=1}^n \bigr) := \decision\bigl( \{ \task \in \TaskSpace \mid \task_i = \word_i \text{ for $i=1,\ldots,n$} \}\bigr)
	\quad\text{for $( \word_i )_{i=1}^n \in \WordSpace_n$.}
	\]
	We now consider the measure for $n+1$ sequences that are conditioned on its first $n$ terms, i.e., let $( \word_i )_{i=1}^n \in \WordSpace_n$ be fixed with $\decision_{n}\bigl( ( \word_i )_{i=1}^n \bigr) > 0$.
	Then, $\decision$ induces to a probability measure $\pi_{\decision} \in \PClass{\WordSpace}$ on the set of words $\WordSpace$ via
	\[
	\pi_{\decision}\bigl( \word \mid ( \word_i )_{i=1}^n \bigr) := \frac{\decision_{n+1}(\word_1, \ldots, \word_n, \word)}{\decision_{n}(\word_1, \ldots, \word_n)}
	\quad\text{for $\word \in \WordSpace$ and $( \word_i )_{i=1}^n \in \WordSpace_n$.}
	\]
	With this notation, one can identify an element $\decision \in \PClass{\TaskSpace}$ with its corresponding representation in $\bigtimes_n \PClass{\WordSpace}(\,\cdot \mid \WordSpace_n)$ by the identity
	\[
	\decision(\task) = \prod_i \pi_{\decision}\bigl(\task \mid( \task_1, \ldots, \task_{n-1}) \bigr).
	\]
	Here $\PClass{\WordSpace}(\,\cdot \mid \WordSpace_n)$ is the set of probability measures on $\WordSpace$ conditioned on $\WordSpace_n$.
	
	Bayes risk in \cref{eq:TaskAverageRisk} associated with a decision rule $\TaskOp{} \colon \RecSpace \to \DecisionSpace$ (potential caption generation operator) for given task prior $\TaskPrior \in \PClass{\TaskSpace}$ can now be expressed through conditional densities on $\WordSpace$:
	\begin{align*}
		\AverageRisk_{\TaskPrior}(\TaskOp{}) 
		&:=  \int_{\TaskSpace} \int_{\RecSpace} \DecisionLoss\bigl(\featuremap(\task),\TaskOp{}(\signal)\bigr)
		\dint \TaskModel_{\task}(\signal) \dint \TaskPrior(\task)
		\\     
		&= \int_{\TaskSpace} \int_{\RecSpace}  - \log \TaskOp{}(\signal)(\task) \dint \TaskModel_{\task}(\signal) \dint \TaskPrior(\task)
		\\
		&= \int_{\TaskSpace} \int_{\RecSpace} - \log \prod_i \pi_{\TaskOp{}(\signal)}(\task \mid \task_1,\ldots, \task_{i-1}) \dint\TaskModel_{\task}(\signal) \dint \TaskPrior(\task)
		\\
		&= \int_{\TaskSpace} \int_{\RecSpace} - \sum_i \log \pi_{\TaskOp{}(\signal)}(\task \mid \task_1,\ldots, \task_{i-1}) \dint\TaskModel_{\task}(\signal) \dint \TaskPrior(\task).
	\end{align*}
	To derive the corresponding empirical risk minimization problem, we consider a fixed class of decision rules that are given via a parametrization of their representation in term of marginal densities. 
	More precisely, given a parameter set $\MLtaskparam \in \Real^m$, the decision rule $\TaskOp{\MLtaskparam}$ is given by
	\[ 
	\TaskOp{\MLtaskparam}(\signal)(\task) = \prod_i \Psi_{\MLtaskparam}(\signal; \ \task \mid \task_1 \ldots ,\task_{i-1})
	\quad\text{where $\signal \in \RecSpace$ and $\task \in \TaskSpace$,}
	\]
	and $\Psi_{\MLtaskparam}$ is parametrized by a recurrent neural network, for example using the \ac{LSTM} architecture \cite{Hochreiter:1997}.
	The corresponding empirical risk minimization in \cref{eq:TrainTaskLoss} is now given by  
	\begin{equation}\label{eq:TrainingCaptionGen}
		\MLtaskparam^* 
		\in \argmin_{\MLtaskparam\in \Real^N} 
		\biggl\{
		\int_{\TaskSpace} \int_{\RecSpace} - \sum_i \log \Psi_{\MLtaskparam}(\signal; \ \task \mid \task_1 \ldots ,\task_{i-1})
		\dint \TaskModel_{\task}(\signal) \dint\TaskPrior(\task)
		\biggr\}.
	\end{equation}
	Solving \cref{eq:TrainingCaptionGen} corresponds to training a image caption generator \cite{Vinyals:2015}, see also \cite{Karpathy:2017aa}.
	
	\subsection{Further imaging tasks}\label{sec:OtherTasks}
	A common trait with the examples worked out in \cref{sec:ExampleTasks} is that each of them can be recast as finding an optimal decision rule in a statistical estimation problem . 
	Furthermore, deep neural networks offer a computationally feasible implementation of estimating this decision rule from supervised training data. 
	
	There is a wide range of other tasks beyond those mentioned in \cref{sec:ExampleTasks} that can be represented as a non-randomised decision rule, which in turn is efficiently parametrized by a suitable deep neural network architecture. 
	The (incomplete) list below is based on \cite{Ahmed:2018aa,Prince:2012aa} and aims to show the diversity of tasks from computer vision that can be approached successfully using a suitable deep neural network architecture.
	\begin{description}
		\item[Inpainting:]
		This is essentially interpolation/extrapolation to recover lost or deteriorated parts of images and videos and approaches based on trainable neural networks 
		\cite{Xie:2012aa}.
		\item[Depixelization/super-resolution:]
		The task is here to upsample, i.e., to synthesize realistic details into images while enhancing their resolution \cite{Dahl:2017aa} or to fill out information ``between'' pixels by increasing the resolution of the final picture, also known as the single image super-resolution problem \cite{Romano:2017ab}. 
		\item[Demosaicing:]
		The task here is to reconstructing a full color image from the incomplete color samples output from an image sensor \cite{Syu:2018aa}.
		Almost all digital cameras, ranging from smartphone cameras to the top-of-the-line digital SLR cameras, use a demosaicing algorithm to convert the captured sensor information into a color image.
		\item[Colourising:]
		The task is to apply color to grey scale photos and videos \cite{Iizuka:2016aa}.
		\item[Image translation:]  
		The task is to translate between two classes of images of the same object, e.g., \ac{CT} and \ac{MRI} images \cite{Wolterink:2017aa}.
		\item[Object recognition:]  
		This visual classification task involves localization, detection and classification and this can be seen as an example of constellation models, which are a general class of model that describe objects in terms of a set of parts and their spatial relations \cite[section~20.5]{Prince:2012aa}.
		An integrated framework based on deep convolutional networks for detection and localizatio was already introduced in \cite{Sermanet:2013aa}, see also \cite{He:2016aa} for recognition and \cite{Farabet:2013aa} for scene understanding.
		The most well-known use case is recognition of multiple faces in an image where statistical shape models play a central role \cite{Zhao:2003aa,Druzhkov:2016aa,Wang:2018aa} 
		An analogous task relevant for clinical image guided diagnostics is detecting melanoma from images of skin lesions \cite{Esteva:2017aa,Haenssle:2018aa}.
		\item[Non-rigid image registration:]  
		The task here is to deform a template image in a ``natural'' way so that it matches a target image. This is a key step in spatiotemporal imaging and 
		deep neural networks have been utilized for this purpose \cite{Ghosal:2017ac,Yang:2017aa}.
		\item[Parametric regression:]  
		The task here is to statistically determine relationships among variables (parameters) in a statistical model. 
		This is important in clinical diagnostics where one seeks to determine risk factors and biomarkers of diagnostic and prognostic value associated with clinical progression and severity of specific diseases.
		An example of image guided regression is estimating cardiovascular risk factors, such as age, from retinal fundus photographs  \cite{Poplin:2017aa}.
		Another is predicting patient overall survival as in the 2018 Multimodal Brain Tumor Segmentation Challenge based on BRATS imaging data \cite{Menze:2015aa}, and predicting scores for Alzheimer's disease from imaging data \cite{McCrackin:2018aa,Lee:2018aa}.
	\end{description}
	
	The abstract framework in \cref{sec:TaskAdaptedAbstract} allows one to perform \emph{any} of the above tasks (and those in \cref{sec:ExampleTasks}) \emph{jointly} with a reconstruction step for solving an inverse problem.
	Examples of the latter are deconvolution, Fourier inversion (\ac{MRI} imaging), or more elaborate schemes for various types of tomographic imaging. 
	A key success factor is access to suitable training data, another is usage of a differentiable loss function along with a trainable differentiable parametrization (deep neural network architecture) of the task operator.

	\section{Task adapted reconstruction}\label{sec:TaskAdaptedRec} 
	As stated in the introduction (\cref{sec:Intro}), solving an inverse problem (reconstruction) is one step in a pipeline (\cref{fig:Pipeline}) that often involves further coupled tasks necessary for decision making.
	\emph{Task adapted reconstruction} refers to methods that integrate the reconstruction with (parts of) the decision making procedure, thereby adapting the reconstruction method for the specific task at hand.
%
%
%
	
	\subsection{Abstract setting}\label{sec:TaskAdaptedAbstract}
	We will present a generic framework for task adapted reconstruction that is computationally feasible \emph{and} adaptable to specific inverse problems and tasks.
	A key element is to formalize both the reconstruction and task as non-randomized decision rules within a statistical estimation problem.
	
	More precisely, our starting point is the inverse problem in \cref{sec:InvProb} where the data model in \cref{eq:DataModel} is known. Following \cref{sec:RecoDecisionRule}, the reconstruction step can be seen as a decision rule in a statistical estimation problem defined by the statistical model $\bigl( (\DataSpace,\DataSpaceSAlg),\{ \DataModel(\signal) \}_{\signal \in \RecSpace} \bigr)$, decision space $(\RecSpace,\RecSpaceSAlg)$, and loss $\SignalLoss \colon \RecSpace \times \RecSpace \to \Real$.
	Given a prior $\SignalPrior \in \PClass{\RecSpace}$ allows us to define a reconstruction method as a non-randomized decision rule that minimizes the $\SignalPrior$-average risk (Bayes risk), i.e., as a mapping that solves
	\begin{equation}\label{eq:RecOpDecision}  
		\RecOpOptim \in \!\!
		\argmin_{\RecOp{} \in \MeasurableMaps(\DataSpace,\RecSpace)} 
		\biggl\{
		\Expect_{\SignalPrior \otimes \DataModel(\signal)}
		\Bigl[ 
		\SignalLoss\bigl(\stsignal,\RecOp{}(\stdata) \bigr)
		\Bigr]
		\biggr\}
		\quad\text{where $(\stsignal,\stdata) \sim \SignalPrior \otimes \DataModel(\signal)$.}
	\end{equation}
	Likewise, following \cref{sec:AbstractTask} a task becomes a decision rule in a statistical estimation problem defined by the statistical model $\bigl( (\RecSpace,\RecSpaceSAlg), \{ \TaskModel_{\task} \}_{\task \in \TaskSpace} \bigr)$, decision space $(\DecisionSpace,\DecisionSpaceSAlg)$, and loss given by \cref{eq:TaskLoss} with known $\featuremap \colon \TaskSpace \to\DecisionSpace$ (feature extraction map) and $\DecisionLoss \colon \DecisionSpace \times \DecisionSpace \to \Real$ (decision distance).
	Given a task prior $\TaskPrior \in \PClass{\TaskSpace}$, the non-randomized decision rule representing the task (task operator) can be seen as a minimizer to the $\TaskPrior$-average risk (Bayes risk): 
	\begin{equation}\label{eq:TaskOpDecision}  
		\TaskOpOptim \in 
		\argmin_{\TaskOp{} \in \MeasurableMaps(\RecSpace,\DecisionSpace)} 
		\biggl\{
		\Expect_{\TaskPrior \otimes \TaskModel_{\task}} 
		\Bigl[ 
		\DecisionLoss\bigl(\featuremap(\sttask),\TaskOp{}(\stsignal) \bigr) 
		\Bigr]
		\biggr\}
		\quad\text{where $\bigl(\sttask,\stsignal \bigr) \sim \TaskPrior \otimes \TaskModel_{\task}$.}
	\end{equation}
	We now have the following three approaches to task adapted reconstruction.
	\begin{description}
		\item[Sequential approach:]
		A sequential approach starts with determining the reconstruction operator, and then uses it to define the the task operator. It is based on assuming that statistical assumptions for $(\sttask,\stsignal) \sim \TaskPrior \otimes \TaskModel_{\task}$ and $(\stsignal,\stdata) \sim \SignalPrior \otimes \DataModel(\signal)$ are consistent, e.g., by assuming that $\TaskModel_{\task}$ is the push forward of $\DataModel(\signal)$ through the reconstruction operator\footnote{One may here consider alternate assumptions, like assuming that 
			$\SignalPrior \in \PClass{\RecSpace}$ can be obtained by marginalizing the measure $\TaskPrior \otimes \TaskModel_{\task}$ over $\TaskSpace$ using $\TaskPrior \in \PClass{\TaskSpace}$.}.
		The task adapted reconstruction is then given as
		\begin{equation}\label{eq:JointOpGen}  
			\TaskOpOptim \circ \RecOpOptim \colon \DataSpace \to \DecisionSpace 
		\end{equation}
		where $\RecOpOptim \in \MeasurableMaps(\DataSpace,\RecSpace)$ solves \cref{eq:RecOpDecision} and $\TaskOpOptim \in \MeasurableMaps(\RecSpace,\DecisionSpace)$ solves \cref{eq:TaskOpDecision}.
		
		\item[End-to-end approach:]
		The fully end-to-end approach ignores the distinction between reconstruction and the task. Assuming $(\sttask,\stdata) \sim \JointLawTaskData$ for some measure $\JointLawTaskData \in \PClass{\TaskSpace \times \DataSpace}$, the task adapted reconstruction is here given as 
		$\widehat{\EndToEnd} \colon \DataSpace \to \DecisionSpace$
		that solves
		\begin{equation}\label{eq:EndToEndApproach}  
			\widehat{\EndToEnd} \in 
			\argmin_{\EndToEnd \in \MeasurableMaps(\DataSpace,\DecisionSpace)} 
			\biggl\{
			\Expect_{\TaskPrior \otimes \TaskModel_{\task}} 
			\Bigl[ 
			\DecisionLoss\bigl(\featuremap(\sttask),\EndToEnd(\stdata) \bigr) 
			\Bigr]
			\biggr\}
			\quad\text{where $(\sttask,\stsignal) \sim \JointLawTaskData$.}
		\end{equation}
		
		\item[Joint approach:]\label{sec:JointApproach}
		The joint approach is a middle-way between the sequential and end-to-end approaches. It is based on assuming that there is a joint law $(\sttask,\stsignal,\stdata) \sim \JointLawTaskSignalData$, which by the chain rule in probability can be written in terms of conditional probabilities: 
		\[
		\dint \JointLawTaskSignalData(\task, \signal, \data)
		=
		\dint \CondLaw(\data \mid \task, \signal)
		\dint \CondLaw(\signal \mid \task)
		\dint \CondLaw(\task).
		\]
		Assume that $\signal$ is a sufficient statistic for $\data$, i.e., $\dint \CondLaw(\data \mid \task, \signal) = \dint \CondLaw(\data \mid \signal)$.
		Combined with $(\sttask,\stsignal) \sim \TaskPrior \otimes \TaskModel_{\task}$ and $(\stsignal,\stdata) \sim \SignalPrior \otimes \DataModel(\signal)$, we obtain 
		\[
		\dint \JointLawTaskSignalData(\task, \signal, \data)
		=
		\dint \DataModel(\signal)(\data)
		\dint \TaskModel_{\task}(\signal)
		\dint \TaskPrior(\task).
		\]
		Next, we introduce a \emph{joint loss} that interpolates between the sequential case and the end-to-end case. Specifically, we let $\JointLoss \colon (\RecSpace \times \DecisionSpace) \times  (\RecSpace \times \DecisionSpace) \to \Real$ be given as 
		\begin{equation}\label{eq:JointLoss}
			\JointLoss\bigl((\signal,\decision),(\signalother,\decisionother)\bigr) 
			:= (1 - C) \SignalLoss(\signal,\signalother) + C \DecisionLoss(\decision,\decisionother)
			\quad\text{for fixed $C \in [0, 1]$.}
		\end{equation}
		Task adapted reconstruction is now given by \cref{eq:JointOpGen} where the operators jointly solve 
		\begin{equation}\label{eq:JointApproach}
			(\RecOpOptim,\TaskOpOptim) 
			\in \!\!
			\argmin_{\substack{\TaskOp{} \in \MeasurableMaps(\RecSpace,\DecisionSpace) \\ \RecOp{} \in \MeasurableMaps(\DataSpace,\RecSpace)}} 
			\Expect_{\JointLawTaskSignalData}
			\Bigr[
			\JointLoss\Bigl( 
			\bigl(\stsignal,\featuremap(\sttask)\bigr),
			\bigl(\RecOp{}(\stdata),\TaskOp{} \circ \RecOp{}(\stdata)\bigr)
			\Bigr)
			\Bigr].
		\end{equation}
	\end{description}
	Note first that in the limit $C \to 0$, the joint approach becomes the sequential one. 
	Next, it may seem sufficient to only consider the loss $\DecisionLoss$ in \cref{eq:JointApproach}, i.e., to set $C=1$ in \cref{eq:JointLoss}, which recovers the end-to-end approach. 
	There is however a problem with non-uniqueness in this case since 
	\begin{equation}\label{eq:JointRegEffect}
		(\RecOpOptim,\TaskOpOptim) \text{ solves \cref{eq:JointApproach}} 
		\implies
		(\OpB^{-1}\circ \RecOpOptim, \TaskOpOptim \circ \OpB) \text{ solves \cref{eq:JointApproach} for any invertible $\OpB \colon \RecSpace \to \RecSpace$.}
	\end{equation}
	This non-uniqueness does not arise when $C < 1$, so incorporating a loss term associated with the reconstruction may act as a regularizer.
	This also indicates that the limit $C \to 1$ does not necessarily coincide with the case $C = 1$.

	\subsection{Computational implementation}\label{sec:TaskCompImpl}
	In \cref{sec:GoodPrior} we mentioned the difficulty to select an appropriate prior $\SignalPrior \in \PClass{\RecSpace}$ for Bayesian inversion whereas the measure $\DataModel(\signal) \in \PClass{\DataSpace}$ is often known by the data model. 
	Furthermore, both measures $\TaskPrior \in \PClass{\TaskSpace}$ and $\TaskModel_{\task} \in \PClass{\RecSpace}$ must be considered as unknown for most tasks.
	Hence any realistic scenario would contain $\JointLawTaskSignalData$ as an unknown.
	An option is to replace these measures by their empirical counterparts given by suitable supervised training data.
	
	Another concern is computational feasibility. 
	The optimizations in \cref{eq:RecOpDecision,eq:TaskOpDecision,eq:EndToEndApproach,eq:JointApproach} are taken over all measurable mappings between relevant spaces, which is computationally unfeasible. 
	This can be addressed by considering parametrized sets of measurable mappings as done in \cref{sec:LIR,sec:AbstractTask}.
	More precisely, we employ a learned iterative scheme to parametrize a family of reconstruction methods $\RecOp{\MLsignalparam} \colon \DataSpace \to \RecSpace$ since this parametrization includes knowledge about the data model (\cref{sec:LIR}). 
	Likewise, decision rules associated with the task are given by an appropriate parametrized family of mappings $\TaskOp{\MLtaskparam} \colon \RecSpace \to \DecisionSpace$.
	Finally, the approach for the end-to-end setting is to directly parametrize $\EndToEnd_{\MLtaskparam} \colon \DataSpace \to \DecisionSpace$.
	
	Using such parametrizations allows one to reformulate \cref{eq:RecOpDecision,eq:TaskOpDecision} as \cref{eq:TrainSequantialApproach}, \cref{eq:EndToEndApproach} as \cref{eq:TrainSequantialApproach}, and \cref{eq:JointApproach} as \cref{eq:TrainEndToEndApproach}.
	A key aspect for the implementation is to use \ac{SGD} based methods for finding appropriate parameters by approximately solving the empirical versions of \cref{eq:TrainSequantialApproach,eq:EndToEndApproach,eq:TrainEndToEndApproach}. 
	This requires that the above parametrizations are differentiable, which in particular requires using differentiable loss-functions. 
	
	Depending on the type of supervised training data, we can now pursue either of the three approaches listed in \cref{sec:TaskAdaptedAbstract}.
	\begin{description}
		\item[Sequential approach:]
		Here we have access to two sets of supervised training data that are coupled: 
		\begin{equation}\label{eq:TDataSeq}
			\begin{split}
				(\signal_i,\data_i) \in \RecSpace \times \DataSpace 
				&\quad\text{generated by $(\stsignal,\stdata) \sim \SignalPrior \otimes \DataModel(\signal)$ for $i=1,\ldots,m$}
				\\
				(\task_i,\signal_i) \in \TaskSpace \times \RecSpace 
				& \quad\text{generated by $(\sttask,\stsignal) \sim \TaskPrior \otimes \TaskModel_{\task}$ for $i=1,\ldots,m$.}
			\end{split}
		\end{equation}
		The coupling is that $\signal_i$'s in the second data set (bottom one) are reconstructions obtained from $\data_i$'s in the first data set (top one). 
		This ensures consistency with the statistical assumptions mentioned before for the sequential approach. 
		The reconstruction is then given by the mapping 
		\begin{equation}\label{eq:JointOpGenParam}  
			\TaskOp{\MLtaskparam^*} \circ \RecOp{\MLsignalparam^*} \colon \DataSpace \to \DecisionSpace
		\end{equation}
		where $\MLsignalparam^* \in \MLSignalParamSet$ solves \cref{eq:TrainRecLoss} and $\MLtaskparam^* \in \MLTaskParamSet$ solves \cref{eq:TrainTaskLoss}, i.e., 
		\begin{equation}\label{eq:TrainSequantialApproach} 
			\begin{split} 
				\MLsignalparam^* &\in 
				\argmin_{\MLsignalparam\in \MLSignalParamSet} \Bigl\{ \frac{1}{m} \sum_{i=1}^m \SignalLoss\bigl(\signal_i,\RecOp{\MLsignalparam}(\data_i)\bigr) \Bigr\}
				\\  
				\MLtaskparam^* &\in 
				\argmin_{\MLtaskparam\in \MLTaskParamSet} \Bigl\{ \frac{1}{m} \sum_{i=1}^m \DecisionLoss\bigl( \featuremap(\task_i),\TaskOp{\MLtaskparam}(\signal_i) \bigr) \Bigr\}.
			\end{split}
		\end{equation}
		Note here that the only requirement for the sequential approach is that the assumptions $(\sttask,\stsignal) \sim \TaskPrior \otimes \TaskModel_{\task}$ and $(\stsignal,\stdata) \sim \SignalPrior \otimes \DataModel(\signal)$ are jointly consistent.  
		The learned task operator given by solving for $\MLtaskparam^*$ in \cref{eq:TrainSequantialApproach} is only well defined for input taken from the support of it's training data, so it may fail when applied to data it has never seen. This is especially the case if new data has a different statistical assumption.
		Hence, it is important to ensure the range of the reconstruction operator is contained in the support of the elements $\signal \in \RecSpace$ used to train the task.
		In most practical implementations, this is ensured by simply letting $\signal_i$'s in $(\task_i,\signal_i)$ in \cref{eq:TDataSeq} be the output of the learned reconstruction operator $\RecOp{\MLsignalparam^*} \colon \RecSpace \to \DataSpace$.
		
		\item[End-to-end approach:]
		Supervised training data is here of the form 
		\begin{equation}\label{eq:TDataEndToEnd}
			(\task_i,\data_i) \in \TaskSpace \times \RecSpace  
			\quad\text{generated by $(\sttask,\stsignal) \sim \TaskPrior \otimes \TaskModel_{\task}$ for $i=1,\ldots,m$.}
		\end{equation}
		The reconstruction is given by 
		$\EndToEnd_{\MLtaskparam} \colon \DataSpace \to \DecisionSpace$
		where $\MLtaskparam^* \in \MLTaskParamSet$ solves  
		\begin{equation}\label{eq:TrainEndToEndApproach} 
			\MLtaskparam^* \in 
			\argmin_{\MLtaskparam\in \MLTaskParamSet} \Bigl\{ \frac{1}{m} \sum_{i=1}^m \DecisionLoss\bigl( \featuremap(\task_i),\EndToEnd_{\MLtaskparam}(\data_i) \bigr) \Bigr\}.
		\end{equation}
		
		\item[Joint approach:]
		In this approach we assume access to supervised training data that jointly involves the reconstruction and task: 
		\begin{equation}\label{eq:TDataJoint}
			(\signal_i,\data_i,\task_i) \in \RecSpace \times \DataSpace \times \TaskSpace
			\quad\text{generated by $(\stsignal,\stdata,\sttask) \sim \JointLawTaskSignalData$ for $i=1,\ldots,m$.}
		\end{equation}
		Given such data, the corresponding reconstruction method can be defined as in \cref{eq:JointOpGenParam} where $(\MLsignalparam^*,\MLtaskparam^*) \in \MLSignalParamSet \times \MLTaskParamSet$ solves the following joint empirical loss minimization: 
		\begin{equation}\label{eq:TrainJointApproach}
			\begin{split}
				(\MLsignalparam^*,\MLtaskparam^*) 
				&\in\!\!\!
				\argmin_{(\MLsignalparam,\MLtaskparam) \in \MLSignalParamSet \times \MLTaskParamSet} 
				\biggl\{ 
				\frac{1}{m}\sum_{i=1}^m
				\JointLoss\Bigl( 
				\bigl(\signal_i,\featuremap(\task_i)\bigr),
				\bigl(\RecOp{\MLsignalparam}(\stdata_i), \TaskOp{\MLtaskparam} \circ \RecOp{\MLsignalparam}(\stdata_i) \bigr)
				\Bigr)            
				\biggr\}      
			\end{split}
		\end{equation}
		where $\JointLoss \colon (\RecSpace \times \DecisionSpace) \times  (\RecSpace \times \DecisionSpace) \to \Real$ is the joint loss in \cref{eq:JointLoss}.
		Note that one may \emph{in addition} have access to separate sets of training data of the form \cref{eq:TDataSeq} \emph{and} \cref{eq:TDataJoint}.
		In such case, it is possible to first \emph{pre-train} by solving for \eqref{eq:TrainRecLoss} and \eqref{eq:TrainTaskLoss} separately, and use the resulting outcomes to initialize an algorithm for solving \cref{eq:TrainJointApproach}.
	\end{description}
	
	\subsection{Applications}\label{sec:Applications}
	In the following we demonstrate performance of the task adapted reconstruction scheme for \cref{eq:JointOpGenParam} that is based on solving \cref{eq:TrainJointApproach}.
	All cases involve tomographic reconstruction from 2D parallel beam data and as tasks, we consider MNIST classification and segmentation.
	
	\subsubsection{Joint tomographic reconstruction and classification}\label{sec:MNIST}
	\begin{description}
		\item[Task:] Recover probabilities that a 2D grey scale MNIST image is a \texttt{0},\texttt{1}, \ldots,\texttt{9}  from noisy parallel beam tomographic data (see \cref{sec:Classification}).
		\item[Data:]
		Elements in $\DataSpace$ are real-valued functions representing samples of a Poisson random variable with mean equal to the exponential of the parallel beam ray transform and an intensity corresponding to 60 photons/line. The ray transform is digitized by sampling the angular variable at 5~uniformly sampled points in $[0,\pi]$ with 25~lines/angle.
		\item[Model parameter space:] 
		Elements in $\RecSpace$ are functions representing images supported on a fixed rectangular region $\signaldomain \subset \Real^2$, so $\RecSpace:=\Lebesgue^2(\signaldomain,\Real)$.
		These are discretized by sampling on a uniform $28 \times 28$ grid.
		The loss $\SignalLoss \colon \RecSpace \times \RecSpace \to \Real$ is the squared $\Lebesgue^2$-distance on $\RecSpace$.
		\item[Decision space:] 
		$\TaskSpace := \{\text{\texttt{0}},\text{\texttt{1}}, \ldots, \text{\texttt{9}} \}$ is the set of labels and $\DecisionSpace$ is probability distributions over $\TaskSpace$ with a loss function $\DecisionLoss \colon \DecisionSpace \times \DecisionSpace \to \Real$ given by the cross entropy:  
		\[ 
		\DecisionLoss(\decision,\decisionother)    
		:= - \sum_{i \in \TaskSpace} \decision(i) \log\bigl[ \decisionother(i) \bigr] 
		\quad\text{for $\decision, \decisionother \in \PClass{\TaskSpace}$.}
		\]
		In addition to cross entropy, we employ \emph{classification accuracy} to measure performance. Given a probability distribution $\decision \in \DecisionSpace$ over $\TaskSpace$, the single label prediction is defined to be the element in $\TaskSpace$ that is assigned the highest probability, i.e. $\argmax_{\task \in \TaskSpace} \decision(\task)$. The percentage of images in the evaluation data set for which the predicted label coincides with the real one is reported as classification accuracy. 
		\item[Reconstruction and task operators:] 
		Reconstruction $\RecOp{\MLsignalparam} \colon \DataSpace \to \RecSpace$ is given by the learned gradient descent in \cite{Adler:2017aa} and task $\TaskOp{\MLtaskparam} \colon \RecSpace \to \DecisionSpace$ is a MNIST classifier given by a standard convolutional neural net classifier with three convolutional layers, each followed by $2 \times 2$ max pooling for segmentation. The activation functions used were ReLUs, layers had 32, 64 and 128 channels, respectively. 
		The final layer is dense and transforms the output of the last convolutional layer to a logit layer of size 10, with the last activation function being a softmax.
		\item[Joint training:] Joint supervised data is given as 512\,000 
		triplets $(\signal_i,\data_i,\task_i)$ where $\task_i \in\TaskSpace$ is the label corresponding to the MNIST labels.  
		We also performed pre-training for both the reconstruction and task operator (classifier). The reconstruction operator was pre-trained using 256\,000 pairs $(\signal_i,\data_i)$ with 8\,000 steps with a batch size of 64 and the task operator (classifier) was pre-trained until 97.7\% accuracy. 
		Note here that we use about 60\,000 entries from the MNIST database, so the above supervised data is not statistically 
		independent. 
	\end{description}
	Example outcomes, which are summarized in \cref{tab:JointRecoClass} and \cref{fig:JointRecoClass}, show that a joint approach outperforms a sequential one when considering the classification accuracy. 
	Besides an improved classification accuracy we also see a significant improvement regarding interpretability. 
	The reconstructed image part can in the joint setting actually be used as a benchmark to assess the reconstructed classification probabilities. 
	On the other hand, the sequential approach results in classification probabilities that deviates from this intuitive observation. 
	We also see that in both cases, the classification probabilities are unnaturally concentrated on a single label, but this is a know phenomena also for regular for MNIST classification \cite{GuGeSuWe17}.
	\begin{table}[ht]
		\centering
		\begin{tabular}{l l r r r}
			\toprule 
			Approach & \qquad  & Accuracy & $\Lebesgue^2$-loss & Cross entropy \\
			\otoprule
			Pre-training & & 93.61\% & 9.0 & 0.643 \\
			Sequential & & 96.01\% & 9.0 & 0.124 \\
			End-to-end & & 96.70\% & 19.7 & 0.118 \\
			Joint with $C = 0.01$ & & 96.74\% & 12.8 & 0.108 \\
			Joint with $C = 0.5$ & & 97.00\% & 9.2 & 0.100 \\
			Joint with $C = 0.999$ & & 96.61\% & 9.0 & 0.108 \\
			Classification on true images & & 97.76\% & &0.088 \\
			\bottomrule
		\end{tabular} 
		\caption{In both the pre-training and sequential approaches, the reconstruction and task operators are trained separately. 
			In the sequential approach, the task operator is then further trained on the output of the trained reconstruction operator.
			In the end-to-end approach, which corresponds to $C=0$ in \cref{eq:JointLoss}, the reconstruction operator is pre-trained with $\Lebesgue^2$-loss. 
			Finally, the joint approach uses the full loss \cref{eq:JointLoss}.
			We see that the classification accuracy (explained in ``Decision space'' in \cref{sec:MNIST}) improves when using a joint approach.
			In fact, using a ``suitable'' $C$ (\cref{fig:JointRecSegLossMNIST}) yields an accuracy of $97.00\%$ that is quite close to the upper limit of $97.76\%$, which is the accuracy of the classifier when trained on true images.}\label{tab:JointRecoClass}
	\end{table} 
	
	
	\begin{figure}
		\centering
		\begin{subfigure}[t]{\linewidth}
			\centering
			\begin{tikzpicture}
			\begin{axis}[
			height=0.16\textheight,
			width=0.9\linewidth,
			x label style={at={(axis description cs:1,0.35)},anchor=north,xshift=0.2cm},
			xlabel=$C$,
			ylabel style={yshift=0.1cm,rotate=-90},  
			ylabel=$\SignalLoss$,  
			xtick={0, 1, 2, 3, 4, 5},
			xticklabels={0.01, 0.1, 0.5, 0.9, 0.99, 0.999}
			]  
			\addplot table [y=SLoss, x=Loc, col sep=comma]{MNIST_data.txt};
			\end{axis}
			\end{tikzpicture} 
			\\[0.75em]
			\begin{tikzpicture}
			\begin{axis}[
			height=0.16\textheight,
			width=0.9\linewidth,
			x label style={at={(axis description cs:1,0.35)},anchor=north,xshift=0.2cm},
			xlabel=$C$,
			ylabel style={yshift=0.1cm,rotate=-90},  
			ylabel=$\DecisionLoss$,
			xtick={0, 1, 2, 3, 4, 5},
			xticklabels={0.01, 0.1, 0.5, 0.9, 0.99, 0.999}
			]  
			\addplot table [y=DLoss, x=Loc, col sep=comma]{MNIST_data.txt};
			\end{axis}
			\end{tikzpicture}
			\caption{Plot of loss functions after joint training for different $C$ in \cref{eq:JointLoss}.
				Clearly, there is no joint minimizer but $0.5 \lesssim C \lesssim 0.9$ is a good compromise.}\label{fig:JointRecSegLossMNIST}
		\end{subfigure}
		\begin{subfigure}[t]{0.3\linewidth}
			\centering
			\vspace{0pt}
			\includegraphics[height=0.145\textheight, trim={9.75mm 49mm 3.5mm 3.2mm}, clip]{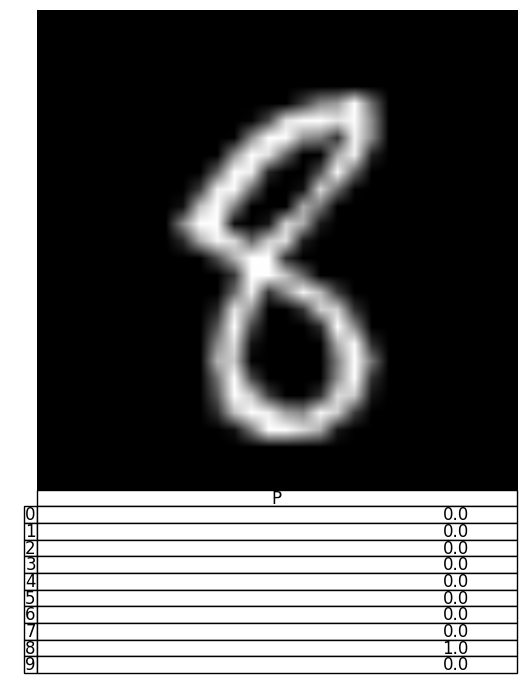}    
			\caption{True image \& class: \texttt{8}.}
		\end{subfigure}%
		\hfill
		\begin{subfigure}[t]{0.67\linewidth}
			\centering
			\begin{minipage}[t]{0.39\linewidth}     
				\centering
				\vspace{0pt}
				\includegraphics[height=0.145\textheight, trim={9.75mm 49mm 3.5mm 3.2mm}, clip]{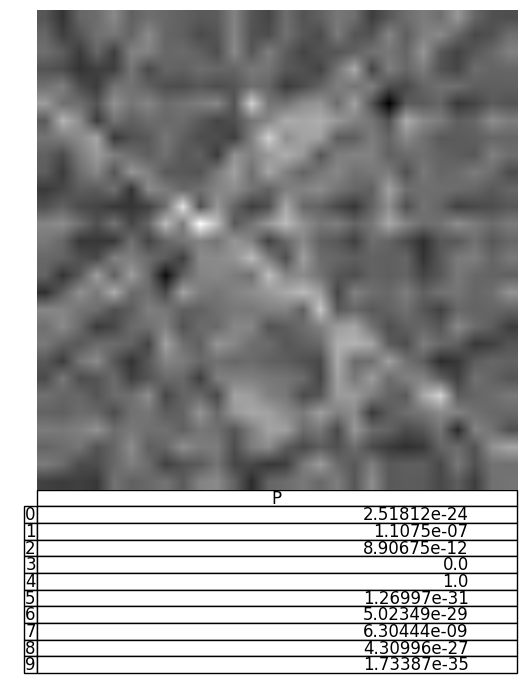}
			\end{minipage}
			\begin{minipage}[t]{0.59\linewidth}
				\centering   
				\vspace{0pt}
				{\small
					\begin{tabular}{c r c r}
						\toprule 
						Class & Prob & Class & Prob \\
						\otoprule  
						\texttt{0} & 0.00\%  & \texttt{5} & 0.00\% \\
						\texttt{1} & 0.00\%  & \texttt{6} & 0.00\% \\
						\texttt{2} & 0.00\%  & \texttt{7} & 0.00\% \\
						\texttt{3} & 0.00\%  & \texttt{8} & 0.01\% \\
						\texttt{4} & 99.99\%  & \texttt{9} & 0.00\% \\
						\bottomrule
					\end{tabular}
				}
			\end{minipage}
			\caption{Sequential approach (FBP): Most likely class is \texttt{4}.}
		\end{subfigure}%
		\\[0.5em]
		\begin{subfigure}[t]{0.3\linewidth}
			\centering
			\vspace{0pt}
			\includegraphics[trim={25mm 60mm 20mm 60mm}, clip, scale=0.19, angle=90]{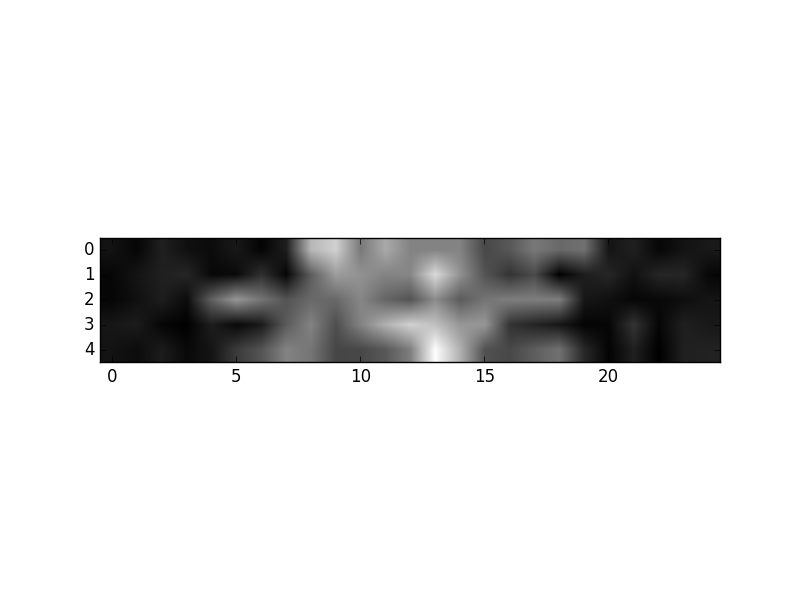}  
			\caption{Data: Sinogram with 5 directions, 25 lines/direction.}  
		\end{subfigure}%
		\hfill
		\begin{subfigure}[t]{0.67\linewidth}
			\centering
			\begin{minipage}[t]{0.39\linewidth}     
				\centering
				\vspace{0pt}
				\includegraphics[height=0.145\textheight, trim={9.75mm 70mm 3.5mm 3.2mm}, clip]{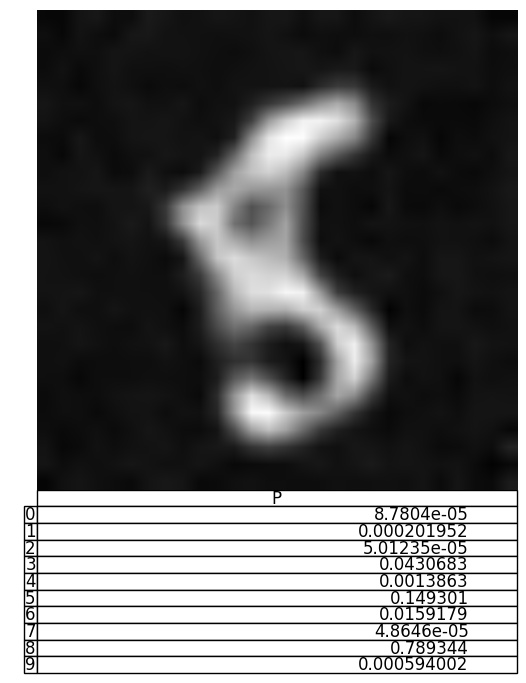}
			\end{minipage}
			\begin{minipage}[t]{0.59\linewidth}
				\centering
				\vspace{0pt}
				{\small
					\begin{tabular}{c r c r}
						\toprule 
						Class & Prob & Class & Prob \\
						\otoprule  
						\texttt{0} & 0.00\%  & \texttt{5} & 14.93\% \\
						\texttt{1} & 0.02\%  & \texttt{6} & 1.59\% \\
						\texttt{2} & 0.00\%  & \texttt{7} & 0.00\% \\
						\texttt{3} & 4.31\%  & \texttt{8} & 78.96\% \\
						\texttt{4} & 0.13\%  & \texttt{9} & 0.06\% \\
						\bottomrule
					\end{tabular}
				}
			\end{minipage}
			\caption{Sequential approach (learned iterative): Most likely class is \texttt{8}.}
		\end{subfigure}%
		\\[1em]
		\begin{subfigure}[t]{0.3\linewidth}
			\centering
			\vspace{0pt}
		\end{subfigure}%
		\hfill
		\begin{subfigure}[t]{0.67\linewidth}
			\centering
			\begin{minipage}[t]{0.39\linewidth}    
				\centering 
				\vspace{0pt}
				\includegraphics[height=0.145\textheight, trim={9.75mm 70mm 3.5mm 3.2mm}, clip]{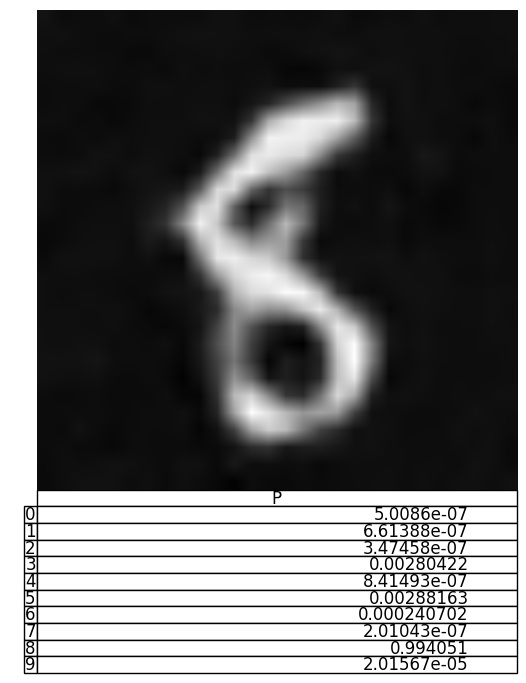}
			\end{minipage}
			\begin{minipage}[t]{0.59\linewidth}
				\centering
				\vspace{0pt}
				{\small
					\begin{tabular}{c r c r}
						\toprule 
						Class & Prob & Class & Prob \\
						\otoprule  
						\texttt{0} & 0.00\%  & \texttt{5} & 0.28\% \\
						\texttt{1} & 0.00\%  & \texttt{6} & 0.03\% \\
						\texttt{2} & 0.00\%  & \texttt{7} & 0.00\% \\
						\texttt{3} & 0.28\%  & \texttt{8} & 99.41\% \\
						\texttt{4} & 0.00\%  & \texttt{9} & 0.00\% \\
						\bottomrule
					\end{tabular}
				}
			\end{minipage}   
			\caption{Joint approach with $C=0.5$. Most likely class is \texttt{8}.}
		\end{subfigure}%
		\vskip-0.5\baselineskip
		\caption{Joint tomographic reconstruction and classification of MNIST images. Training data is to the left and reconstructed image with classification probabilities are on the right. Data on overall performance on all of the MNIST dataset (accuracy) is given in \cref{tab:JointRecoClass}.
		}\label{fig:JointRecoClass}
	\end{figure}

	\FloatBarrier
	\subsubsection{Joint tomographic reconstruction and segmentation}\label{sec:JointRecoSeg}
	\begin{description}
		\item[Task:] Recover the probability map for segmentation of a grey scale image (see \cref{sec:Segmentation} with $k=2$) from noisy parallel beam tomographic data. 
		In this specific example, we consider segmenting the grey matter of \ac{CT} images of the brain, which is relevant in imaging of neurodegerantive diseases like 
		Alzheimers' disease.  
		\item[Data:]
		Elements in $\DataSpace$ are real-valued functions on lines representing parallel beam tomographic data, which are digitized by sampling the the angular variable at 30~uniformly sampled points in $[0,\pi]$ with 183~lines/angle. 
		We furthermore add 0.1\% additive Gaussian noise to data.
		\item[Model parameter space:] 
		Elements in $\RecSpace$ are functions representing images supported on a fixed rectangular region $\signaldomain \subset \Real^2$, so $\RecSpace:=\Lebesgue^2(\signaldomain,\Real)$.
		These are discretized by uniform sampling on a $128 \times 128$ grid.
		The loss $\SignalLoss \colon \RecSpace \times \RecSpace \to \Real$ is the squared $\Lebesgue^2$-distance.
		\item[Decision space:] 
		Elements in $\DecisionSpace$ are point-wise probability distributions over binary images on $\signaldomain$, which can be represented by grey-scale images with values in $[0,1]$ that gives the probability that a point is part of the segmented object. 
		Hence, $\DecisionSpace=\MeasurableMaps\bigl(\signaldomain,[0,1]\bigr)$ with the loss function $\DecisionLoss \colon \DecisionSpace \times \DecisionSpace \to \Real$
		as the cross entropy:  
		\[ 
		\DecisionLoss(\decision,\decisionother)    
		:= \int_{\signaldomain} \Bigl[ - \sum_{i =1}^2 \decision(t)(i) \log\bigl[ \decisionother(t)(i) \bigr] \Bigr]\dint t
		\quad\text{for $\decision, \decisionother \in \MeasurableMaps\bigl(\signaldomain,[0,1]\bigr)$.}
		\]
		\item[Reconstruction and task operators:] 
		Reconstruction $\RecOp{\MLsignalparam} \colon \DataSpace \to \RecSpace$ is given by the Learned Primal-Dual scheme in \cite{Adler:2018aa} and task $\TaskOp{\MLtaskparam} \colon \RecSpace \to \DecisionSpace$ is given by an ``off the shelf'' U-net convolutional neural net for segmentation \cite{Ronneberger:2015aa}.
		\item[Joint training:] Joint supervised data is given as  100 triplets $(\signal_i,\data_i,\decision_i)$ where $\decision_i$ is the segmentation (binary image). We extend joint training data by data argumentation ($\pm 5$ pixel translation and $\pm 10^{\circ}$ rotation).
		There was no pre-training. 
	\end{description}
	Example outcomes are summarized in \cref{fig:JointRecSegLossVal,fig:JointRecoSeg}. 
	Note that $C \to 0$ corresponds to the sequential approach, so the image for $C=0.01$ can be seen as the outcome from a sequential approach. 
	Clearly, a joint approach with $C\approx 0.5$ or $0.9$ outperforms a sequential one.
	
	Next, as $C$ decreases the reconstruction becomes more adapted to the task of segmentation. In the limit $C \to 0$ the task part is viable but the reconstruction image is useless, which is to be expected. 
	In the other direction, as $C$ increases the reconstructed image becomes less adapted to the task and the latter becomes increasingly challenging due to the low contrast between white and grey matter.
	
	Finally, using $C>0$ not only reduces the non-uniqueness as explained in \cref{eq:JointRegEffect}, it further regularizes in the sense that information from the reconstruction guides the segmentation, which otherwise would amount to learning the segmented image directly from the noisy sinogams.
	Intuitively there seems to be an ``information exchange'' between the task of reconstruction and that of segmentation, which when properly balanced by choosing $C$ acts as a regularizer for the segmentation, e.g., the white/grey matter contrast in the reconstruction is overemphasized for small $C$.
	This improves the interpretability since it shows how the reconstructed image ``helps'' in interpreting why a certain segmentation is taken.
	\begin{figure}[t]
		\centering
		\begin{tikzpicture}
		\begin{axis}[
		height=0.17\textheight,
		width=0.9\linewidth,
		x label style={at={(axis description cs:1,0.35)},anchor=north,xshift=0.2cm},
		xlabel=$C$,
		ylabel style={yshift=0.5cm,rotate=-90},  
		ylabel=$\SignalLoss$,  
		ymode=log,
		log basis y={10},
		xtick={0, 1, 2, 3, 4, 5},
		xticklabels={0.01, 0.1, 0.5, 0.9, 0.99, 0.999}
		]  
		\addplot table [y=SLoss, x=Loc, col sep=comma]{rec_seg_data.txt};
		\end{axis}
		\end{tikzpicture} 
		\\[0.75em]
		\begin{tikzpicture}
		\begin{axis}[
		height=0.17\textheight,
		width=0.9\linewidth,
		x label style={at={(axis description cs:1,0.35)},anchor=north,xshift=0.2cm},
		xlabel=$C$,
		ylabel style={yshift=0.5cm,rotate=-90},  
		ylabel=$\DecisionLoss$,
		ymode=log,
		log basis y={10},
		xtick={0, 1, 2, 3, 4, 5},
		xticklabels={0.01, 0.1, 0.5, 0.9, 0.99, 0.999}
		]  
		\addplot table [y=DLoss, x=Loc, col sep=comma]{rec_seg_data.txt};
		\end{axis}
		\end{tikzpicture}
		\caption{Log-log plot of loss functions for joint reconstruction and segmentation after joint training for different $C$ in \cref{eq:JointLoss}.
			Clearly, there is no joint minimizer but $0.5 \lesssim C \lesssim 0.9$ is a good compromise.}\label{fig:JointRecSegLossVal}
	\end{figure}
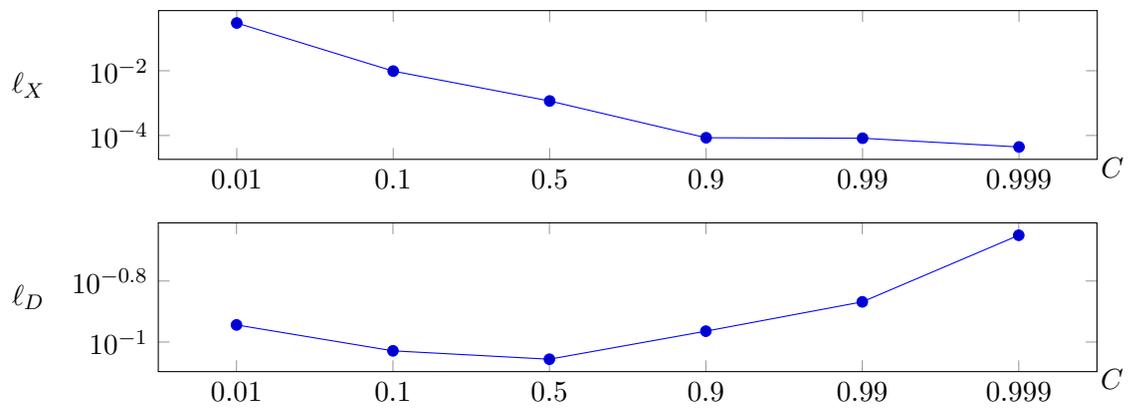
	
	\begin{figure}
		\centering
		\begin{minipage}{0.49\linewidth}
			\centering
			\includegraphics[width=0.49\linewidth, trim={21mm 19mm 32mm 6mm}, clip]{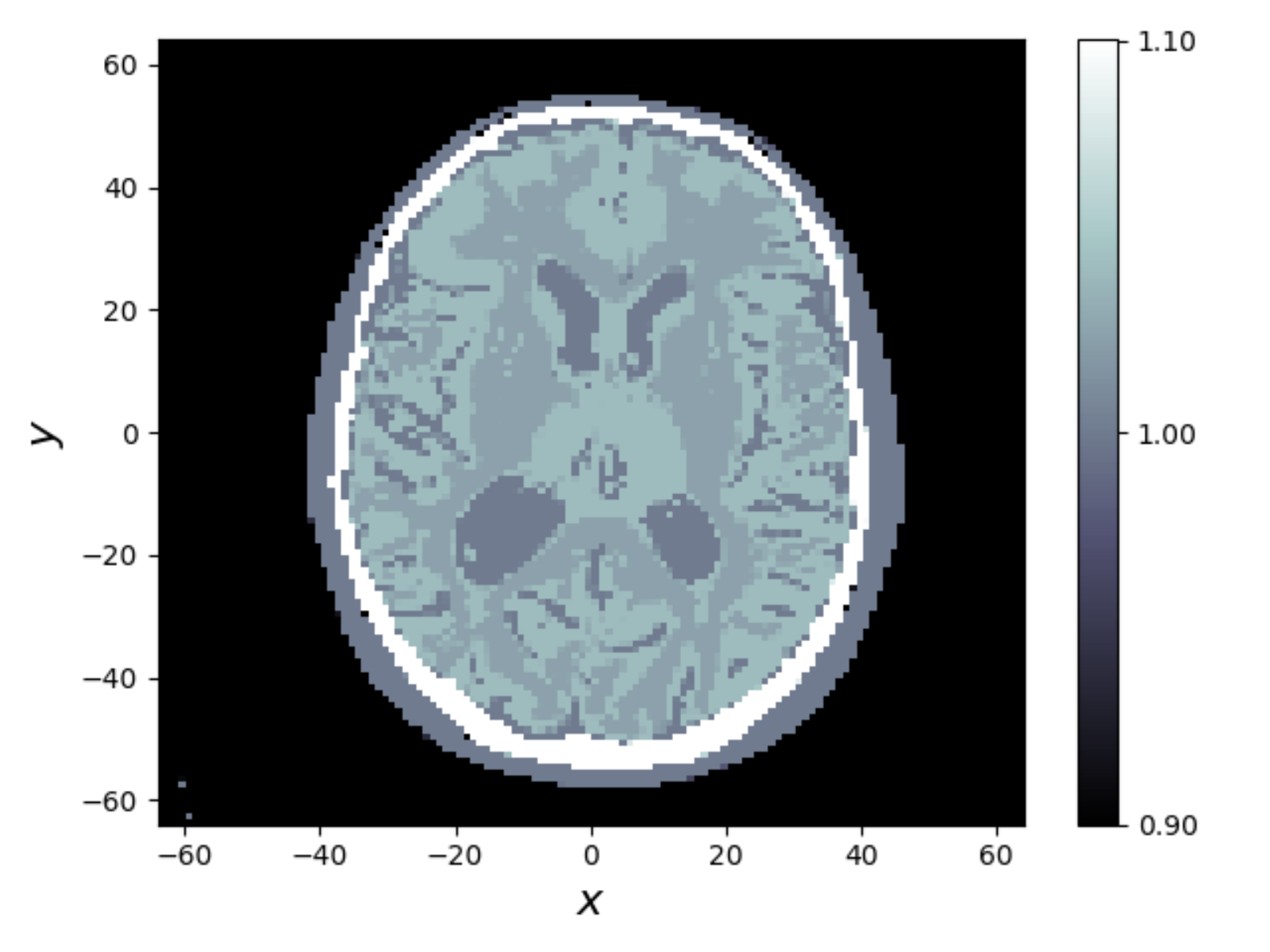}
			\includegraphics[width=0.49\linewidth, trim={21mm 19mm 32mm 6mm}, clip]{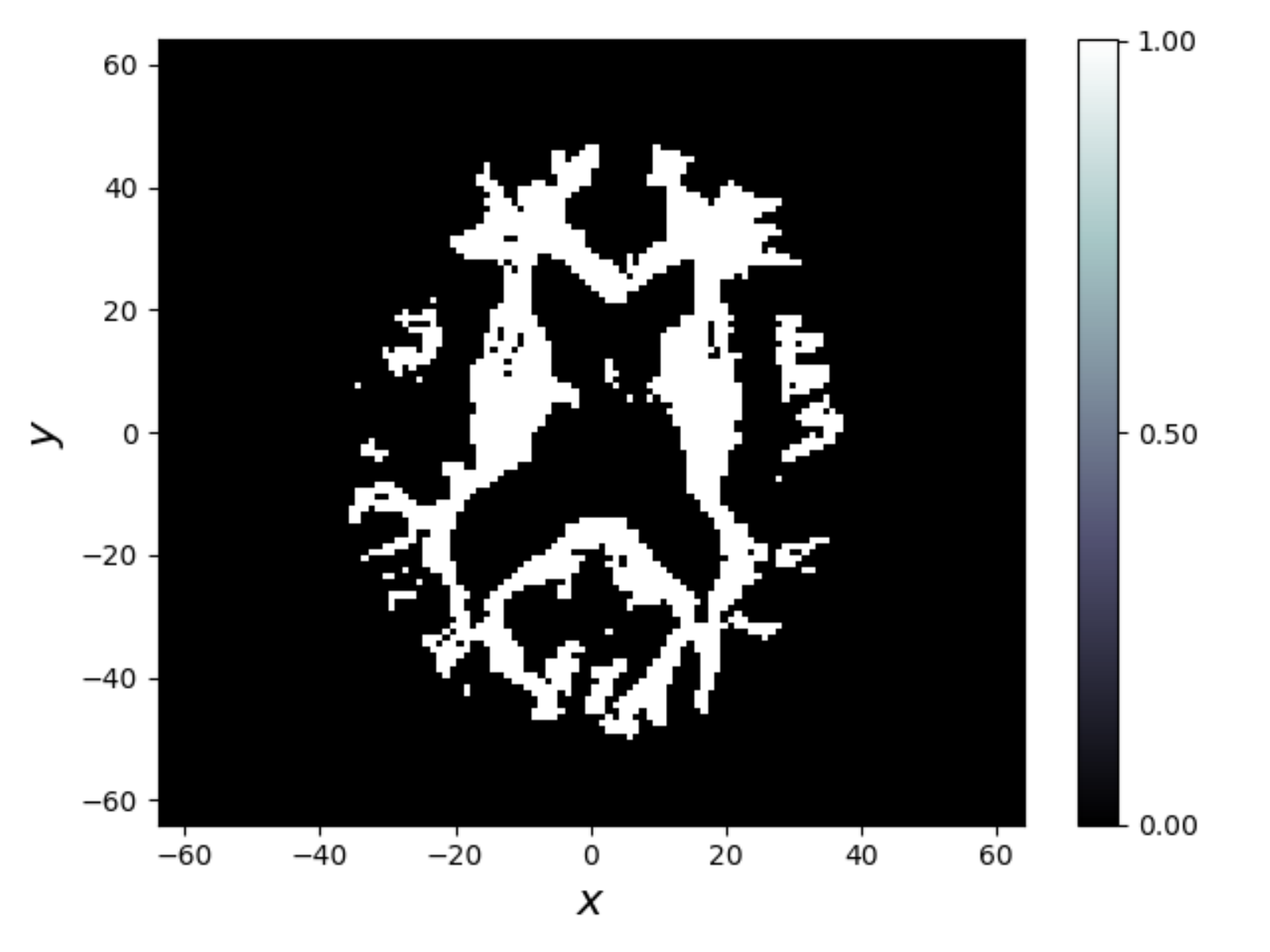}
			True image \& segmentation.    
		\end{minipage}
		\hfill
		\begin{minipage}{0.49\linewidth}
			\centering
			\includegraphics[width=0.49\linewidth, trim={21mm 19mm 32mm 6mm}, clip]{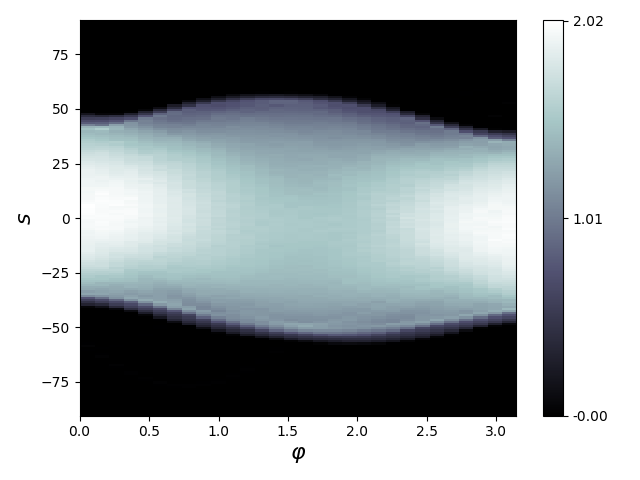}
			\\
			Data: Sinogram, 30 angles \& 177~lines/angle.
		\end{minipage}
		\\[1em]
		\begin{minipage}{0.49\linewidth}
			\centering  
			\includegraphics[width=0.49\linewidth, trim={21mm 19mm 32mm 6mm}, clip]{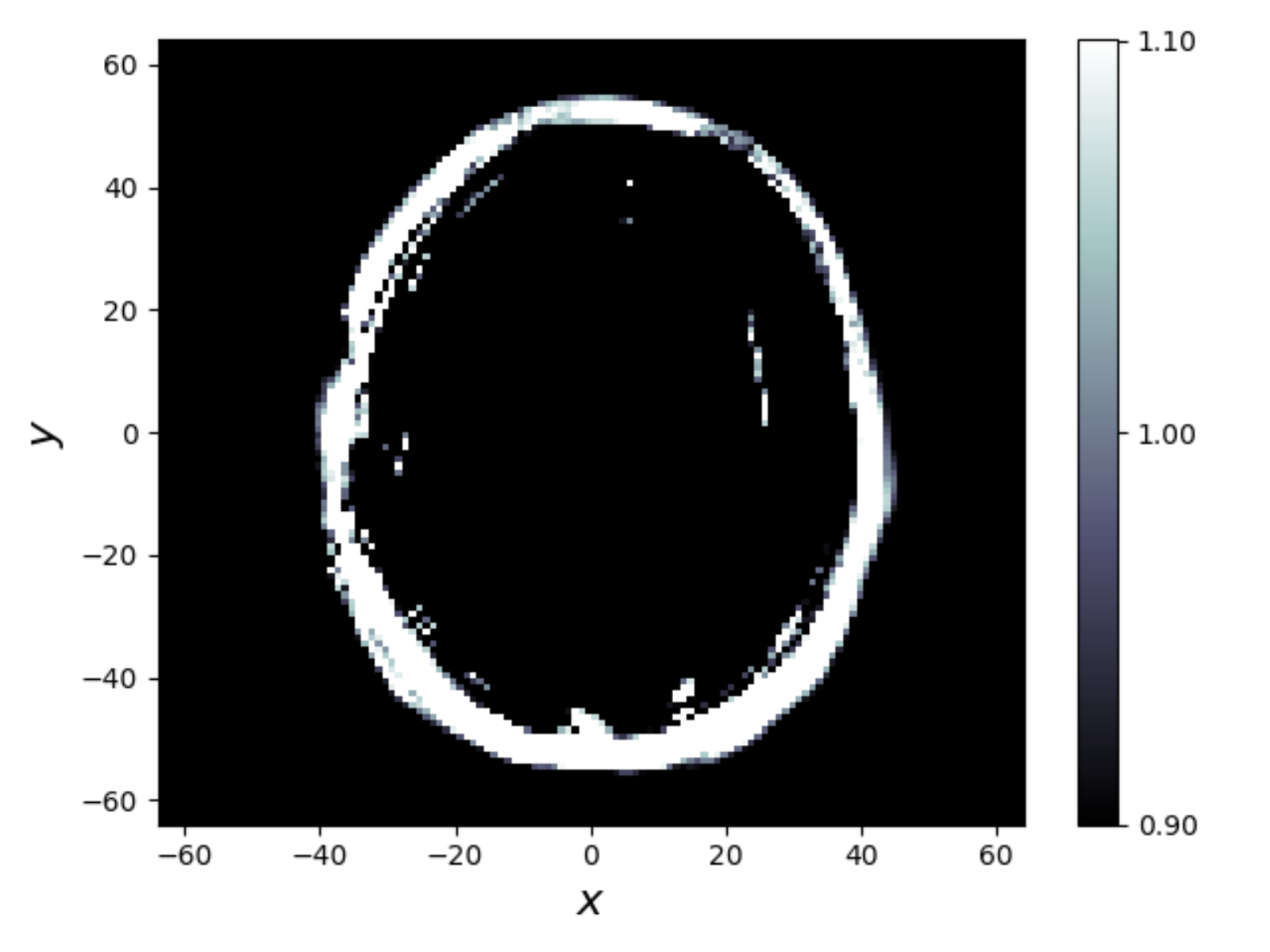}
			\includegraphics[width=0.49\linewidth, trim={21mm 19mm 32mm 6mm}, clip]{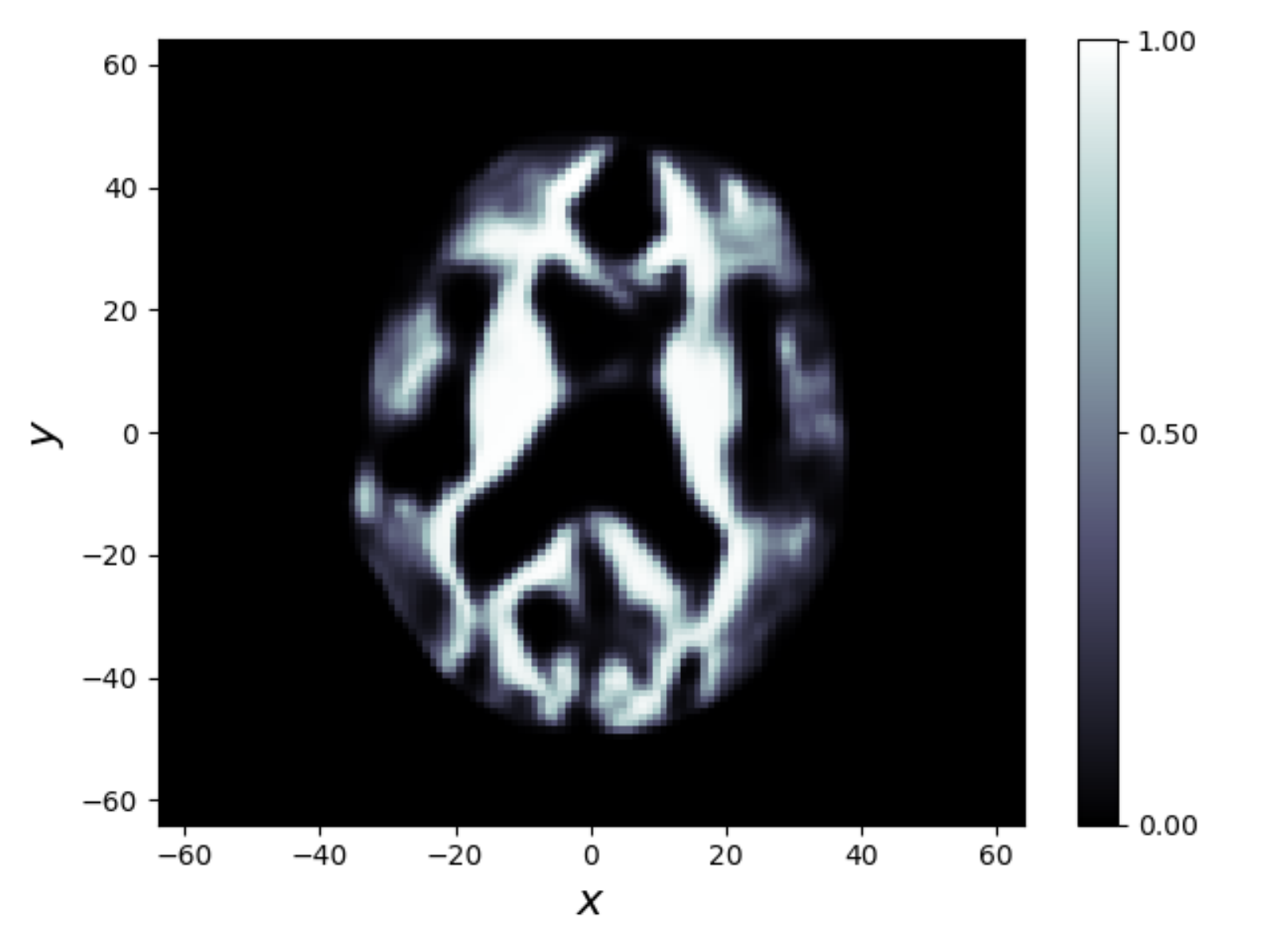}
			Joint reco. \& segmentation: $C=0.01$.
		\end{minipage}
		\hfill
		\begin{minipage}{0.49\linewidth}
			\centering  
			\includegraphics[width=0.49\linewidth, trim={21mm 19mm 32mm 6mm}, clip]{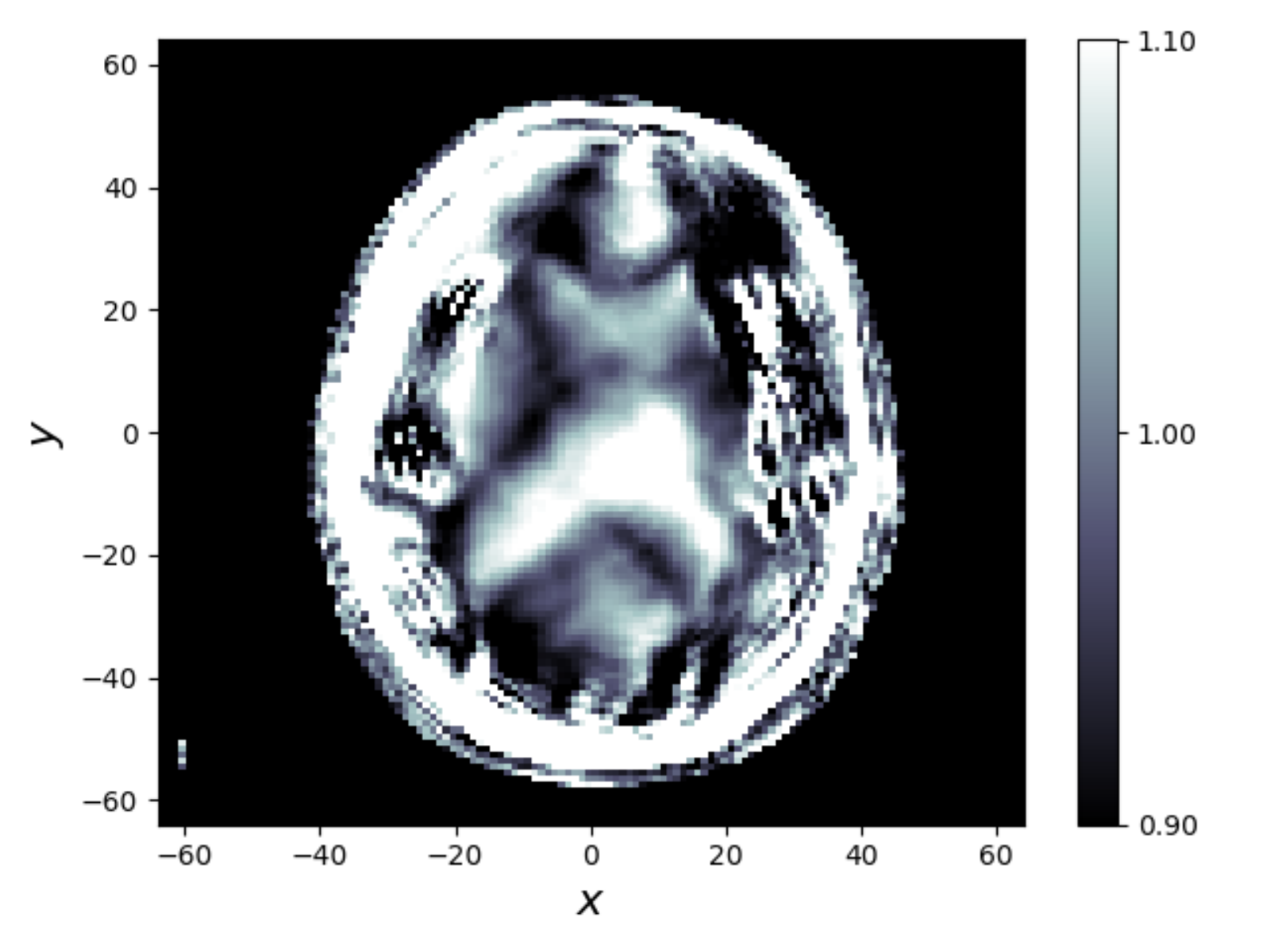}
			\includegraphics[width=0.49\linewidth, trim={21mm 19mm 32mm 6mm}, clip]{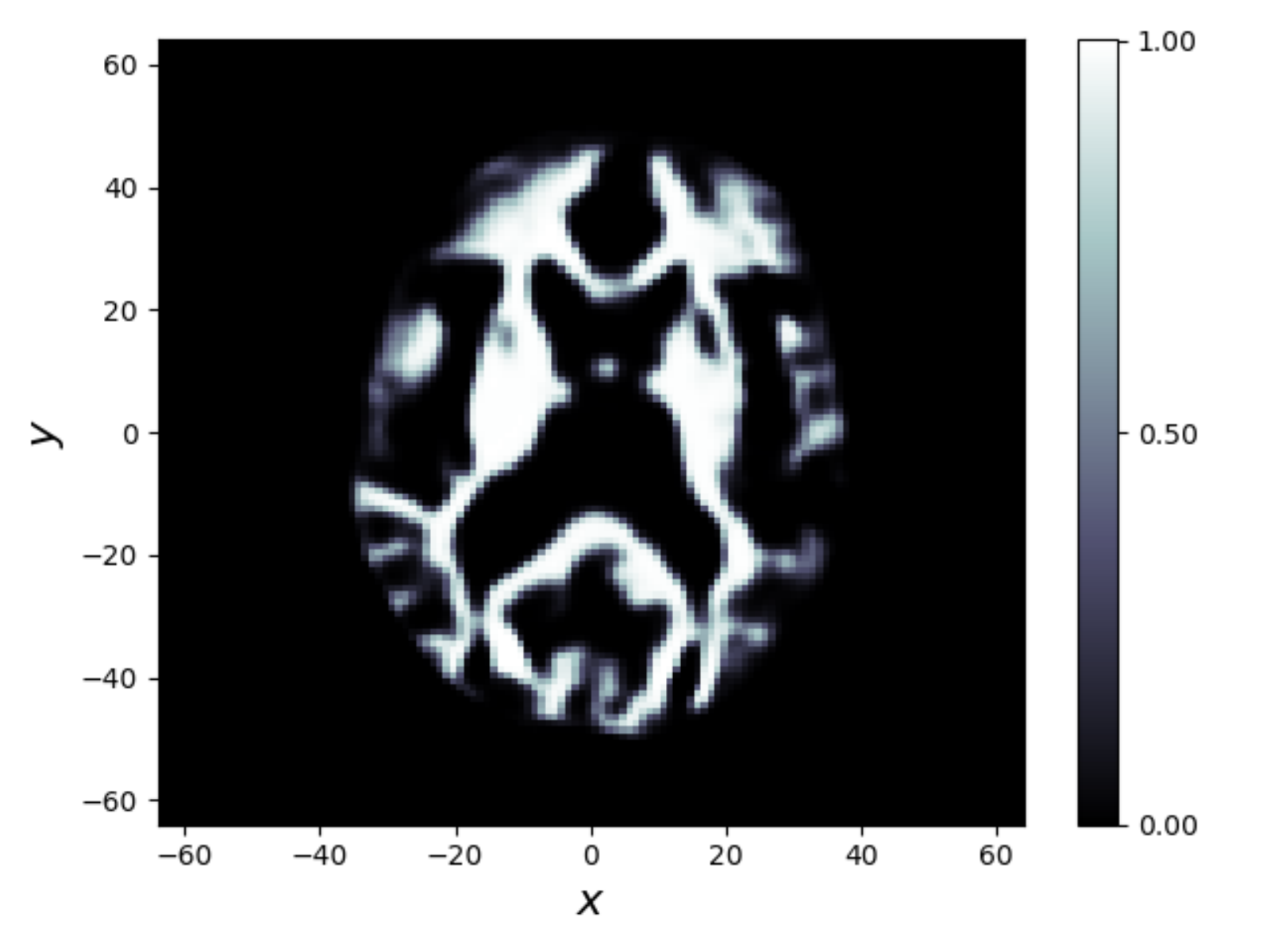}
			Joint reco. \& segmentation: $C=0.1$.
		\end{minipage}  
		\\[1em]
		\begin{minipage}{0.49\linewidth}
			\centering  
			\includegraphics[width=0.49\linewidth, trim={21mm 19mm 32mm 6mm}, clip]{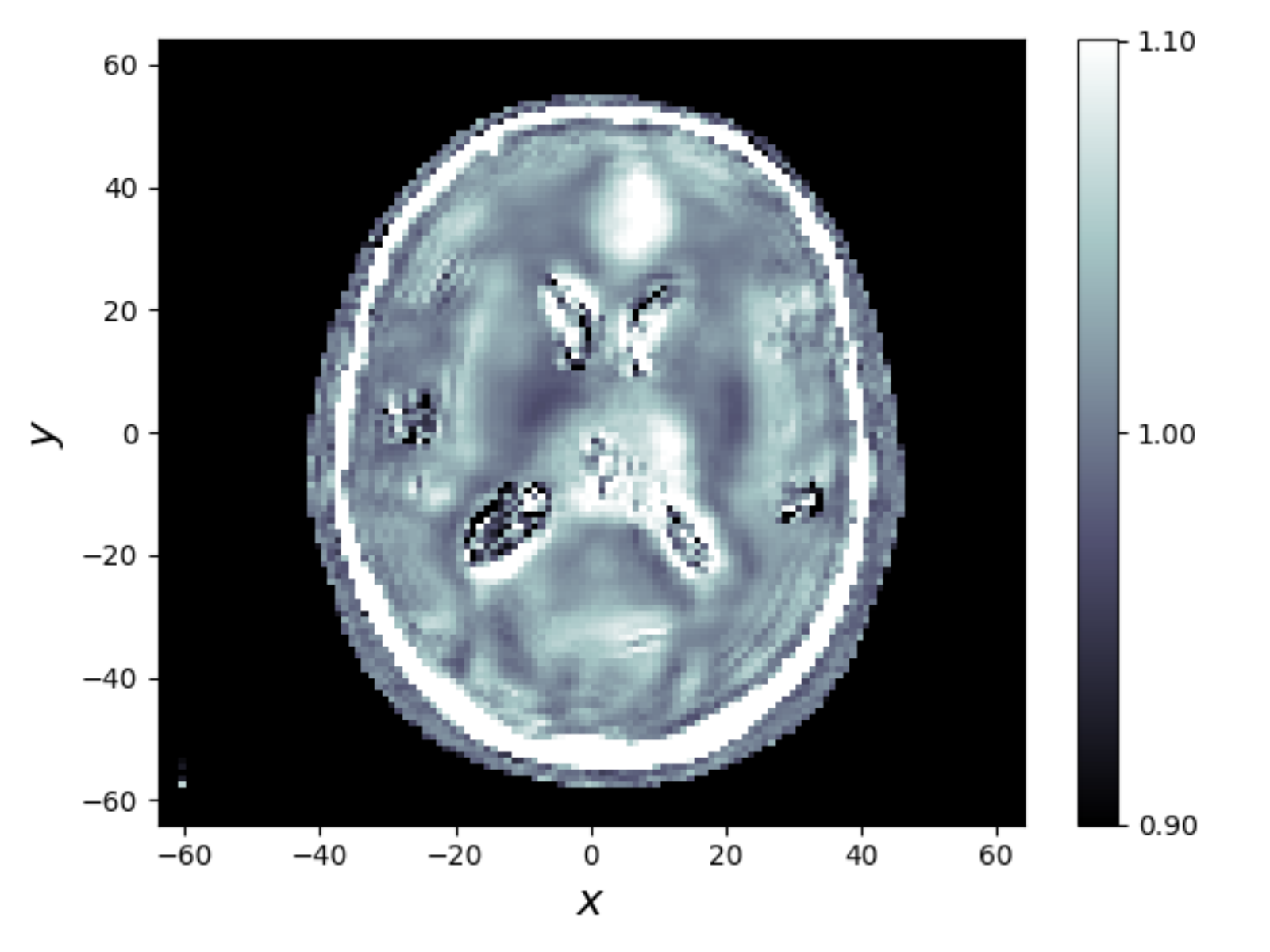}
			\includegraphics[width=0.49\linewidth, trim={21mm 19mm 32mm 6mm}, clip]{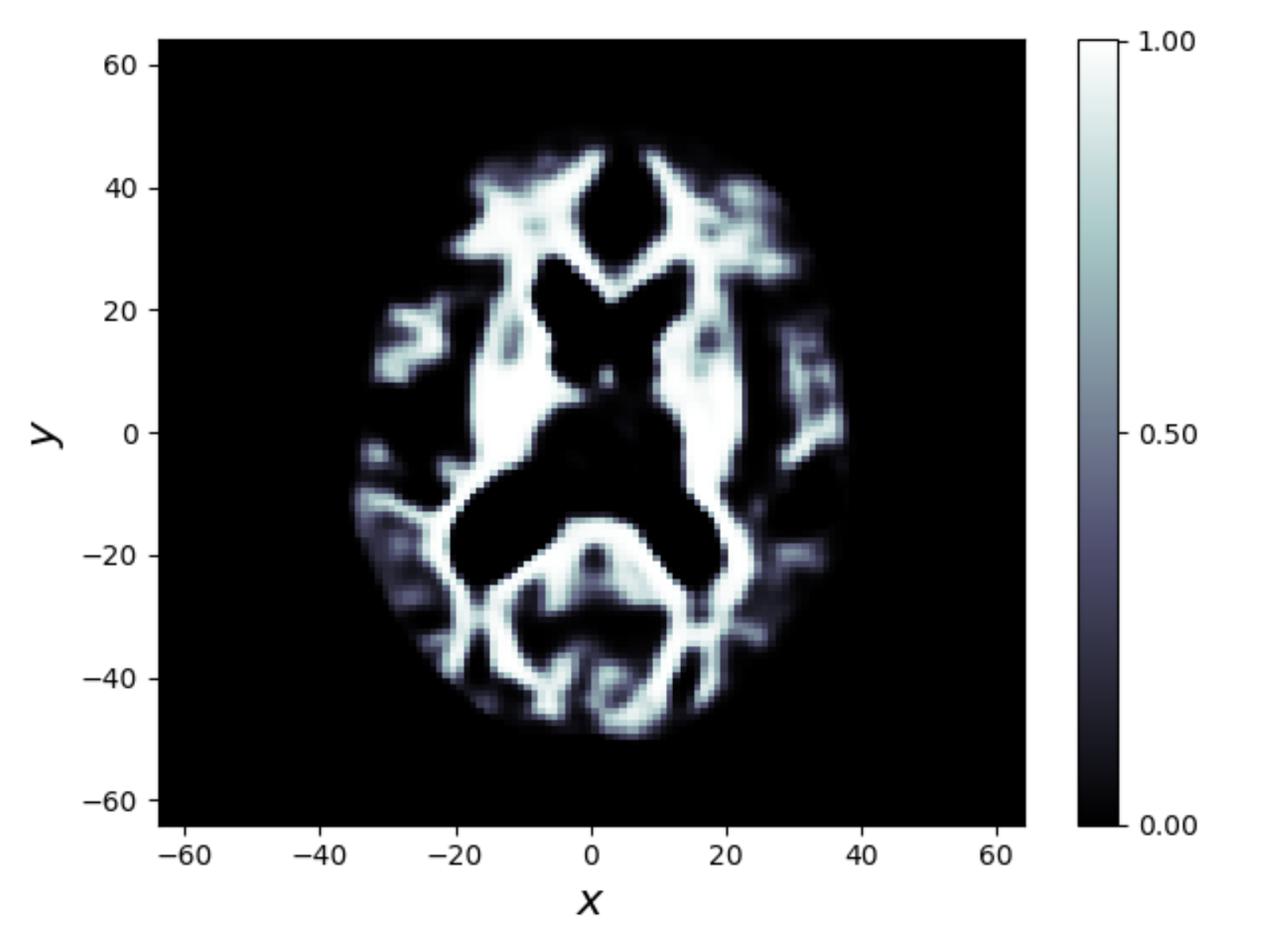}
			Joint reco. \& segmentation: $C=0.5$.
		\end{minipage}
		\hfill
		\begin{minipage}{0.49\linewidth}
			\centering  
			\includegraphics[width=0.49\linewidth, trim={21mm 19mm 32mm 6mm}, clip]{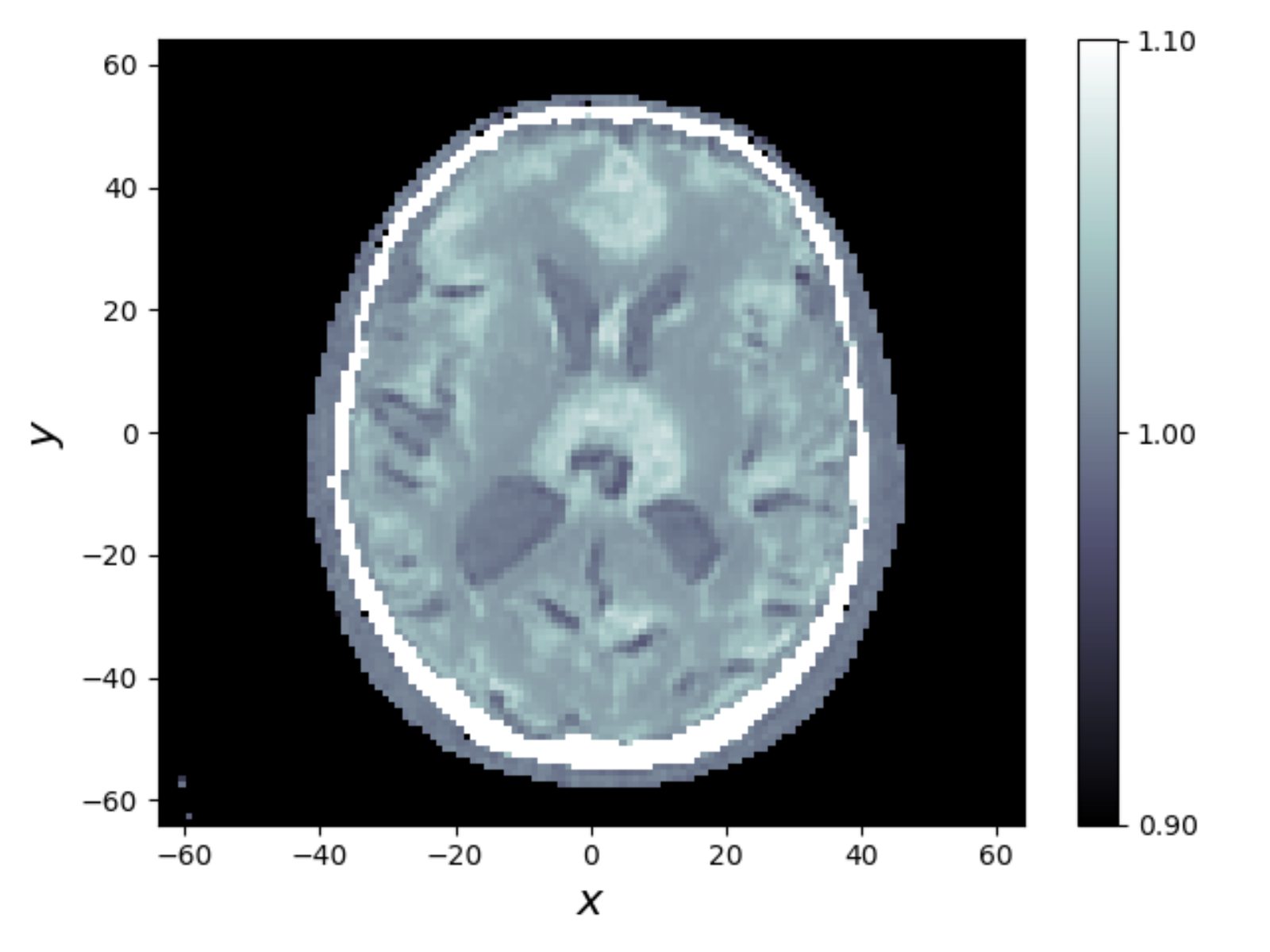}
			\includegraphics[width=0.49\linewidth, trim={21mm 19mm 32mm 6mm}, clip]{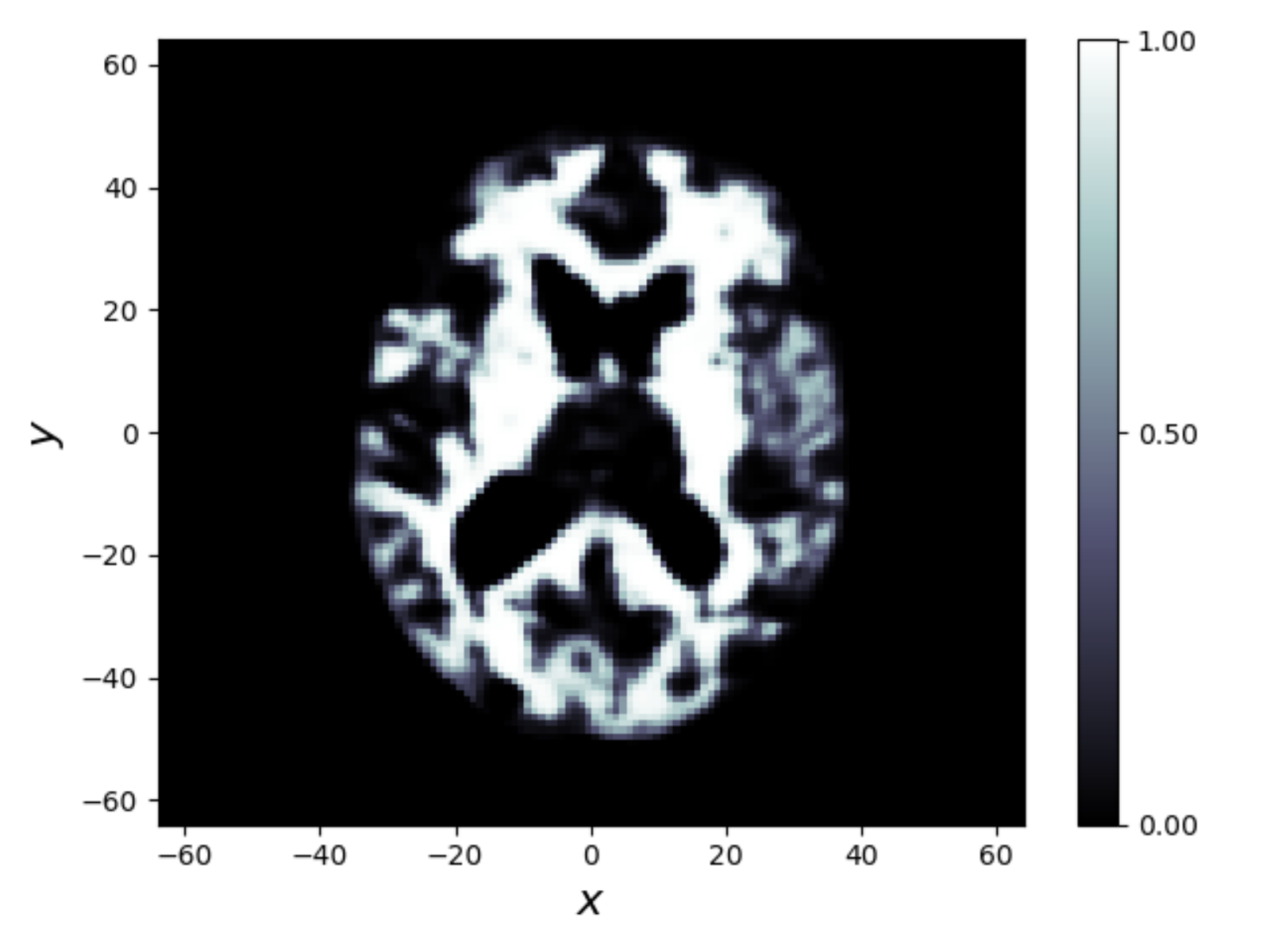}
			Joint reco. \& segmentation: $C=0.9$.
		\end{minipage}  
		\\[1em]  
		\begin{minipage}{0.49\linewidth}
			\centering  
			\includegraphics[width=0.49\linewidth, trim={21mm 19mm 32mm 6mm}, clip]{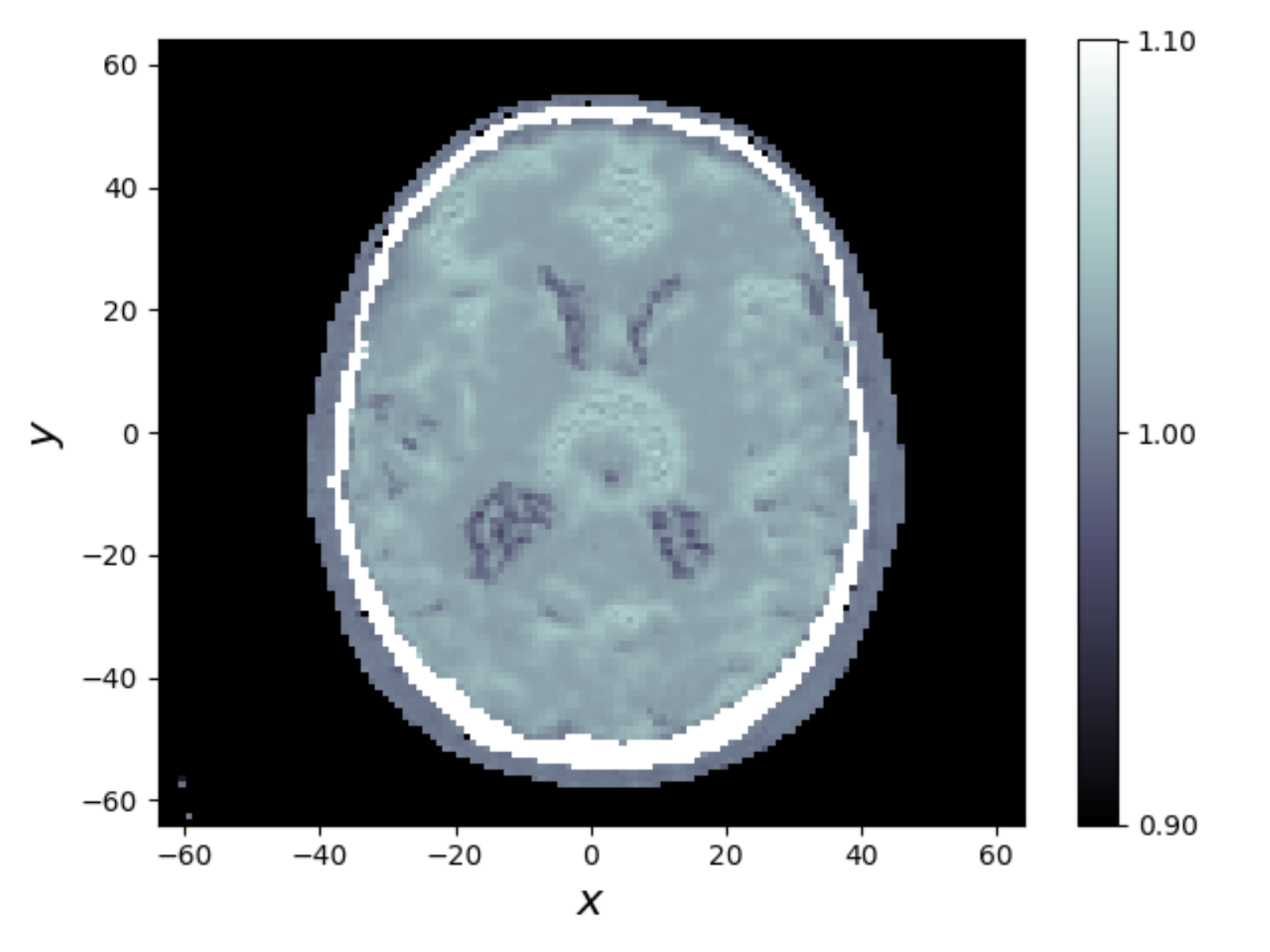}
			\includegraphics[width=0.49\linewidth, trim={21mm 19mm 32mm 6mm}, clip]{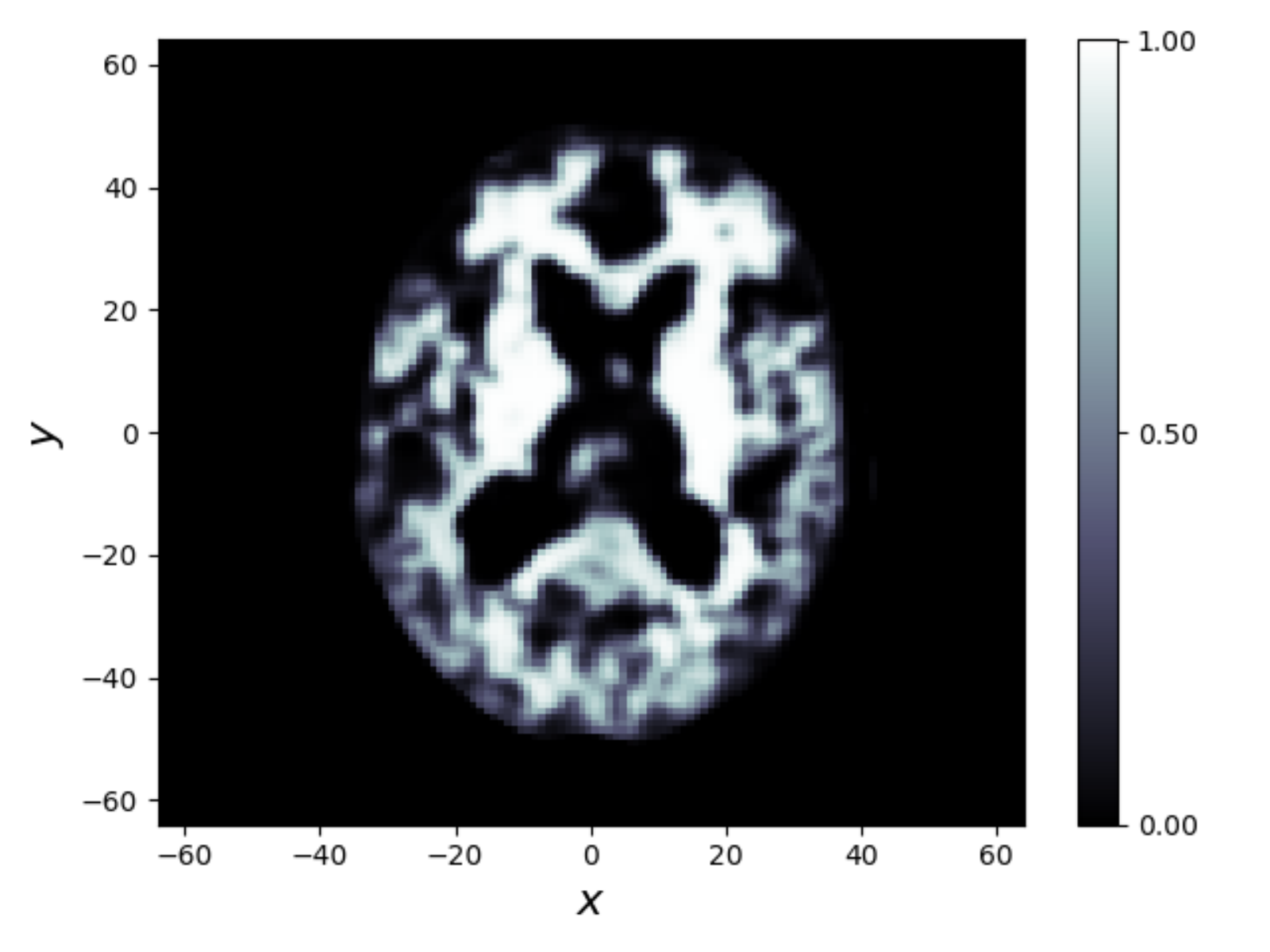}
			Joint reco. \& segmentation: $C=0.99$.
		\end{minipage}
		\hfill
		\begin{minipage}{0.49\linewidth}
			\centering  
			\includegraphics[width=0.49\linewidth, trim={21mm 19mm 32mm 6mm}, clip]{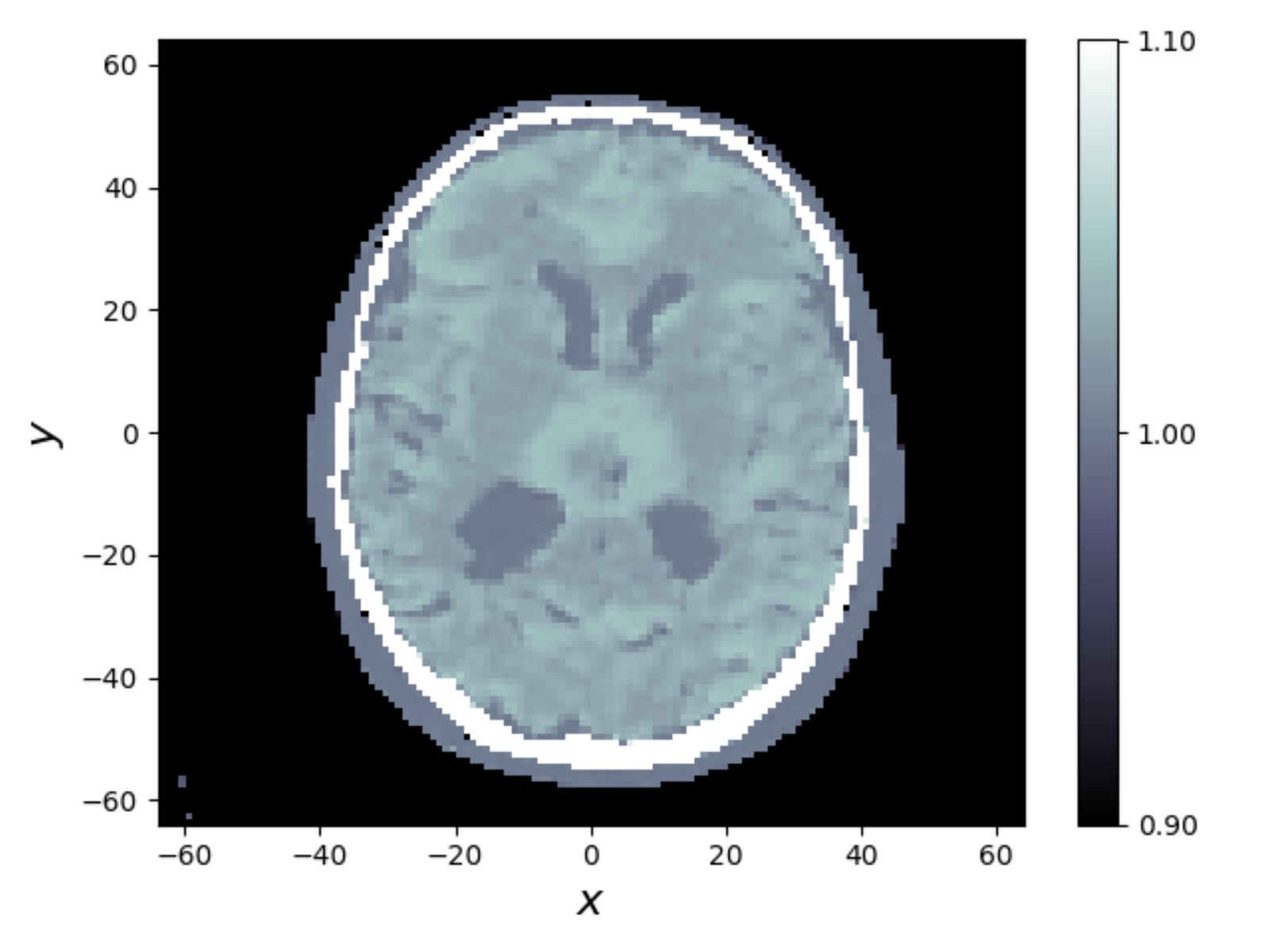}
			\includegraphics[width=0.49\linewidth, trim={21mm 19mm 32mm 6mm}, clip]{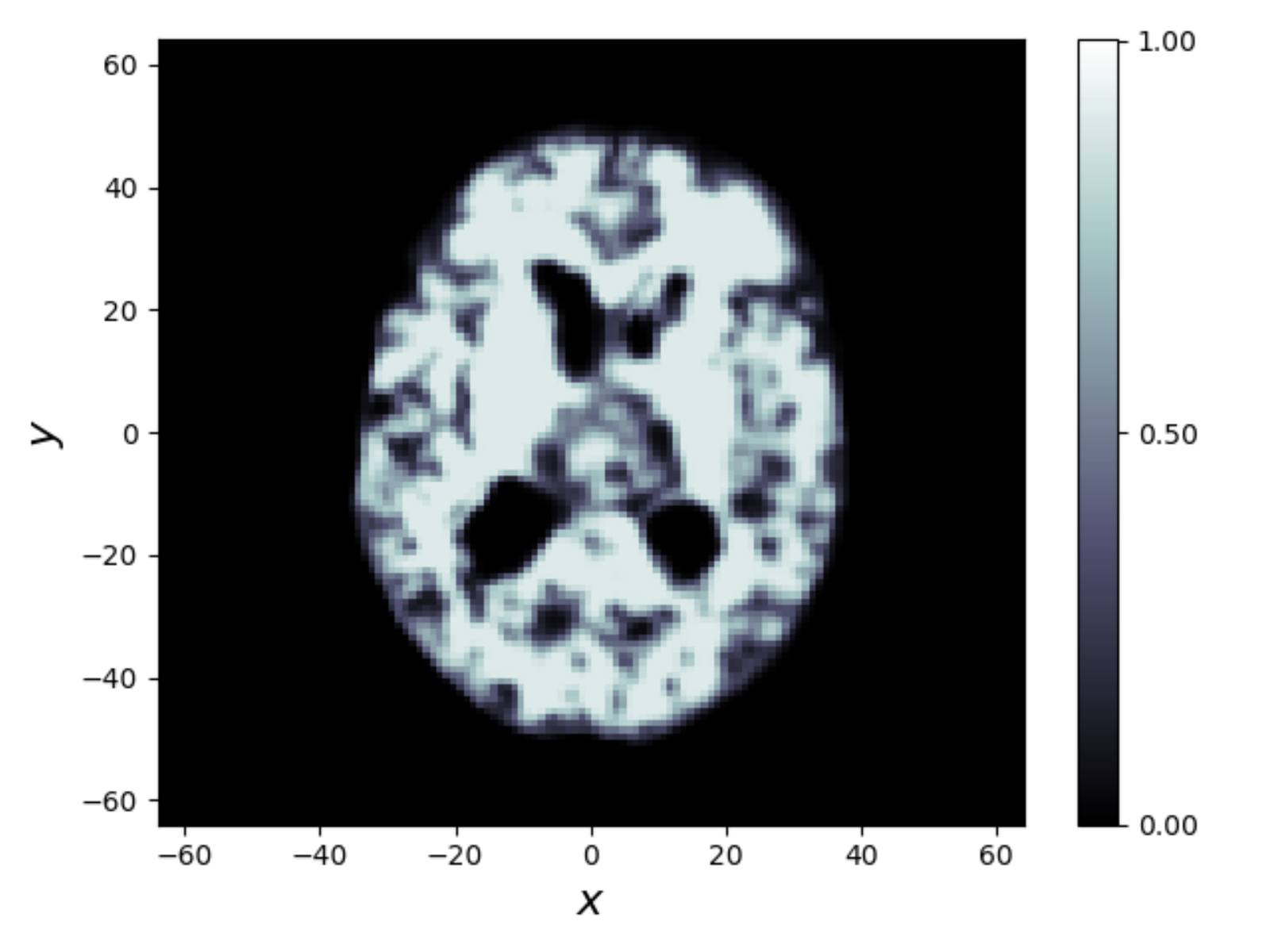}
			Joint reco. \& segmentation: $C=0.999$.
		\end{minipage}
		\caption{Joint tomographic reconstruction and segmentation for different values of $C$ in \cref{eq:JointLoss}. 
			The segmentation is a normalized grey-scale image denoting the probability that a point belongs to the segmented structure.
			The choice $C=0.9$ seems to be a good compromise for a good reconstruction \emph{and} segmentation (see \cref{fig:JointRecSegLossVal}).
			Note also that $C \to 1$ gives the sequential approach, so $C=0.999$ may serve as a proxy for it.
			Reconstructions take a few milliseconds to perform on a desktop gaming PC.}
		\label{fig:JointRecoSeg}  
	\end{figure}

	\FloatBarrier
	\section{Theoretical considerations}\label{sec:TheoryConsider}
	The joint task adapted reconstruction defined in \cref{eq:JointApproach} is given by combining two optimal decision rules into a single decision rule, one for reconstruction that acts on data and the other encoding a task that acts on model parameters. 
	It is therefore natural to investigate whether the theoretical machinery developed for Bayesian inversion can be used to analyze regularizing properties of this joint approach.
	An example would be to investigate the conditions under which the joint approach is a regularization in the formal sense, which means proving existence, stability, and posterior consistency that is preferably complemented by providing contraction rates, see \cite[chapters~6-9]{Ghosal:2017aa} for the precise definitions.
	
	Much of the theory on Bayesian inversion that deals with such matters is well understood for linear problems in the finite dimensional setting \cite{Kaipio:2005aa}, but things quickly become complicated for infinite dimensional non-parametric problems.
	There has been nice progress recently on consistency, posterior contraction rates, and characterization of the microscopic fluctuations of the posterior that is  relevant for Bayesian inversion, see \cite{Stuart:2010aa,Dashti:2016aa,Nickl:2017aa} for nice surveys and \cite{Monard:2017aa} for a in-depth treatment of reconstruction relevant for tomographic imaging.
	On the other hand, the theory and associated results require too many restrictive assumptions that renders them inapplicable for analyzing the task adapted approach in \cref{eq:JointApproach}.
	To conclude, theory of Bayesian inversion is in its current state not useful for characterizing conditions for when the joint task adapted reconstruction in \cref{eq:JointApproach} is a regularization.
	
	Another line of investigation considers the potential advantage that comes with using a joint approach over a sequential one.
	Since the reconstruction and task operators are trained separately in a sequential approach, some information is inevitably lost when applying a regularized reconstruction operator.
	In contrast, both reconstruction and task operators are trained simultaneously in a joint approach so there is a better chance of preserving the information.
	Hence, we expect a joint approach to perform better, which is also supported by the observation in \cref{eq:JointRegEffect} and the empirical evidence in \cref{sec:Applications}.
	
	Now, albeit convincing, the above heuristic argument is flawed! 
	In fact, as stated by \cref{prop:probpinv}, it is surprisingly difficult to theoretically prove that a joint approach outperforms a sequential one in a non-parametric setting where one has access to all of data. 
	The reason is that many standard operators that map data space to model parameter space are formally information conserving in such a setting.
	The adjoint of the forward operator, its Moore-Penrose pseudo-inverse, and even some regularized reconstruction operators such as the usual Tikhonov regularization are information conserving under standard Gaussian noise. 
	For 2D parallel beam tomography, yet another example is the \ac{FBP} reconstruction operator (with a filter that is strictly non-zero in frequency space).
  \begin{proposition}\label{prop:probpinv}
		Let $\stsignal$ be a $\RecSpace$-valued random variable, and $\stdata$ be a $\DataSpace$-valued random variable, both defined on the same probability space.
    Let $\Pi \colon \DataSpace \to \DataSpace$ be a measurable operator with closed range.
		Let $\OpB$ be an arbitrary measurable map defined on $\DataSpace$ that is injective when restricted to $\range(\Pi)$.
    Then, the following holds:
    \begin{equation}
      \Expect\bigl[ f(\stsignal)|\stdata \bigr] = \Expect\bigl[ f(\stsignal) \mid \OpB(\Pi \stdata) \bigr]
      \text{ for all $f$}
      \quad\iff\quad
      \stsignal  \perp\!\!\!\perp (\stdata-\Pi \stdata) \mid \Pi \stdata
    \end{equation}
    where $f$ spans all random variables over $\RecSpace$.
  \end{proposition}
  Before getting to the proof, let us comment on the implication of the statement above.
  The operator $\Pi$ typically represents an orthogonal projection onto the closure of the range of $\ForwardOp$.
  The result above states that the probability of $\stsignal$ conditioned on $\stdata$ is the same as the one conditioned on $\OpB(\Pi \stdata)$ if and only if, given the knowledge of $\Pi \stdata$, the ``noise'' in the null space of $\Pi$, namely $\stdata - \Pi\stdata$, is independent of $\stsignal$.
  \begin{proof}
  The proof is essentially a rewriting of the definitions.
Introduce the notations $\stdata_1 := \Pi \stdata$ and $\stdata_2 := \stdata - \Pi \stdata$, so that $\stdata = \stdata_1 + \stdata_2$.
Then $\Expect[f(\stsignal)|\stdata] = \Expect[f(\stsignal)|\stdata_1,\stdata_2]$, as $\stdata$ and $(\stdata_1, \stdata_2)$ generate the same $\sigma$-algebras.
  The injectivity of $\OpB$ on $\range(\Pi)$ implies that the $\sigma$-algebra generated by $\OpB\circ \Pi \circ \stdata$ and $\Pi \circ \stdata$ are the same, so $\Expect[f(\stsignal)|\stdata_1] = \Expect[f(\stsignal)|\OpB(\stdata_1)]$.
Now, requiring that $\Expect[f(\stsignal)|\stdata_1,\stdata_2] = \Expect[f(\stsignal)|\stdata_1]$ holds for all $f$ is exactly the statement of conditional independence in the claim.
  \end{proof}

	\begin{corollary}\label{corr:TaskVsJoint}
		Consider the setting in \cref{sec:TaskAdaptedAbstract} for task adapted reconstruction and assume in particular that
    $\data$ and $\task$ are conditionally independent given $\signal$.
		Finally, let $\OpB$ satisfy the assumptions in \cref{prop:probpinv};
 we also assume that the equality in \cref{prop:probpinv} holds, that is $\pi(\signal \mid \data) = \pi(\signal \mid \OpB(\data))$.
    Then, 
		\begin{equation}\label{eq:EquivSeqVsJOint}
			\CondLaw(\task \mid \data)
			=
			\CondLaw\bigl(\task \mid \OpB(\data)\bigr).
		\end{equation}
	\end{corollary}
	\begin{proof}
    The conditional independence assumption can be written as $\pi(\task \mid \signal,\data) = \pi(\task \mid \signal)$.
    Using this, we compute
    \(
      \CondLaw(\signal,\data,\task) = \CondLaw(\task \mid \signal,\data) \CondLaw(\signal,\data) = \CondLaw(\task \mid \signal) \CondLaw(\signal,\data)
\),
    which yields
    \begin{equation}
      \label{eq:Eq1}
      \CondLaw(\signal,\task \mid \data) = \CondLaw(\task \mid \signal) \CondLaw(\signal \mid \data)
      .
    \end{equation}

    Notice now that $\OpB(\data)$ and $\task$ are also conditionally independent given $\signal$, so we similarly obtain
    \begin{equation}
      \label{eq:Eq2}
      \CondLaw(\signal, \task \mid \OpB(\data)) = \CondLaw(\task \mid \signal) \CondLaw(\signal \mid \OpB(\data))
      .
    \end{equation}

		
		Now, \cref{eq:Eq1,eq:Eq2} imply in particular that 
		\begin{equation}\label{eq:EqNext}
			\begin{split}
				\CondLaw(\task \mid \data) 
				&= \int  \CondLaw(\task, \signal \mid \data)  \dint \signal
				= \int    \CondLaw(\task \mid \signal) \CondLaw(\signal \mid \data) \dint \signal
				\\
				\CondLaw\bigl(\task \mid \OpB(\data) \bigr)
				&= \int  \CondLaw\bigl(\task, \signal \mid \OpB(\data) \bigr)  \dint \signal
				= \int    \CondLaw(\task \mid \signal) \CondLaw\bigl(\signal \mid \OpB(\data) \bigr) \dint \signal.
			\end{split}     
		\end{equation}
		Our assumption is that $\CondLaw(\signal \mid \data) = \CondLaw\bigl(\signal \mid \OpB(\data)\bigr)$, which combined with \cref{eq:EqNext} yields \cref{eq:EquivSeqVsJOint}.
		This concludes the proof.
	\end{proof}
	By \cref{corr:TaskVsJoint} we see directly that the conditional distribution of $\sttask$ given data $\data \in \DataSpace$ is, as $\TaskSpace$-valued random variables, equal to the conditional distribution of $\sttask$ given an initial reconstruction $\OpB(\data) \in \RecSpace$.
	In particular, a task adapted reconstruction method (either sequential or joint) is equivalent to first performing reconstruction by applying the fixed operator $\OpB \colon \DataSpace \to \RecSpace$, which is not trained, followed by $\OpC \colon \RecSpace \to \DecisionSpace$ that is given as 
	\[ \OpC:= \TaskOp{} \circ \RecOp{} \circ \OpB^{-1}. \]
	Note here that $\OpC$, which is trained, is a measurable map defining a non-randomized decision rule that in principle serves as a ``task'' operator.
	
	To summarize, we cannot resort to ``information bottleneck'' type of arguments as an explanation for why the joint approach should outperform only training a task operator in this general setting. On the other hand, the above argument hints that an explanation must involve either the choice of architecture or the training protocol. Both of these are examples of classical and widely studied problems in deep learning concerning why deep learning ``works'' and these remain largely unsolved.
	Another argument in favor of a joint approach is that it is highly non-trivial to select an appropriate architecture for parametrizing $\OpC$, whereas $\TaskOp{}$ and $\RecOp{}$ are easier to parametrize by means of neural networks. 
	Another possible reason is that the operations, like evaluating $\OpB$ or its inverse $\OpB^{-1}$, may not be stable. 
	Finally, as we have seen from the examples, using knowledge about the reconstruction may in fact act as a regularizer, either by improving the trainability or the generalization properties.
	
	\section{Discussion and outlook}\label{sec:Discussion}
	A key aspect for the implementation of the joint task adapted reconstruction method in \cref{eq:JointApproach} is that both decision rules are given by trainable neural networks, which after joint training  forms a single intertwined neural network. 
	In such case, the problem reduces to solving \cref{eq:TrainJointApproach}.
	
	The neural network for the reconstruction should here preferably incorporate knowledge about how data is generated.
	Learned iterative methods, like the Learned Primal-Dual method, are therefore well suited for this task since they are given by a (deep) neural network that embeds the forward operator and a statistical model for the nose in measured data \cite{Adler:2017aa,Adler:2018aa}. 
	
	Next, as shown in \cref{sec:ExampleTasks,sec:OtherTasks}, a wide range of tasks can be interpreted as applying an optimal decision rule on the model parameters. 
	The abstract framework for task adapted reconstruction (\cref{sec:TaskAdaptedAbstract}) works with \emph{any} task that can be represented by a neural network as long as the parametrization and the loss functions are differentiable, like those listed in \cref{sec:ExampleTasks,sec:OtherTasks}.
	Hence, our approach opens up for \emph{truly task adapted reconstruction that goes well beyond performing reconstruction jointly with simple feature extraction}.
	In particular, more advanced tasks, such as image caption generation or image-processing steps in radiomics \cite{Gillies:2015aa,Lambin:2017aa}, can be performed jointly with reconstruction. 
	This potential is also mentioned in the editorial for the special issue on machine learning for image reconstruction in IEEE Transactions on Medical Imaging \cite{Wang:2018ab} where the editors for the special issue introduce the term \emph{rawdiomics} (on p.~1294) for task adapted reconstruction applied to radiomics.
	
	An important advantage that comes with a joint approach is increased robustness. 
	Advanced tasks, like radiomics, commonly rely on deep neural networks that are trained on images in a supervised setting. 
	Images are however inferred in a pre-processing step from measured data, so contrast and texture may depend on the instrumentation used for acquiring the data and the reconstruction method used for computing the images. 
	Hence, a neural network that has been trained against images acquired from a particular equipment, or obtained using a particular reconstruction method, may generalize poorly when either of these factors change. 
	This is especially the case for tasks involving elements of visual classification, such as semantic segmentation, that can be sensitive to variations in texture and contrast.
	In contrast, task adapted reconstruction acts on measured data instead of images (model parameters). 
	Using a reconstruction step that incorporates a physics guided model for how measured data is generated results in a joint approach that is much more robust against variations in how data is acquired and processed.
	As an example, jointly training a learned iterative method with neural network(s) involved in radiomics will result in a joint scheme that is expected to be much more robust against variations in scanner and acquisition protocol. 
	This is essential if radiomics is to be part of clinical-decision support systems for improving diagnostic, prognostic, and predictive accuracy.
	
	Another important advantage with the proposed task adapted reconstruction method relates to computationally feasibility. 
	The trained neural network for task adapted reconstruction scales to large scale problems. 
	Such scalability remains a serious issue with the variational approaches mentioned in \cref{sec:Overview}. 
	As an example, state-of-the-art methods for joint reconstruction and segmentation are based on a variational approach using a the Mumford-Shah functional, which quickly become computationally unfeasible. 
	In contrast, the 2D examples in \cref{fig:JointRecoSeg} for joint reconstruction and segmentation are obtained using the approach in \cref{sec:JointRecoSeg} and these take a few milliseconds on a desktop gaming PC. 
	
	Finally, examples involving tomographic image reconstruction (\cref{sec:Applications}) support the claim that a joint approach outperforms a sequential one. 
	Understand this theoretically (\cref{sec:TheoryConsider}) is however an open problem.
	In particular, there is currently no theory motivating using a joint loss of the type in \cref{eq:JointLoss}, even though empirical evidence suggests such a choice outperforms the en-to-end and sequential approaches.

	\section*{Acknowledgments}
	The work by Jonas Adler, Olivier Verdier, and Ozan \"Oktem has been supported by the Swedish Foundation for Strategic Research grant AM13-0049, Industrial PhD grant ID14-0055, and Elekta AB. Carola-Bibiane Sch\"onlieb and Sebastian Lunz acknowledge support from the Engineering and Physical Sciences Research Council (EPSRC) ``EP/K009745/1'', the EPSRC grant ``EP/M00483X/1'', the EPSRC centre ``EP/N014588/1'', the Leverhulme Trust project ``Breaking the non-convexity barrier'', the CHiPS (Horizon 2020 RISE project grant), the Cantab Capital Institute for the Mathematics of Information, and the Alan Turing Institute ``TU/B/000071''.

	\bibliographystyle{siamplain}
	\bibliography{references}

\begin{thebibliography}{100}

\bibitem{Adler:2017aa}
{\sc J.~Adler and O.~{\"O}ktem}, {\em Solving ill-posed inverse problems using
  iterative deep neural networks}, Inverse Problems, 33 (2017), p.~124007
  (24pp).

\bibitem{Adler:2018aa}
{\sc J.~Adler and O.~{\"O}ktem}, {\em Learned primal-dual reconstruction}, IEEE
  Transactions on Medical Imaging, 37 (2018), pp.~1322--1332.

\bibitem{Adler:2017ab}
{\sc J.~Adler, A.~Ringh, O.~{\"O}ktem, and J.~Karlsson}, {\em Learning to solve
  inverse problems using {Wasserstein} loss}, ArXiv, 1710.10898 (2017).
\newblock Poster in NIPS 2017.

\bibitem{Ahmed:2018aa}
{\sc E.~Ahmed, A.~Saint, A.~Shabayek, K.~Cherenkova, R.~Das, G.~Gusev,
  D.~Aouada, and B.~Ottersten}, {\em Deep learning advances on different {3D}
  data representations: A survey}, ArXiv, 1808.01462 (2018).

\bibitem{Badrinarayanan:2017aa}
{\sc V.~Badrinarayanan, A.~Kendall, and R.~Cipolla}, {\em {SegNet}: {A} deep
  convolutional encoder-decoder architecture for image segmentation}, IEEE
  Transactions on Pattern Analysis and Machine Intelligence, 39 (2017),
  pp.~2481--2495.

\bibitem{Banerjee:2005aa}
{\sc A.~Banerjee, X.~Guo, and H.~Wang}, {\em On the optimality of conditional
  expectation as a {Bregman} predictor}, IEEE Transactions on Information
  Theory, 51 (2005), pp.~2664--2669.

\bibitem{Benning:2018aa}
{\sc M.~Benning and M.~Burger}, {\em Modern regularization methods for inverse
  problems}, Acta Numerica, 27 (2018), pp.~1--111.

\bibitem{Biegler:2011aa}
{\sc L.~Biegler, G.~Biros, O.~Ghattas, M.~Heinkenschloss, D.~Keyes, B.~Mallick,
  Y.~Marzouk, L.~Tenorio, B.~van Bloemen~Waanders, and K.~Willcox}, eds., {\em
  Large-Scale Inverse Problems and Quantification of Uncertainty}, John Wiley
  \& Sons, 2011.

\bibitem{Blei:2017aa}
{\sc D.~M. Blei, A.~K{\"u}{\c c}{\"u}kelbir, and J.~D. McAuliffe}, {\em
  Variational inference: {A} review for statisticians}, Journal of the American
  Statistical Association, 112 (2017), pp.~859--877.

\bibitem{Burger:2017aa}
{\sc M.~Burger, H.~Dirks, and L.~Frerking}, {\em On optical flow models for
  variational motion estimation}, in Variational Methods In Imaging and
  Geometric Control, M.~Bergounioux, G.~Peyr{\'e}, C.~Schn{\"o}rr, J.-P.
  Caillau, and T.~Haberkorn, eds., vol.~18 of Radon Series on Computational and
  Applied Mathematics, Walter de Gruyter, 2017, pp.~225--251.

\bibitem{Burger:2013aa}
{\sc M.~Burger and S.~Osher}, {\em A guide to the tv zoo}, in Level Set and PDE
  Based Reconstruction Methods in Imaging, Springer-Verlag, 2013, pp.~1--70.

\bibitem{Calvetti:2005aa}
{\sc D.~Calvetti and E.~Somersalo}, {\em Priorconditioners for linear systems},
  Inverse Problems, 21 (2005), pp.~1397--1418.

\bibitem{Calvetti:2017aa}
{\sc D.~Calvetti and E.~Somersalo}, {\em Inverse problems: From regularization
  to {Bayesian} inference}, {WIREs} Computational Statistics, 10 (2017),
  p.~e1427.

\bibitem{Chen:2018aa}
{\sc C.~Chen and O.~{\"O}ktem}, {\em Indirect image registration with large
  diffeomorphic deformations}, SIAM Journal on Imaging, 11 (2018),
  pp.~575--617.

\bibitem{Cicek:2016aa}
{\sc {\"O}.~{\c C}i{\c c}ek, A.~A., S.~S. Lienkamp, T.~Brox, and
  O.~Ronneberger}, {\em {3D U-Net}: {L}earning dense volumetric segmentation
  from sparse annotation}, in Medical Image Computing and Computer-Assisted
  Intervention -- MICCAI 2016: 19th International Conference, Athens, Greece,
  October 17-21, 2016, Proceedings, Part II, S.~Ourselin, L.~Joskowicz, S.~M.
  R., G.~{\"U}nal, and W.~Wells, eds., vol.~9901 of Lecture Notes in Computer
  Science, Springer-Verlag, 2016, pp.~424--432.

\bibitem{Csiszar:2008aa}
{\sc I.~Csisz{\'a}r}, {\em Axiomatic characterizations of information
  measures}, Entropy, 10 (2008), pp.~261--273.

\bibitem{Dahl:2017aa}
{\sc R.~Dahl, M.~Norouzi, and J.~Shlens}, {\em Pixel recursive super
  resolution}, ArXiv,  (2017).

\bibitem{Dashti:2016aa}
{\sc M.~Dashti and A.~Stuart}, {\em The {Bayesian} approach to inverse
  problems}, in Handbook of Uncertainty Quantification, R.~Ghanem, D.~Higdon,
  and H.~Owhadi, eds., Springer-Verlag, New York, 2016, ch.~10.

\bibitem{Deuflhard:2010aa}
{\sc P.~Deuflhard, O.~D{\"o}ssel, A.~K. Louis, and S.~Zachow}, {\em More
  mathematics into medicine!}, in Production Factor Mathematics,
  M.~Gr{\"o}tschel, K.~Lucas, and V.~Mehrmann, eds., Springer-Verlag, 2010,
  pp.~357--378.

\bibitem{Diamond:2017aa}
{\sc S.~Diamond, V.~Sitzmann, S.~Boyd, G.~Wetzstein, and F.~Heide}, {\em Dirty
  pixels: Optimizing image classification architectures for raw sensor data},
  ArXiv, 1701.06487 (2017).

\bibitem{Diaspro:2007aa}
{\sc A.~Diaspro, M.~Schneider, P.~Bianchini, V.~Caorsi, D.~Mazza, M.~Pesce,
  I.~Testa, G.~Vicidomini, and C.~Usai}, {\em Two-photon excitation
  fluorescence microscopy}, in Science of Microscopy, P.~W. Hawkes and J.~C.~H.
  Spence, eds., vol.~2, Springer-Verlag, 2007, ch.~11, pp.~751--789.

\bibitem{Druzhkov:2016aa}
{\sc P.~N. Druzhkov and V.~D. Kustikova}, {\em A survey of deep learning
  methods and software tools for image classification and object detection},
  Pattern Recognition and Image Analysis, 26 (2016), pp.~9--15.

\bibitem{Engl:2000aa}
{\sc H.~W. Engl, M.~Hanke, and A.~Neubauer}, {\em Regularization of inverse
  problems}, vol.~375 of Mathematics and Its Applications, Springer-Verlag,
  2000.

\bibitem{Esteva:2017aa}
{\sc A.~Esteva, B.~Kuprel, R.~A. Novoa, J.~Ko, S.~M. Swetter, H.~M. Blau, and
  S.~Thrun}, {\em Dermatologist-level classification of skin cancer with deep
  neural networks}, Nature, 542 (2017), pp.~115--118.

\bibitem{Evans:2002aa}
{\sc S.~N. Evans and P.~B. Stark}, {\em Inverse problems as statistics},
  Inverse Problems, 18 (2002), pp.~R1--R55.

\bibitem{Farabet:2013aa}
{\sc C.~Farabet, C.~Couprie, L.~Najman, and Y.~LeCun}, {\em Learning
  hierarchical features for scene labeling}, IEEE transactions on pattern
  analysis and machine intelligence, 35 (2013), pp.~1915--1929.

\bibitem{Foucart:2013aa}
{\sc S.~Foucart and H.~Rauhut}, {\em Mathematical Introduction to Compressive
  Sensing}, Springer-Verlag, 2013.

\bibitem{Fox:2012aa}
{\sc C.~Fox and S.~Roberts}, {\em A tutorial on variational {Bayes}},
  Artificial Intelligence Review, 38 (2012), pp.~85--95.

\bibitem{Ghosal:2017ac}
{\sc S.~Ghosal and N.~Ray}, {\em Deep deformable registration: {Enhancing}
  accuracy by fully convolutional neural net}, Pattern Recognition Letters, 94
  (2017), pp.~81--86.

\bibitem{Ghosal:2017aa}
{\sc S.~Ghosal and A.~W. van~der Vaart}, {\em Fundamentals of Nonparametric
  {Bayesian} Inference}, Cambridge Series in Statistical and Probabilistic
  Mathematics, Cambridge University Press, 2017.

\bibitem{Gillies:2015aa}
{\sc R.~J. Gillies, P.~E. Kinahan, and H.~Hricak}, {\em Radiomics: Images are
  more than pictures, they are data}, Radiology, 278 (2015), pp.~563--577.

\bibitem{Gris:2018aa}
{\sc B.~Gris, C.~Chen, and O.~{\"O}ktem}, {\em Image reconstruction through
  metamorphosis}, ArXiv, 1806.01225 (2018).

\bibitem{GuGeSuWe17}
{\sc C.~Guo, G.~Geoff~Pleiss, Y.~Sun, and K.~O. Weinberger}, {\em On
  calibration of modern neural networks}, ArXiv, 1706.04599 (2017).

\bibitem{Guo:2018aa}
{\sc Y.~Guo, Y.~Liu, T.~Georgiou, and M.~S. Lew}, {\em A review of semantic
  segmentation using deep neural networks}, International Journal of Multimedia
  Information Retrieval, 7 (2018), pp.~87--93.

\bibitem{Gupta:2018aa}
{\sc H.~Gupta, K.~H. Jin, H.~Q. Nguyen, M.~T. McCann, and M.~Unser}, {\em
  {CNN}-based projected gradient descent for consistent {CT} image
  reconstruction}, IEEE Transactions on Medical Imaging, 37 (2018),
  pp.~1440--1453.

\bibitem{Haenssle:2018aa}
{\sc H.~A. Haenssle, C.~Fink, R.~Schneiderbauer, F.~Toberer, T.~Buhl, A.~Blum,
  A.~Kalloo, A.~Ben Hadj~Hassen, L.~Thomas, A.~Enk, and L.~Uhlmann}, {\em Man
  against machine: diagnostic performance of a deep learning convolutional
  neural network for dermoscopic melanoma recognition in comparison to 58
  dermatologists}, Annals of Oncology,  (2018).
\newblock Epub ahead of print.

\bibitem{Hammernik:2018aa}
{\sc K.~Hammernik, E.~Klatzer, T. ans~Kobler, M.~P. Recht, D.~K. Sodickson,
  T.~Pock, and F.~Knoll}, {\em Learning a variational network for
  reconstruction of accelerated {MRI} data}, Magnetic Resonance in Medicine, 79
  (2018), pp.~3055--3071.

\bibitem{Hauptmann:2018aa}
{\sc A.~Hauptmann, F.~Lucka, M.~Betcke, N.~Huynh, J.~Adler, B.~Cox, P.~Beard,
  S.~Ourselin, and S.~Arridge}, {\em Model-based learning for accelerated
  limited-view {3-D} photoacoustic tomography}, IEEE Transactions on Medical
  Imaging, 37 (2018), pp.~1382--1393.

\bibitem{HeZhReSu16}
{\sc K.~He, X.~Zhang, S.~Ren, and J.~Sun}, {\em Deep residual learning for
  image recognition}, in 2016 IEEE Conference on Computer Vision and Pattern
  Recognition ({CVPR}), 2016, pp.~770--778.

\bibitem{He:2016aa}
{\sc K.~He, X.~Zhang, S.~Ren, and J.~Sun}, {\em Deep residual learning for
  image recognition}, in 2016 IEEE Conference on Computer Vision and Pattern
  Recognition {(CVPR)}, 2016, pp.~770--778.

\bibitem{Hell:2007aa}
{\sc S.~W. Hell, A.~Sch{\"o}nle, and A.~Van~den Bos}, {\em Nanoscale resolution
  in far-field fluorescence microscopy}, in Science of Microscopy, P.~W. Hawkes
  and J.~C.~H. Spence, eds., vol.~2, Springer-Verlag, 2007, ch.~12,
  pp.~790--834.

\bibitem{Hochreiter:1997}
{\sc S.~Hochreiter and J.~Schmidhuber}, {\em Long short-term memory}, Neural
  computation, 9 (1997), pp.~1735--1780.

\bibitem{Hohage:2016}
{\sc T.~Hohage and F.~Werner}, {\em Inverse problems with {Poisson} data:
  statistical regularization theory, applications and algorithms}, Inverse
  Problems, 32 (2016), p.~093001 (56 pp).

\bibitem{Hohm:2015aa}
{\sc K.~Hohm, M.~Storath, and A.~Weinmann}, {\em An algorithmic framework for
  {M}umford-{S}hah regularization of inverse problems in imaging}, Inverse
  Problems, 31 (2015), p.~115011 (30pp).

\bibitem{Iizuka:2016aa}
{\sc S.~Iizuka, E.~Simo-Serra, and H.~Ishikawa}, {\em Let there be color!:
  {Joint} end-to-end learning of global and local image priors for automatic
  image colorization with simultaneous classification}, ACM Transactions on
  Graphics, 35 (2016).
\newblock Proceedings of ACM SIGGRAPH 2016.

\bibitem{Johnson:2016aa}
{\sc J.~Johnson, A.~Alahi, and L.~Fei-Fei}, {\em Perceptual losses for
  real-time style transfer and super-resolution}, in European Conference on
  Computer Vision (ECCV 2016): 14th European Conference, Amsterdam, The
  Netherlands, October 11-14, 2016, Proceedings, Part II, B.~Leibe, J.~Matas,
  N.~Sebe, and W.~M., eds., vol.~9906 of Lecture Notes in Computer Science,
  Springer-Verlag, 2016, pp.~694--711.

\bibitem{Kadrmas:2004aa}
{\sc D.~J. Kadrmas}, {\em {LOR-OSEM}: statistical {PET} reconstruction from raw
  line-of-response histograms}, Physics in Medicine \& Biology, 49 (2004),
  pp.~4731--4744.

\bibitem{Kaipio:2005aa}
{\sc J.~P. Kaipio and E.~Somersalo}, {\em Statistical and Computational Inverse
  Problems}, vol.~160 of Applied Mathematical Sciences, Springer-Verlag, 2005.

\bibitem{Kaltenbacher:2008aa}
{\sc B.~Kaltenbacher, A.~Neubauer, and O.~Scherzer}, {\em Iterative
  Regularization Methods for Nonlinear Ill-Posed Problems}, vol.~6 of Radon
  Series on Computational and Applied Mathematics, Walter de Gruyter, 2008.

\bibitem{Karlsson:2017aa}
{\sc J.~Karlsson and A.~Ringh}, {\em Generalized sinkhorn iterations for
  regularizing inverse problems using optimal mass transport}, SIAM Journal on
  Imaging Sciences, 10 (2017), pp.~1935--1962.

\bibitem{Karpathy:2017aa}
{\sc A.~Karpathy and L.~Fei-Fei}, {\em Deep visual-semantic alignments for
  generating image descriptions}, IEEE Transactions on Pattern Analysis and
  Machine Intelligence, 39 (2017), pp.~664--676.

\bibitem{Kirsch:2011aa}
{\sc A.~Kirsch}, {\em An Introduction to the Mathematical Theory of Inverse
  Problems}, vol.~120 of Applied Mathematical Sciences, Springer-Verlag,
  2nd~ed., 2011.

\bibitem{Krishnan:2015aa}
{\sc V.~P. Krishnan and E.~T. Quinto}, {\em Microlocal analysis in tomography},
  in Handbook of Mathematical Methods in Imaging, O.~Scherzer, ed.,
  Springer-Verlag, 2nd~ed., 2015, pp.~847--902.

\bibitem{Krizhevsky:2012aa}
{\sc A.~Krizhevsky, I.~Sutskever, and G.~E. Hinton}, {\em {ImageNet}
  classification with deep convolutional neural networks}, in {NIPS'12}
  Proceedings of the 25th International Conference on Neural Information
  Processing Systems - Volume 1, 2012, pp.~1097--1105.

\bibitem{Kutyniok:2012aa}
{\sc G.~Kutyniok and D.~Labate}, eds., {\em Shearlets: Multiscale Analysis for
  Multivariate Data}, Springer-Verlag, 2012.

\bibitem{Lambin:2017aa}
{\sc P.~Lambin, R.~T.~H. Leijenaar, T.~M. Deist, J.~Peerlings, E.~E.~C.
  de~Jong, J.~van Timmeren, S.~Sanduleanu, R.~T. H.~M. Larue, A.~J.~G. Even,
  A.~Jochems, Y.~van Wijk, H.~Woodruff, J.~van Soest, T.~Lustberg, E.~Roelofs,
  W.~van Elmpt, A.~Dekker, F.~M. Mottaghy, J.~E. Wildberger, and S.~Walsh},
  {\em Radiomics: the bridge between medical imaging and personalized
  medicine}, Nature Reviews Clinical Oncology, 14 (2017), pp.~749--762.

\bibitem{LeCun:2015aa}
{\sc Y.~LeCun, Y.~Bengio, and G.~Hinton}, {\em Deep learning}, Nature, 521
  (2015), pp.~436--444.

\bibitem{LeBoBeHa98}
{\sc Y.~LeCun, L.~Bottou, Y.~Bengio, and P.~Haffner}, {\em Gradient-based
  learning applied to document recognition}, Proceedings of the IEEE, 86
  (1998), pp.~2278--2324.

\bibitem{Lee:2018aa}
{\sc J.~S. Lee, C.~Kim, J.-H. Shin, H.~Cho, D.-S. Shin, N.~Kim, H.~J. Kim,
  Y.~Kim, S.~N. Lockhart, D.~L. Na, S.~S. W., and J.-K. Seong}, {\em Machine
  learning-based individual assessment of cortical atrophy pattern in
  {Alzheimer's} disease spectrum: Development of the classifier and
  longitudinal evaluation}, Scientific Reports, 8 (2018).

\bibitem{Liese:2008aa}
{\sc F.~Liese and K.-J. Miescke}, {\em Statistical Decision Theory: Estimation,
  Testing, and Selection}, Springer Series in Statistics, Springer-Verlag, New
  York, 2008.

\bibitem{Long:2015aa}
{\sc J.~Long, E.~Shelhamer, and T.~Darrell}, {\em Fully convolutional networks
  for semantic segmentation}, in 2015 IEEE Conference on Computer Vision and
  Pattern Recognition {(CVPR)}, 2015, pp.~3431--3440.

\bibitem{Louis:2011aa}
{\sc A.~K. Louis}, {\em Feature reconstruction in inverse problems}, Inverse
  Problems, 27 (2011), p.~065010 (21pp).

\bibitem{Lunz:2018aa}
{\sc S.~Lunz, O.~{\"O}ktem, and C.-B. Sch{\"o}nlieb}, {\em Adversarial
  regularizers in inverse problems}, ArXiv, 1805.11572 (2018).
\newblock Submitted to NIPS 2018.

\bibitem{Mardani:2017ab}
{\sc M.~Mardani, E.~Gong, J.~Y. Cheng, S.~Vasanawala, G.~Zaharchuk, M.~Alley,
  N.~Thakur, S.~Han, W.~Dally, J.~M. Pauly, and L.~Xing}, {\em Deep generative
  adversarial networks for compressed sensing automates {MRI}}, ArXiv,
  1706.00051 (2017).

\bibitem{Mardani:2017aa}
{\sc M.~Mardani, H.~Monajemi, V.~Papyan, S.~Vasanawala, D.~Donoho, and
  J.~Pauly}, {\em Recurrent generative adversarial networks for proximal
  learning and automated compressive image recovery}, ArXiv, 1711.10046 (2017).

\bibitem{Mardani:2018aa}
{\sc M.~Mardani, Q.~Sun, S.~Vasawanala, V.~Papyan, H.~Monajemi, J.~Pauly, and
  D.~Donoho}, {\em Neural proximal gradient descent for compressive imaging},
  ArXiv, 1806.03963 (2018).

\bibitem{McCrackin:2018aa}
{\sc L.~McCrackin}, {\em Early detection of {Alzheimer's} disease using deep
  learning}, in Advances in Artificial Intelligence: 31st {Canadian} Conference
  on Artificial Intelligence, Canadian AI 2018, Toronto, ON, Canada, May 8--11,
  2018, Proceedings, E.~Bagheri and J.~C.~K. Cheung, eds., vol.~10832 of
  Lecture Notes in Artificial Intelligence, 2018, pp.~355--359.

\bibitem{Menze:2015aa}
{\sc B.~H. Menze, A.~Jakab, S.~Bauer, J.~Kalpathy-Cramer, K.~Farahani,
  J.~Kirby, Y.~Burren, N.~Porz, J.~Slotboom, R.~Wiest, L.~Lanczi, E.~Gerstner,
  M.~A. Weber, T.~Arbel, B.~B. Avants, N.~Ayache, P.~Buendia, D.~L. Collins,
  N.~Cordier, J.~J. Corso, A.~Criminisi, T.~Das, H.~Delingette, {\c
  C}.~Demiralp, C.~R. Durst, M.~Dojat, S.~Doyle, J.~Festa, F.~Forbes,
  E.~Geremia, B.~Glocker, P.~Golland, X.~Guo, A.~Hamamci, K.~M. Iftekharuddin,
  R.~Jena, N.~M. John, E.~Konukoglu, D.~Lashkari, J.~A. Mariz, R.~Meier,
  S.~Pereira, D.~Precup, S.~J. Price, T.~R. Raviv, S.~M. Reza, M.~Ryan,
  D.~Sarikaya, L.~Schwartz, H.~C. Shin, J.~Shotton, C.~A. Silva, N.~Sousa,
  N.~K. Subbanna, G.~Szekely, T.~J. Taylor, O.~M. Thomas, N.~J. Tustison,
  G.~Unal, F.~Vasseur, M.~Wintermark, D.~H. Ye, L.~Zhao, B.~Zhao, D.~Zikic,
  M.~Prastawa, M.~Reyes, and K.~Van~Leemput}, {\em The multimodal brain tumor
  image segmentation benchmark ({BRATS})}, IEEE Transactions on Medical
  Imaging, 34 (2015), pp.~1993--2024.

\bibitem{Minka:2001aa}
{\sc T.~Minka}, {\em Expectation propagation for approximate {Bayesian}
  inference}, in UAI '01: Proceedings of the 17th Conference in Uncertainty in
  Artificial Intelligence, University of Washington, Seattle, Washington, USA,
  J.~S. Breese and D.~Koller, eds., 2001, pp.~362--369.

\bibitem{Mohammad-Djafari:2009aa}
{\sc A.~Mohammad-Djafari}, {\em {G}auss-{M}arkov-{P}otts priors for images in
  computer tomography resulting to joint optimal reconstruction and
  segmentation}, International Journal of Tomography and Statistics, 11 (2009),
  pp.~76---92.

\bibitem{Monard:2017aa}
{\sc F.~Monard, R.~Nickl, and G.~P. Paternain}, {\em Efficient nonparametric
  {Bayesian} inference for x-ray transforms}, ArXiv, 1708.06332 (2017).

\bibitem{Natterer:2001ab}
{\sc F.~Natterer and F.~W{\"u}bbeling}, {\em Mathematical Methods in Image
  Reconstruction}, Mathematical Modeling and Computation, Society for
  Industrial and Applied Mathematics, 2001.

\bibitem{Nickl:2017aa}
{\sc R.~Nickl}, {\em On {Bayesian} inference for some statistical inverse
  problems with partial differential equations}, Bernoulli News, 24 (2017),
  pp.~5--9.

\bibitem{Noh:2015aa}
{\sc H.~Noh, S.~Hong, and B.~Han}, {\em Learning deconvolution network for
  semantic segmentation}, in 2015 IEEE Conference on Computer Vision and
  Pattern Recognition {(CVPR)}, 2015, pp.~1520--1528.

\bibitem{Pinkus:1999aa}
{\sc A.~Pinkus}, {\em Approximation theory of the {MLP} model in neural
  networks}, Acta Numerica,  (1999), pp.~143--195.

\bibitem{Poplin:2017aa}
{\sc R.~Poplin, A.~V. Varadarajan, K.~Blumer, Y.~Liu, M.~V. McConnell, G.~S.
  Corrado, L.~Peng, and D.~R. Webster}, {\em Predicting cardiovascular risk
  factors from retinal fundus photographs using deep learning}, ArXiv,
  1708.09843 (2017).

\bibitem{Prince:2012aa}
{\sc S.~J.~D. Prince}, {\em Computer Vision: Models, Learning, and Inference},
  Cambridge University Press, 2012.

\bibitem{Ramlau:2007aa}
{\sc R.~Ramlau and W.~Ring}, {\em A {M}umford--{S}hah level-set approach for
  the inversion and segmentation of {X}-ray tomography data}, Journal of
  Computational Physics, 221 (2007), pp.~539---557.

\bibitem{Romano:2017ab}
{\sc Y.~Romano, J.~Isidoro, and P.~Milanfar}, {\em {RAISR}: {R}apid and
  accurate image super resolution}, IEEE Transactions on Computational Imaging,
  3 (2017), pp.~110--125.

\bibitem{Romanov:2016aa}
{\sc M.~Romanov, B.~A. Dahl, Y.~Dong, and P.~C. Hansen}, {\em Simultaneous
  tomographic reconstruction and segmentation with class priors}, Inverse
  Problems in Science and Engineering, 24 (2016), pp.~1432--1453.

\bibitem{Ronneberger:2015aa}
{\sc O.~Ronneberger, P.~Fischer, and T.~Brox}, {\em {U-Net}: {C}onvolutional
  networks for biomedical image segmentation}, in {Medical Image Computing and
  Computer-Assisted Intervention -- MICCAI 2015: 18th International Conference,
  Munich, Germany, October 5-9, 2015, Proceedings, Part III}, N.~Navab,
  J.~Hornegger, W.~M. Wells, and A.~F. Frangi, eds., vol.~9351 of Lecture Notes
  in Computer Science, Springer-Verlag, 2015, pp.~234--241.

\bibitem{Rubin:2014aa}
{\sc G.~D. Rubin}, {\em Computed tomography: Revolutionizing the practice of
  medicine for 40 years}, Radiology, 273 (2014), pp.~45--74.

\bibitem{Rubinstein:2010aa}
{\sc R.~Rubinstein, M.~Bruckstein, and M.~Elad}, {\em Dictionaries for sparse
  representation modeling}, Proceedings of the IEEE, 98 (2010), pp.~1045--1057.

\bibitem{Saito:2016aa}
{\sc S.~Saito, T.~Li, and H.~Li}, {\em Real-time facial segmentation and
  performance capture from {RGB} input}, in {ECCV 2016}: 14th European
  Conference on Computer Vision, Amsterdam, The Netherlands, October 11-14,
  2016, Proceedings, Part VIII, B.~Leibe, J.~Matas, N.~Sebe, and M.~Welling,
  eds., 2016, pp.~244--261.

\bibitem{Scherzer:2009aa}
{\sc O.~Scherzer, M.~Grasmair, H.~Grossauer, M.~Haltmeier, and F.~Lenzen}, {\em
  Variational Methods in Imaging}, vol.~167 of Applied Mathematical Sciences,
  Springer-Verlag, 2009.

\bibitem{Sermanet:2013aa}
{\sc P.~Sermanet, D.~Eigen, X.~Zhang, M.~Mathieu, R.~Fergus, and Y.~LeCun},
  {\em {OverFeat}: {Integrated} recognition, localization and detection using
  convolutional networks}, ArXiv, 1312.6229 (2013).

\bibitem{Streit:2010aa}
{\sc R.~L. Streit}, {\em Poisson Point Processes: Imaging, Tracking, and
  Sensing}, Springer-Verlag, 2010.

\bibitem{Stuart:2010aa}
{\sc A.~M. Stuart}, {\em Inverse problems: A {Bayesian} perspective}, Acta
  Numerica, 19 (2010), pp.~451--559.

\bibitem{Sun:2018ab}
{\sc L.~Sun, Z.~Fan, Y.~Huang, X.~Ding, and J.~Paisley}, {\em Joint {CS-MRI}
  reconstruction and segmentation with a unified deep network}, ArXiv,
  1805.02165 (2018).

\bibitem{Syu:2018aa}
{\sc N.-S. Syu, Y.-S. Chen, and Y.-Y. Chuang}, {\em Learning deep convolutional
  networks for demosaicing}, ArXiv,  (2018).

\bibitem{Thoma:2016aa}
{\sc M.~Thoma}, {\em A survey of semantic segmentation}, ArXiv, 1602.06541
  (2016).

\bibitem{Vinyals:2015}
{\sc O.~Vinyals, A.~Toshev, S.~Bengio, and D.~Erhan}, {\em Show and tell: A
  neural image caption generator}, in Proceedings of the IEEE conference on
  computer vision and pattern recognition, 2015, pp.~3156--3164.

\bibitem{Wang:2018ab}
{\sc G.~Wang, J.~C. Ye, K.~Mueller, and J.~A. Fessler}, {\em Image
  reconstruction is a new frontier of machine learning}, IEEE Transactions on
  Medical Imaging, 37 (2018), pp.~1289--1296.

\bibitem{Wang:2018aa}
{\sc M.~Wang and W.~Deng}, {\em Deep face recognition: A survey}, ArXiv,
  1804.06655 (2018).

\bibitem{Wolterink:2017aa}
{\sc J.~M. Wolterink, A.~M. Dinkla, M.~H.~F. Savenije, P.~R. Seevinck, C.~A.~T.
  van~den Berg, and I.~I{\v s}gum}, {\em Deep {MR} to {CT} synthesis using
  unpaired data}, in Simulation and Synthesis in Medical Imaging. {SASHIMI
  2017}, S.~Tsaftaris, A.~Gooya, A.~Frangi, and J.~Prince, eds., vol.~10557 of
  Lecture Notes in Computer Science, 2017, pp.~14--23.

\bibitem{Wu:2017}
{\sc D.~Wu, K.~Kim, B.~Dong, and Q.~Li}, {\em End-to-end abnormality detection
  in medical imaging}, ArXiv, 1711.02074 (2017).

\bibitem{Xie:2012aa}
{\sc J.~Xie, L.~Xu, and E.~Chen}, {\em Image denoising and inpainting with deep
  neural networks}, in Proceedings of the 25th International Conference on
  Neural Information Processing Systems (NIPS 2012), 2012, pp.~341--349.

\bibitem{Yang:2017aa}
{\sc X.~Yang, R.~Kwitt, M.~Styner, and M.~Niethammer}, {\em Quicksilver: {Fast}
  predictive image registration -- a deep learning approach}, NeuroImage, 158
  (2017), pp.~378--396.

\bibitem{Yoon:2010aa}
{\sc S.~Yoon, A.~R. Pineda, and R.~Fahrig}, {\em Simultaneous segmentation and
  reconstruction: A level set method approach for limited view computed
  tomography}, Medical Physics, 37 (2010), pp.~2329--2340.

\bibitem{Zhao:2003aa}
{\sc W.~Zhao, R.~Chellappa, A.~Rosenfeld, and P.~J. Phillips}, {\em Face
  recognition: A literature survey}, ACM Computing Surveys, 35 (2003),
  pp.~399--458.

\end{thebibliography}
	
\end{document}